\pdfoutput=1
\documentclass[11pt]{article}

\usepackage{graphicx}
\usepackage{placeins}
\usepackage{comment}
\usepackage{subcaption}
\usepackage[margin=1in]{geometry}
\usepackage{lmodern}
\usepackage[T1]{fontenc}
\usepackage{microtype}
\usepackage{times} 
\usepackage{natbib}
\usepackage{amsmath,amssymb,amsfonts}
\usepackage{mathrsfs}
\usepackage{calrsfs}

\usepackage{amsmath,amsfonts,bm}

\def\eqref#1{equation~\ref{#1}}

\def\1{\bm{1}}

\DeclareMathAlphabet{\mathsfit}{\encodingdefault}{\sfdefault}{m}{sl}
\SetMathAlphabet{\mathsfit}{bold}{\encodingdefault}{\sfdefault}{bx}{n}

\newcommand{\E}{\mathbb{E}}

\newcommand{\R}{\mathbb{R}}

\newcommand{\KL}{D_{\mathrm{KL}}}

\usepackage{booktabs}

\usepackage{textcomp}

\usepackage{xcolor}
\usepackage{url}
\usepackage[colorlinks=true,linkcolor=blue,citecolor=blue,urlcolor=blue]{hyperref}
\hypersetup{ pdftitle={Smoothness-Based Derandomization of PAC-Bayes Bounds}, pdfauthor={Alexandre Lemire Paquin, Brahim Chaib-Draa, Philippe Giguère} }
\usepackage{titlesec}
\allowdisplaybreaks
\titleformat{\section}
  {\normalfont\Large\bfseries}
  {\thesection}
  {1em}
  {}

\titleformat{\subsection}
  {\normalfont\large\bfseries}
  {\thesubsection}
  {1em}
  {}

\titleformat{\subsubsection}
  {\normalfont\normalsize\bfseries}
  {\thesubsubsection}
  {1em}
  {}

\titlespacing*{\section}{0pt}{2.0ex plus 0.5ex minus 0.3ex}{1.0ex plus 0.3ex}
\titlespacing*{\subsection}{0pt}{1.5ex plus 0.3ex minus 0.2ex}{0.7ex}
\titlespacing*{\subsubsection}{0pt}{1.0ex plus 0.2ex minus 0.1ex}{0.5ex}
\usepackage{setspace}

\usepackage{amsthm}

\newtheorem{lem}{Lemma}[section]

\newtheorem{prop}[lem]{Proposition}
\newtheorem{thm}[lem]{Theorem}
\newtheorem{coro}[lem]{Corollary}

\theoremstyle{definition}
\newtheorem{de}[lem]{Definition}

\theoremstyle{remark}
\newtheorem{rem}[lem]{Remark}

\newcommand{\klb}{\operatorname{kl}}

\newcommand{\jen}{{\operatorname{Jen}}}
\newcommand{\var}{{\operatorname{Var}}}
\newcommand{\tr}{{\operatorname{Trace}}}

\DeclareMathOperator{\erf}{erf}

\newcommand{\Endproof}{\hfill$\blacksquare$\par\medskip}
\title{Smoothness-Based Derandomization of PAC-Bayes Bounds}

\author{
Alexandre Lemire Paquin \quad Brahim Chaib-Draa \quad Philippe Gigu\`ere\\
Department of Computer Science and Software Engineering\\
Université Laval, Québec, Canada\\
\texttt{alexandre.lemire-paquin.1@ulaval.ca}\\
\texttt{brahim.chaib-draa@ift.ulaval.ca}\\
\texttt{philippe.giguere@ift.ulaval.ca}
}
\date{June 2026}

\begin{document}

\maketitle

\begin{abstract}
We study PAC-Bayes derandomization for smooth loss functions. Our goal is to obtain generalization bounds that hold with high probability for deterministic predictors by exploiting smoothness properties of both the loss and the predictor class. We show that passing from the Gibbs predictor to the deterministic predictor at the posterior mean has a precise cost, given by the generalization gap of the Jensen gap class. We control this class through its Rademacher complexity, leading to bounds for deterministic predictors that involve flatness quantities expressed in terms of parameter Jacobians and Hessians of the score map. The framework applies to both bounded and unbounded smooth loss functions, and we specialize the results to linear predictors and smooth neural networks. Finally, the Jacobian and Hessian quantities appearing in the theory motivate a practical regularizer. For BatchNorm networks, we compute this regularizer with respect to effective BatchNorm weights obtained by folding the BatchNorm transformation into the adjacent affine weights. Experiments on CIFAR-10 illustrate the behavior of this regularizer under different batch sizes.
\end{abstract}

\newpage
\setcounter{tocdepth}{1}
\tableofcontents
\newpage

\section{Introduction}

Flatness measures defined through parameter perturbations are among the more
promising predictors of the generalization gap in large-scale empirical studies
on neural networks
\citep{jiang2020fantastic,dziugaite2020search}. PAC-Bayes theory provides a
natural framework for deriving high-probability generalization bounds involving
such measures
\citep{mcallester1998some,mcallester1999pacbayes,seeger2002pacbayes,
catoni2007pacbayes,haddouche2025flatminima}. However, standard PAC-Bayes bounds apply to stochastic
Gibbs predictors, whereas a single deterministic predictor is often used
in practice. Derandomization of PAC-Bayes bounds provides a rigorous bridge
between guarantees for stochastic classifiers and guarantees for deterministic
predictors
\citep{langford2002pacbayes,viallard2021general,
neyshabur2018spectralpacbayes,banerjee2020derandomized,biggs2022margins}.
Existing results connect ideas involving classification margins and flatness
through an analysis of parameter perturbations. While this provides an
important connection, it ties the resulting guarantees to margin losses, which
cannot be directly optimized using gradient-based methods. In this work, we
consider a new general strategy for PAC-Bayes derandomization that relies on
smoothness assumptions. This strategy yields bounds involving flatness measures
that apply directly to deterministic predictors, for both classification and
regression with bounded or unbounded smooth loss functions.

We show that one can pass from a PAC-Bayes bound for the stochastic Gibbs
predictor to a bound for the deterministic predictor at the posterior mean by
paying an additional penalty term. This penalty is the generalization gap of
the Jensen gap of the loss class; see
Theorem~\ref{thm:pb-jensen-gap-omega} and
Corollary~\ref{cor:replace-omega-by-g-delta-half}. The Jensen gap of a function
\(f\) with respect to a posterior \(Q\) is the difference
\(f(\mathbb{E}_{w\sim Q}[w])-\mathbb{E}_{w\sim Q}f(w)\); see
Definition~\ref{def:jensen-gap}. This penalty therefore identifies the cost of
replacing the Gibbs classifier by its posterior mean.

The remaining task is to control this Jensen gap class uniformly with high
probability. For this purpose, we use a Talagrand comparison inequality for
sub-Gaussian processes to upper bound its Rademacher complexity; see
Section~\ref{sec:jg-rademacher}. The resulting bounds explicitly depend on
smoothness and flatness properties of both the loss and the hypothesis class.
In particular, Theorem~\ref{combinedThm} gives the general result in the
unbounded case, while Theorem~\ref{combinedThmLin} gives its specialization to
linear predictors.

The main contributions of this work are as follows.
\begin{enumerate}
\item We obtain uniform high-probability control of the Jensen gap class by
bounding its Rademacher complexity (Corollary~\ref{cor:rad-jensen-gap},
Proposition~\ref{cor:jensen-gap-genbound}, and
Proposition~\ref{prop:jensen-gap-genbound-dyadic}).

\item We derive general derandomized PAC-Bayes bounds for unbounded and bounded
smooth loss functions. The unbounded case is treated in
Section~\ref{sec:pb-unbounded}, with the main general result stated in
Theorem~\ref{combinedThm}. The bounded case is treated separately in
Section~\ref{sec:pb-bounded}, leading in particular to
Theorem~\ref{bounded} and to the inverse-kl bound of
Theorem~\ref{thm:bounded-kl-regrouped-rademacher-sigmai}.

\item We specialize the theory to linear predictors. In the unbounded case,
Theorem~\ref{combinedThmLin} gives the corresponding linear bound, while
Theorem~\ref{boundedLin} gives the analogue for bounded losses. A key step in
the unbounded analysis is the evaluation of the PAC-Bayes complexity term under
Gaussian priors. Lemma~\ref{lem:pacbayes-gaussian-term} gives a square root
dependence on the dimension for this term. This improves over the dimension
dependence obtained in
\citet{haddouche2020pacbayes}. Specific results for classification with the
multi-class unhinged loss
\citep{shoham2021exploration,Zhou2023unhinged,
lemirepaquin2026symmetrization} and for regression are given in
Proposition~\ref{unhprop}, Corollary~\ref{cor:absolute-loss-linear}, and
Corollary~\ref{cor:normalized-huber-loss-linear}. A comparison with a classical Rademacher complexity bound for a bounded smooth loss is presented in Section~\ref{sec:linear-bounded-experiments} and
Figures~\ref{fig:relative-improvement}--\ref{fig:linear-ce-bounds}.

\item We specialize our bounds to SHEL networks and compare them
with the margin bound of \citet{biggs2022margins} in
Section~\ref{sec:shel} and
Figures~\ref{fig:ce-training-batch-size}--\ref{fig:ce-training-width}.

\item We propose a practical regularizer inspired by the Jacobian and
Hessian quantities appearing in the theory. For BatchNorm networks, we
investigate a parametrization obtained by combining the BatchNorm
transformation with the preceding affine weights, leading to effective
BatchNorm weights \(w_{\mathrm{eff},B'}\). The resulting experiments are
presented in Section~\ref{sec:jh-bn-regularizer},
Figures~\ref{fig:batchsize-clean-fixed-vs-sqrt}--\ref{fig:batchsize-jh-effective-comparison},
and Table~\ref{tab:jh-bprime-probes-sqrt-epoch-comparison}.

\end{enumerate}

\paragraph{Structure of the document.}
In Section~\ref{related}, we contextualize our general method.
The presentation in Section~\ref{sec:pb-unbounded} follows the conceptual development of the bounds. Some of the technical ingredients
required for Theorem~\ref{combinedThm} are established later in
Sections~\ref{sec:jg-uniform-upper}--\ref{sec:jg-rademacher}. For clarity of
exposition, we state the main results first and refer to those later sections
within the proofs when the corresponding technical estimates are needed.

\section{Motivation and relation to existing derandomization approaches}
\label{related}
We first motivate the role of the Jensen gap through the simple case where the loss is convex in the parameters.
If a loss function $L(h_w(x),y)$ is convex with respect to $w$, then Jensen's inequality implies
\[
L(h_{\overline{w}}(x),y)\leq \E_{w\sim Q}[L(h_w(x),y)],
\]
where $\overline{w}=\E_{w\sim Q}[w]$. 
Starting from a PAC-Bayesian bound for $L$, one can then obtain in a straightforward manner a generalization bound for the deterministic classifier $\overline{w}$ of the form
\[
L_D(\overline{w})\leq L_S(\overline{w})
+\frac{1}{n}\sum_{i=1}^n\Big[
\E_{w\sim Q}L(h_w(x_i),y_i)
- L(h_{\overline{w}}(x_i),y_i)
\Big]
+\widetilde{O}\!\left(\frac{1}{\sqrt{n}}\right).
\]
The quantity
\[
\frac{1}{n}\sum_{i=1}^n\Big[
\E_{w\sim Q}L(h_w(x_i),y_i)
- L(h_{\overline{w}}(x_i),y_i)
\Big]
\]
can be interpreted as an empirical measure of flatness of the loss $L$ around the mean $\overline{w}$. 
This direct argument constitutes perhaps the simplest form of derandomization of a PAC-Bayes bound. 
However, it relies on convexity and therefore cannot be applied to typical neural network hypothesis classes. Moreover, even for linear predictors, it does not apply for example to
non-linear symmetric loss functions \(L(z,y)\), which are non-convex
\citep{Ghoshneuro}. This motivates the development of a more general
approach.

The framework developed in this work replaces this convexity argument by an exact
decomposition. The additional term required for derandomization is the
generalization gap of the Jensen gap of the loss class, and the resulting
bounds take the schematic form
\[
L_D(\overline{w})
\leq
L_S(\overline{w})
+
\frac{\text{smoothness/flatness terms}}{\sqrt n}
+
\widetilde{O}\!\left(\frac{1}{\sqrt n}\right).
\]
This perspective can be advantageous even in convex settings, since the
empirical flatness term appearing in the direct convexity argument need not
vanish with the sample size.
 
We now position our approach within the literature that seeks  guarantees on deterministic classifiers derived from PAC-Bayesian bounds.

\citet{viallard2021general} develop a general framework for deriving pointwise PAC-Bayesian generalization bounds that apply to an individual hypothesis sampled from the posterior, rather than to the Gibbs predictor.
This form of derandomization is referred to as the \emph{disintegration} of PAC-Bayesian bounds. Related deterministic PAC-Bayesian perspectives were later developed by \citet{clerico2025deterministic}.
In the disintegrated setting, for instance, if one assumes an explicit posterior family (e.g., a Gaussian centered at the algorithm’s output), the resulting guarantee applies to a specific hypothesis drawn from that posterior and is probabilistic with respect to this sampling step (in addition to the randomness of the sample $S$).
In contrast, our bounds apply directly to the deterministic classifier at the posterior mean in the classical sense and do not introduce additional hypothesis level randomness.

Among direct derandomization approaches, a prominent line of work relies on margin-based arguments. 
Early PAC-Bayesian margin formulations appear in \citet{langford2002pacbayes}. 
Subsequent developments derived results for modern neural network architectures \citep{neyshabur2018spectralpacbayes,biggs2022margins,banerjee2020derandomized}. These approaches provide sharp guarantees when suitable margin notions are available, but they depend on margin-specific constructions and are therefore not directly adapted to arbitrary loss functions, such as the symmetric loss functions considered in this work.
Define the margin of an example $(x,y)$ as
\[
m_{\overline w}(x,y)
=
h_{\overline w}(x)_y
-
\max_{k\neq y} h_{\overline w}(x)_k .
\]
For a margin parameter $\gamma>0$, the corresponding margin loss is defined by
\[
L_\gamma(h_{\overline w}(x),y)
=
\mathbf 1\{m_{\overline w}(x,y)\le \gamma\},
\]
which upper bounds the $0\text{-}1$ loss. Margin-based PAC-Bayesian approaches exploit inequalities involving the $0\text{-}1$ loss, the margin loss and the margin loss under a posterior $Q$. Consider as an illustration the strategy of \citet{neyshabur2018spectralpacbayes}. They define a subset $S_{\overline w}$ of weights $w$ such that the outputs of $h_w$ remain uniformly close to those of $h_{\overline w}$, namely
\[
S_{\overline w}
\subseteq
\Big\{
w:\;
\sup_{x\in\mathcal X}
\|h_w(x)-h_{\overline w}(x)\|_\infty
<
\gamma/4
\Big\}.
\]
Conditioning the posterior on this set yields a truncated posterior $\widetilde Q$ under which the inequalities
\[
L_D^{0\text{-}1}(h_{\overline w})
\le
\mathbb E_{w\sim \widetilde Q}
\big[
L_{D,\gamma/2}(h_{w})
\big]
\qquad\text{and}\qquad
\mathbb E_{w\sim \widetilde Q}
\big[
L_{S,\gamma/2}(h_{w})
\big]
\le
L_{S,\gamma}(h_{\overline w})
\]
hold. This enables the conversion of a stochastic PAC-Bayes guarantee into a deterministic margin-based bound. 

A related derandomization strategy, developed in \citet{banerjee2020derandomized}, also relies on inequalities that directly relate the margin loss of the deterministic predictor to the corresponding Gibbs (i.e., posterior-averaged) margin loss under $Q$. These inequalities are obtained by controlling how Gaussian perturbations of the parameters affect the margin, and they introduce a residual term of the form
\[
c_0 \exp\!\big(-\min(c_2\gamma^2,\; c_1\gamma)\big),
\]
which captures the probability that the margin changes significantly under the perturbation. This term is intrinsic to their derandomization step and does not vanish with the sample size $n$. In contrast, our approach preserves statistical consistency in the sense that all additional terms in the bound vanish as $n\to\infty$.

Fundamentally, the Jensen gap framework adopted in this work does not rely on
such ``derandomization inequalities'' for the margin loss. Instead, it starts
from an exact decomposition relating deterministic and stochastic risks through
the Jensen gap. This gives a conceptually distinct route to derandomization
that is naturally suited to smooth loss functions beyond margin-based
constructions. Moreover, the empirical loss appearing in our bounds can
coincide with the loss used for training, so gradient-based optimization
directly acts on a term that enters the bound. By contrast, the margin loss
\(L_\gamma\) cannot be directly optimized using gradient-based methods and
therefore requires a surrogate training loss that is different from the loss
appearing in the bound.

Another line of work obtaining deterministic guarantees from PAC-Bayes bounds is the literature on weighted majority-vote classifiers \citep{mcallester1999pacbayes,langford2002pacbayes}. Explicit bounds for the risk of the majority vote were derived in terms of the mean and variance of the error of the Gibbs classifier in \citep{lacasse2007majorityvote}, leading to the C-bound and the MinCq learning algorithm \citep{germain2015risk}. Second-order refinements of the C-bound were later established in \citep{masegosa2020secondorder}. More recently, \citet{leblanc2025deterministicrisk} proposed a general framework for extracting deterministic risk guarantees from stochastic PAC-Bayesian bounds, with applications to majority-vote classifiers. While powerful, this line of work is inherently tied, in its main applications, to vote-based predictors and does not directly address deterministic predictors defined through the posterior mean of parameters.  

\section{Generalization bounds for unbounded loss functions from the PAC-Bayes framework}
\label{sec:pb-unbounded}
Unbounded loss functions arise naturally in several learning problems. Standard
examples include regression losses such as the squared loss, the absolute loss,
and the Huber loss, as well as classification losses such as the cross-entropy
loss. They also arise in classification through symmetric losses such as the
multi-class unhinged loss
\citep{shoham2021exploration,Zhou2023unhinged} and through related
symmetrized losses studied in \citet{lemirepaquin2026symmetrization}. In this section, we rely on a PAC-Bayes result applicable to unbounded loss functions from \citet{alquier2016properties}; see also \citep{rivasplata2020pacbayesbeyond,casado2024pacbayeschernoff} for related PAC-Bayesian treatments beyond the bounded loss setting. We then exploit the assumption that, for each individual hypothesis \(h\), the loss function \(L(h(x),y)\) is bounded on \(X\) (typically a bounded domain) in order to apply Hoeffding’s inequality. This strategy is in the same spirit as \citet{haddouche2020pacbayes}, where hypothesis-dependent boundedness is used to extend PAC-Bayes bounds to learning problems involving unbounded regression loss functions.
\begin{thm}[{\cite{alquier2016properties}}]
\label{Alquier}
Given a distribution $D$ over $X \times Y$, a hypothesis set $\mathcal{H}$, a loss function $L : \mathcal{H} \times X \times Y \to \mathbb{R}$, a prior distribution $P$ over $\mathcal{H}$, a $\delta \in (0, 1]$, and a real number $\lambda > 0$, with probability at least $1 - \delta$ over the choice of $S \sim D^n$, we have:

\[\displaystyle
\forall Q \text{ on } \mathcal{H}:\quad 
\mathbb{E}_{h \sim Q} L_{D}(h) 
\leq 
\mathbb{E}_{h \sim Q} L_{S}(h) 
+ \frac{1}{\lambda} \left[
\KL(Q \| P) + \ln \frac{1}{\delta} + \Psi_{L,P,D}(\lambda, n)
\right],
\]

where
\[\displaystyle
\Psi_{L,P,D}(\lambda, n) := \ln \mathbb{E}_{h \sim P} \mathbb{E}_{S \sim D^n} \exp\left[
\lambda \left(
L_{D}(h) - L_{S}(h)
\right)
\right].
\]
\end{thm}
\noindent \textbf{Proof:} see \cite{alquier2016properties} and \cite{germain2016pac}. \Endproof

 Our next step is to upper bound $\Psi_{L,P,D}(\lambda,n)$ under a hypothesis-dependent boundedness condition.

\begin{prop}
\label{ah}
Assume that $\|h(x)\| \leq a_h$ on $X$ for every $h \in \mathcal{H}$.
Furthermore, assume that $L(z,y)$ is $\ell$-Lipschitz in $z$ for every $y$,
and that $L(0,y)$ is independent of $y$. Then, for every $\lambda>0$,
\[
\Psi_{L,P,D}(\lambda,n)
\le
\ln \E_{h\sim P}
\exp\!\left(
\frac{\lambda^2\ell^2 a_h^2}{2n}
\right).
\]
\end{prop}

\noindent\textbf{Proof:}
Without loss of generality, assume that $L(0,y)=0$, since adding a constant
to the loss does not affect $L_D(h)-L_S(h)$. Then, for every fixed $h$,
\[
|L(h(x),y)|\le \ell \|h(x)\| \le \ell a_h .
\]
Thus, by Hoeffding's lemma,
\[
\E_{S\sim D^n}
\exp\!\left(\lambda(L_D(h)-L_S(h))\right)
\le
\exp\!\left(
\frac{\lambda^2\ell^2 a_h^2}{2n}
\right).
\]
Taking expectation over $h\sim P$ and then the logarithm gives
\[
\Psi_{L,P,D}(\lambda,n)
=
\ln \E_{h\sim P}\E_{S\sim D^n}
\exp\!\left(\lambda(L_D(h)-L_S(h))\right)
\le
\ln \E_{h\sim P}
\exp\!\left(
\frac{\lambda^2\ell^2 a_h^2}{2n}
\right).
\]
\Endproof

We now convert the bound on the Gibbs risk from Theorem \ref{Alquier} into a bound at the mean of the distribution.

\begin{thm}
\label{thm:pb-jensen-gap-omega}
Let \(D\) be a distribution on \(X\times Y\), let \(S=\{(x_i,y_i)\}_{i=1}^n\sim D^n\), let
\(\mathcal H=\{h_w:w\in\R^m\}\), and let \(P\) be a prior on \(\R^m\).
Fix \(\delta\in(0,1]\) and \(\lambda>0\).
For any distribution \(Q\) on \(\R^m\) and any sample \(S\), define
\[
\Omega_{L,Q,D,S}(n)
:=
\Big(L_D(h_{\overline w})-\E_{w\sim Q}L_D(h_w)\Big)
-
\Big(L_S(h_{\overline w})-\E_{w\sim Q}L_S(h_w)\Big),
\qquad
\overline w:=\E_{w\sim Q}[w].
\]
Then, with probability at least \(1-\delta\) over the draw of \(S\), for all distributions
\(Q\) on \(\R^m\),
\[
L_D(h_{\overline w})
\le
L_S(h_{\overline w})
+
\Omega_{L,Q,D,S}(n)
+
\frac{1}{\lambda}\Big[
\KL(Q\|P)+\ln\tfrac1\delta+\Psi_{L,P,D}(\lambda,n)
\Big],
\]
where \(\Psi_{L,P,D}(\lambda,n)\) is as in Theorem~\ref{Alquier}.
\end{thm}

\noindent\textbf{Proof:}
Start from the identity
\[
L_D(h_{\overline w})
=
\E_{w\sim Q}L_D(h_w)
+
\Big(L_D(h_{\overline w})-\E_{w\sim Q}L_D(h_w)\Big).
\]
On the event of probability at least \(1-\delta\) given by Theorem~\ref{Alquier}, we have
simultaneously for all \(Q\):
\[
\E_{w\sim Q}L_D(h_w)
\le
\E_{w\sim Q}L_S(h_w)
+
\frac{1}{\lambda}\Big[
\KL(Q\|P)+\ln\tfrac1\delta+\Psi_{L,P,D}(\lambda,n)
\Big].
\]
Substituting this bound into the previous identity yields
\[
L_D(h_{\overline w})
\le
\E_{w\sim Q}L_S(h_w)
+
\Big(L_D(h_{\overline w})-\E_{w\sim Q}L_D(h_w)\Big)
+
\frac{1}{\lambda}\Big[
\KL(Q\|P)+\ln\tfrac1\delta+\Psi_{L,P,D}(\lambda,n)
\Big].
\]
Next, decompose the Gibbs empirical risk by adding and
subtracting \(L_S(h_{\overline w})\):
\[
\E_{w\sim Q}L_S(h_w)
=
L_S(h_{\overline w})
+
\Big(\E_{w\sim Q}L_S(h_w)-L_S(h_{\overline w})\Big).
\]
Plugging this decomposition into the last inequality gives
\[
\begin{aligned}
L_D(h_{\overline w})
\le\;&
L_S(h_{\overline w})
+
\bigg[\Big(L_D(h_{\overline w})-\E_{w\sim Q}L_D(h_w)\Big)
-
\Big(L_S(h_{\overline w})-\E_{w\sim Q}L_S(h_w)\Big)\bigg]
\\
&+
\frac{1}{\lambda}\Big[
\KL(Q\|P)+\ln\tfrac1\delta+\Psi_{L,P,D}(\lambda,n)
\Big],
\end{aligned}
\]
which concludes the proof.
\Endproof

\begin{rem}
    The term $\Omega_{L,Q,D,S}(n)$ is the exact penalty term needed to obtain a bound on the deterministic classifier at the mean of the distribution $Q$ from a bound on the stochastic classifier.
\end{rem}

\begin{coro}
\label{cor:replace-omega-by-g-delta-half}
Let \(D\) be a distribution on \(X\times Y\), let \(S\sim D^n\), and let \(L\) be a loss function.
Let \(\Theta\) be a nonempty subset of probability distributions on \(\R^m\).
Assume that there exists a function \(g(\Theta,L,S,\delta,n)\) such that, for any \(\delta\in(0,1]\), with
probability at least \(1-\delta\) over \(S\sim D^n\),
\[
\sup_{Q\in\Theta}\,\Omega_{L,Q,D,S}(n)\;\le\; g(\Theta,L,S,\delta,n).
\]
Fix \(\delta\in(0,1]\), \(\lambda>0\), and a prior \(P\) on \(\R^m\). Then, with probability at least \(1-\delta\) over \(S\sim D^n\), we have simultaneously for all
\(Q\in\Theta\),
\[
L_D(h_{\overline w})
\le
L_S(h_{\overline w})
+
g(\Theta,L,S,\delta/2,n)
+
\frac{1}{\lambda}\Big[
\KL(Q\|P)+\ln\tfrac{2}{\delta}+\Psi_{L,P,D}(\lambda,n)
\Big].
\]
\end{coro}
\noindent\textbf{Proof:} This is a direct union bound. \Endproof

Corollary~\ref{cor:replace-omega-by-g-delta-half} reduces the problem of controlling deterministic predictors to obtaining a uniform upper bound on $\Omega_{L,Q,D,S}(n)$ over a family of posteriors. Such a bound will be established in
Section~\ref{sec:jg-rademacher}. In the next definition, we formalize the key object that will control this penalty term: the Jensen gap.

\begin{de}
\label{def:jensen-gap}
Let \( f : \mathbb{R}^m \to \mathbb{R}^n \) be a measurable function and let \( Q \) be a probability distribution on \( \mathbb{R}^m \). We will refer to
\[
\jen_Q[f(w)] :=  f\left( \mathbb{E}_{w \sim Q}[w] \right) - \mathbb{E}_{w \sim Q}[f(w)] 
\]
as the \emph{Jensen gap} of \( f \) with respect to \( Q \).
The quantity $\|\jen_Q[f(w)]\|$ (or $|\jen_Q[f(w)]|$) is a measure of non-linearity of the function $f$. Note that \( f \) is not required to be convex. 
\end{de}

Combining Corollary \ref{cor:replace-omega-by-g-delta-half} with Proposition \ref{ah} and Proposition \ref{cor:jensen-gap-genbound} from Section~\ref{sec:jg-rademacher} leads to the following result:

\begin{thm}
\label{combinedThm}
Assume that \(L(z,y)\) is a real-valued loss function such that, for every \(y\), the map \(z\mapsto L(z,y)\) is \(\ell\)-Lipschitz and \(\beta\)-smooth, and that \(L(0,y)\) is independent of \(y\). Let $D$ be a distribution over $X\times Y$ and let $S=\{(x_i,y_i)\}_{i=1}^n\sim D^n$.

Let $\mathcal H=\{h_w:\ w\in\R^m\}$ be a hypothesis class.
Assume that for every $h_w\in\mathcal H$ there exists $a_w>0$ such that
\[
\|h_w(x)\|\le a_w
\qquad\text{for all }x\in X .
\]
Let
\[
\Theta := \{\overline w\in\R^m:\ \|\overline w\|_2\le R\},
\]
and fix $\sigma>0$. For each $\overline w\in\Theta$, define
\[
Q:=\mathcal N(\overline w,\sigma^2 I_m),
\qquad\text{so that}\qquad
\overline w=\E_{w\sim Q}[w].
\]
Assume that the hypotheses of Corollary~\ref{cor:rad-jensen-gap} (Section~\ref{sec:jg-rademacher}) hold,
and assume moreover that there exists $B>0$ such that, for all $(x,y)$ and all
$\overline w\in\Theta$,
\[
|\jen_{Q}\!\big[L(h_w(x),y)\big]|\le\sigma^2 B .
\]
Let $P$ be any prior distribution on $\R^m$, let $\delta\in(0,1]$, and let $\lambda>0$.
Then, with probability at least $1-\delta$ over $S\sim D^n$, we have uniformly for all
$\overline w\in\Theta$,
\[
\begin{aligned}
L_D(h_{\overline w})
&\;\le\;
L_S(h_{\overline w})
\;+\;
2c\,\frac{R\sigma}{\sqrt n}\,
\sqrt{\beta^2 J_{S,\sigma}+\ell^2 H_{S,\sigma}}
\;+\;
8\sigma^2 B\,\sqrt{\frac{2\ln(8/\delta)}{n}} \\
&\;+\;
\frac{1}{\lambda}\Bigg[
\KL(Q\|P)
+
\ln \bigg(
\frac{2}{\delta}
\mathbb{E}_{w \sim P}
\exp\!\left(
\frac{\lambda^2 \ell^2 a_w^2}{2n}
\right)
\bigg)
\Bigg].
\end{aligned}
\]

where $c>0$ is an absolute constant, and $J_{S,\sigma}$ and $H_{S,\sigma}$ are defined as in Corollary~\ref{cor:rad-jensen-gap} and are empirical measures of flatness depending on the hypothesis class.
\end{thm}

\begin{lem}[Evaluation of the PAC-Bayes complexity term for Gaussian $P$]
\label{lem:pacbayes-gaussian-term}
Let $P$ be a prior distribution on the full parameter space of a $T$-layer model,
denoted $W_{1:T}$. Let $W_T\in\R^{m_T}$ be the weights of the last linear layer, and
assume that the marginal distribution of $W_T$ under $P$ is
\[
W_T \sim \mathcal N(0,\sigma^2 I_{m_T}),
\qquad \sigma>0.
\]
Assume that there exists $M>0$ such that, for all $x\in X$,
\[
\|h_{W_{1:T}}(x)\|\le M\|W_T\|.
\]
This holds both in the linear case $h_{W_T}(x)=W_Tx$ with $\|x\|\le M$, and in the
neural-network case when the representation vector entering the last layer has norm upper bounded by $M$.
Let $\ell>0$, $n\ge1$, and choose
\[
\lambda=\frac{\sqrt n}{M\ell\sigma\sqrt{m_T+1}}.
\]
Then, with $a_{W_{1:T}}:=M\|W_T\|$, we have
\[
\frac{1}{\lambda}
\ln
\mathbb{E}_{W_{1:T}\sim P}
\exp\!\left(
\frac{\lambda^2\ell^2a_{W_{1:T}}^2}{2n}
\right)
\le
\frac{M\ell\sigma\sqrt{m_T+1}}{2\sqrt n}.
\]
\end{lem}

\noindent\textbf{Proof:}
The quantity inside the expectation depends on $W_{1:T}$ only through
$W_T$. Writing $P_T$ for the marginal distribution of $W_T$ under $P$, we have
\[
\mathbb{E}_{W_{1:T}\sim P}
\exp\!\left(
\frac{\lambda^2\ell^2a_{W_{1:T}}^2}{2n}
\right)
=
\mathbb{E}_{W_T\sim P_T}
\exp\!\left(
\frac{\lambda^2\ell^2M^2\|W_T\|^2}{2n}
\right).
\]
Let
\[
t=\frac{M^2\lambda^2\ell^2}{2n}.
\]
Since $W_T\sim \mathcal N(0,\sigma^2 I_{m_T})$, we have
\[
\mathbb{E}_{W_T\sim P_T}
\exp\!\left(t\|W_T\|^2\right)
=
(1-2t\sigma^2)^{-m_T/2},
\]
when $2t\sigma^2<1$.
For the chosen value
\[
\lambda=\frac{\sqrt n}{M\ell\sigma\sqrt{m_T+1}},
\]
we have
\[
2t\sigma^2=\frac{1}{m_T+1}.
\]
Therefore,
\[
\mathbb{E}_{W_{1:T}\sim P}
\exp\!\left(
\frac{\lambda^2\ell^2a_{W_{1:T}}^2}{2n}
\right)
=
\left(\frac{m_T+1}{m_T}\right)^{m_T/2}.
\]
Hence,
\[
\frac{1}{\lambda}
\ln
\mathbb{E}_{W_{1:T}\sim P}
\exp\!\left(
\frac{\lambda^2\ell^2a_{W_{1:T}}^2}{2n}
\right)
=
\frac{M\ell\sigma\sqrt{m_T+1}}{\sqrt n}
\frac{m_T}{2}\ln\left(1+\frac{1}{m_T}\right).
\]
Using
\[
m_T\ln\left(1+\frac{1}{m_T}\right)\le 1,
\]
we obtain
\[
\frac{1}{\lambda}
\ln
\mathbb{E}_{W_{1:T}\sim P}
\exp\!\left(
\frac{\lambda^2\ell^2a_{W_{1:T}}^2}{2n}
\right)
\le
\frac{M\ell\sigma\sqrt{m_T+1}}{2\sqrt n}.
\]
\Endproof

Lemma~\ref{lem:pacbayes-gaussian-term} turns the PAC-Bayes logarithmic term into an explicit expression when the prior is Gaussian on the last layer. We now illustrate how the full derandomized bound simplifies in the linear case.

\section{Specialization of the results to linear hypothesis classes}
\label{sec:linear-classes}

\begin{thm}[Linear case of Theorem~\ref{combinedThm}]
\label{combinedThmLin}
Assume that each hypothesis $h_W\in\mathcal H$ is given by $h_W(x)=Wx$ with
$W\in\R^{C\times d}$, and let $m:=Cd$.
Assume that $\|x\|\le M$ for all $x\in X$. Let \(L(z,y)\) be a real-valued loss function such that, for every \(y\), the map \(z\mapsto L(z,y)\) is \(\ell\)-Lipschitz and \(\beta\)-smooth, and assume that \(L(0,y)\) is independent of \(y\).
Let
\[
\Theta:=\{\,W\in\R^{C\times d}:\ \|W\|_F\le R\,\}.
\]
Fix $\sigma>0$ and let $\delta\in(0,1]$.
Then, with probability at least $1-\delta$ over $S\sim D^n$, we have uniformly
for all $W\in\Theta$,
\[
\begin{aligned}
L_D(h_W)
\;\le\;
& L_S(h_W)
\;+\;
2\sqrt{2}\,\frac{\beta R\sigma C M^2}{\sqrt n}
\;+\;
4\beta C M^2\sigma^2\sqrt{\frac{2\ln(8/\delta)}{n}} \\
&\;+\;
\frac{M\ell\sqrt{m+1}}{2\sigma\sqrt n}\,
\|W\|_F^2
\;+\;
\frac{M\ell\sigma\sqrt{m+1}}{\sqrt n}
\left(\ln\frac{2}{\delta}+\frac12\right).
\end{aligned}
\]
\end{thm}
\noindent \textbf{Proof:}
Let the prior $P$ and posterior $Q$ be isotropic Gaussian distributions with the
same covariance matrix $\sigma^2 I_{m\times m}$, with prior mean $0$ and posterior mean $W$.
Then,
\[
\KL(Q\|P)=\frac{\|W\|_F^2}{2\sigma^2}.
\]

We apply Theorem~\ref{combinedThm} to the linear hypothesis class $h_W(x)=Wx$ with
parameter set $\Theta=\{W\in\R^{C\times d}:\ \|W\|_F\le R\}$ and with $\|x\|\le M$.
The data-dependent complexity term appearing
in Theorem~\ref{combinedThm} satisfies
\[
2c\,\frac{R\sigma}{\sqrt n}\,
\sqrt{\beta^2 J_{S,\sigma}+\ell^2 H_{S,\sigma}}
\;\le\;
2\sqrt{2}\,\frac{\beta R\sigma C M^2}{\sqrt n}
\]
(see Remark \ref{remlin} and Theorem \ref{thm:rad-jensen-gap-linear-multiclass} in Section~\ref{sec:jg-rademacher}).

Moreover, specializing the Jensen-gap upper bound in Theorem \ref{jenUpperBound} from Section \ref{sec:jg-uniform-upper} to the linear case
($T=1$) yields
\[
|\jen_Q[L(h_W(x),y)]|
\le
\frac{\sigma^2}{2}\,\beta\,C\,M^2.
\]
Thus, the uniform assumption $|\jen_Q[L(h_W(x),y)]|\le\sigma^2 B$ in
Theorem~\ref{combinedThm} holds with
\[
B=\frac{\beta C M^2}{2}.
\]
The corresponding concentration term therefore becomes
\[
8\sigma^2 B\sqrt{\frac{2\ln(8/\delta)}{n}}
=
4\beta C M^2\sigma^2\sqrt{\frac{2\ln(8/\delta)}{n}}.
\]
Finally, using Lemma~\ref{lem:pacbayes-gaussian-term} and the above value of the KL divergence, the PAC-Bayes complexity term is bounded by
\[
\frac{M\ell\sqrt{m+1}}{2\sigma\sqrt n}\|W\|_F^2
+
\frac{M\ell\sigma\sqrt{m+1}}{\sqrt n}\ln\frac{2}{\delta}
+
\frac{M\ell\sigma\sqrt{m+1}}{2\sqrt n}.
\]
\Endproof

\begin{rem}
\label{RemOnCequal1}
    In the scalar-output case \(C=1\), Theorem~\ref{combinedThmLin} is also true but
can be sharpened by replacing Maurer's vector-contraction inequality with the
scalar contraction lemma in the analysis. This removes the \(\sqrt{2}\) constant in the Rademacher complexity term.
Consequently, the first complexity term in
Theorem~\ref{combinedThmLin} becomes
\[
2\,\frac{\beta R\sigma M^2}{\sqrt n}
\]
instead of
\[
2\sqrt{2}\,\frac{\beta R\sigma C M^2}{\sqrt n}.
\]
\end{rem}

We now want to apply this result to the multi-class unhinged loss function. We first define a proper surrogate version of the loss that can be used within our generalization bounds.

\begin{lem}
\label{unhSurrogate}
Let \(C>2\) and \( r > 0 \). Assume that \( \|z\| \leq r \). Define the multi-class unhinged surrogate as
\[
L_r^{\mathrm{unh}}(z, y) := a\left(-z_y + \frac{1}{C} \sum_{k=1}^C z_k\right) + b
\]
with
\[
a = \frac{\sqrt{C}}{r\big(\sqrt{C-1}-\sqrt{\frac{C-2}{2}}\big)}, 
\qquad
b = a r \sqrt{\frac{C-1}{C}} =\frac{1}{1-\sqrt{\frac{C-2}{2(C-1)}}}.
\]
Then, we have
\begin{enumerate}
\item \( \inf_{\|z\| \leq r} L_r^{\mathrm{unh}}(z, y) = 0 \),
\item \( L^{0\text{-}1}(z, y) \leq L_r^{\mathrm{unh}}(z, y) \),
\item The Lipschitz constant of \( L_r^{\mathrm{unh}}(z, y) \) is
\[
\frac{1}{r\left(1-\sqrt{\frac{C-2}{2(C-1)}}\right)}.
\]
\end{enumerate}
\end{lem}

\noindent \textbf{Proof:} Let us write
\[
L^{\mathrm{unh}}(z, y) = a\left(-z_y + \frac{1}{C} \sum_{k=1}^C z_k\right) + b
\]
for some constants \( a, b > 0 \).

We want \( \inf_{\|z\| \leq r} L^{\mathrm{unh}}(z, y) = 0 \). This infimum can be computed by standard Lagrangian optimization. The solution is
\[
\inf_{\|z\| \leq r} L^{\mathrm{unh}}(z, y) = -a r \sqrt{\frac{C-1}{C}} + b.
\]
To make the infimum equal to 0, we set:
\[
b = a r \sqrt{\frac{C-1}{C}}.
\]

We now turn our attention towards upper bounding the 0-1 loss. Suppose there exists \( k \neq y \) such that \( z_k \geq z_y \), so the 0-1 loss is $1$ at this value of $(z,y)$. We want to compute the infimum of $-z_y + \frac{1}{C} \sum_{k=1}^C z_k$ under the constraints $\|z\|\leq r$ and $z_k\geq z_y$. At a global minimizer, we must actually have $z_k=z_y$ (this can be seen by a simple perturbation argument). We may therefore restrict the optimization to the set
\[
B(0,r)\cap \{z:\, z_k=z_y\}.
\]

Both the objective and the constraint are invariant under permutations of the coordinates \(\{z_j\}_{j\neq y,k}\). Therefore, if \(z\) is a minimizer, any permutation of these coordinates is also a minimizer. Averaging over all such permutations gives another feasible point satisfying
\[
z_j=t \qquad (\forall j\neq y,k)
\]
for some \(t\in\mathbb{R}\), and
\[
z_y=z_k=s
\]
for some \(s\in\mathbb{R}\). The optimization problem thus reduces to
\[
\min_{s,t\in\mathbb{R}}
\frac{C-2}{C}(t-s)
\]
subject to
\[
2s^2+(C-2)t^2\le r^2.
\]
The solution of this problem can be obtained using Lagrange multipliers. One obtains
\[
t-s = -r\sqrt{\frac{C}{2(C-2)}}.
\]
Therefore,
\[
\inf_{\substack{\|z\|\le r\\ z_k\ge z_y}}
\left(-z_y+\frac1C\sum_{i=1}^C z_i\right)
=
-r\sqrt{\frac{C-2}{2C}}
\]
and
\[
L^{\mathrm{unh}}(z,y)
\ge
-a r\sqrt{\frac{C-2}{2C}}+b.
\]
Substituting \( b = a r \sqrt{\frac{C-1}{C}} \), we obtain
\[
L^{\mathrm{unh}}(z, y)
\ge
a r
\left(
\sqrt{\frac{C-1}{C}}
-
\sqrt{\frac{C-2}{2C}}
\right).
\]
Therefore, \(L^{0\text{-}1}(z,y) \le L^{\mathrm{unh}}(z,y)\) provided
\[
a \ge
\frac{1}{r\left(\sqrt{\frac{C-1}{C}}-\sqrt{\frac{C-2}{2C}}\right)}
=
\frac{\sqrt{C}}{r\big(\sqrt{C-1}-\sqrt{\frac{C-2}{2}}\big)}.
\]
With the choice of \(a\) in the statement of the Lemma, we thus indeed have
\[
L^{0\text{-}1}(z, y) \leq L^{\mathrm{unh}}(z, y).
\]
Finally, the gradient of the function 
$
-z_y + \frac{1}{C} \sum_{k=1}^C z_k
$
has norm
$
\sqrt{\frac{C-1}{C}},
$
so the Lipschitz constant of \( L^{\mathrm{unh}} \) is
\[
\sup_z\| \nabla_z L^{\mathrm{unh}}(z,y) \|
=
a\sqrt{\frac{C-1}{C}} .
\]
Simplifying concludes the proof.
\Endproof

The multi-class unhinged loss is typically studied in the context of learning with uniform label noise; see, for example, \citet{lemirepaquin2026symmetrization}. We therefore state below a result for this setting.

\begin{prop}
\label{unhprop}
Let \(C>2\) be the number of classes. Define
\[
b_C:=\frac{1}{1-\sqrt{\frac{C-2}{2(C-1)}}}.
\]
Assume that \( \|x\| \leq M \).
Furthermore, assume that each hypothesis \( h_W \in \mathcal{H} \) is given by \( h_W(x) = Wx \), where \( W \in \mathbb{R}^{C \times d} \). We will simplify notations by replacing $h_W$ with $W$.
Let $\mu^{unh}_S$ be the unhinged multi-class data centroid (see \citet{lemirepaquin2026symmetrization}), defined by:
\begin{center}
    $\mu^{unh}_S:=\frac{1}{N}\sum_{i=1}^N c_{y_i}x_i^\top$,
\end{center}
where $c_y$ is a column vector having its $y^{th}$ entry given by $\frac{C-1}{C}$ and every other entry given by $\frac{-1}{C}$. Given a distribution $D$ over $X \times Y$, a probability of corruption $p \in [0, 1)$ and a confidence parameter $\delta \in (0, 1]$,  with probability at least $1 - \delta$ over the choice of $\overline{S} \sim \overline{D}^n$, we have:
\[
\begin{aligned}
L^{0\text{-}1}_{D}(\mu^{unh}_{\overline{S}}) 
&\leq b_C\bigg[1-\sqrt{\frac{C}{C-1}}\frac{\|\mu^{\mathrm{unh}}_{\overline{S}}\|}{(1-p)M}\bigg]
+ \tilde{O}\bigg(\frac{b_C\sqrt{dC}}{(1-p)\sqrt{n}}\bigg).
\end{aligned}
\]
\end{prop}
\noindent \textbf{Proof:} Consider the multi-class unhinged surrogate $L_r^{\mathrm{unh}}(z, y)$ with $r=M$. We will omit the subscript $r$ to simplify notation. Then, for any $\overline{S}$,
\[\begin{aligned}
    L^{0\text{-}1}_{D}(\mu^{unh}_{\overline{S}})=L^{0\text{-}1}_{D}\bigg(\frac{\mu^{unh}_{\overline{S}}}{\|\mu^{unh}_{\overline{S}}\|}\bigg)&\leq L^{\mathrm{unh}}_D\bigg(\frac{\mu^{unh}_{\overline{S}}}{\|\mu^{unh}_{\overline{S}}\|}\bigg)\\
    &=\frac{1}{1-p}\bigg[L^{\mathrm{unh}}_{\overline{D}}\bigg(\frac{\mu^{unh}_{\overline{S}}}{\|\mu^{unh}_{\overline{S}}\|}\bigg)-\frac{p}{C}\mathop{\mathbb{E}}_{x\sim D_x }\sum_{y=1}^C L^{\mathrm{unh}}\bigg(\frac{\mu^{unh}_{\overline{S}}}{\|\mu^{unh}_{\overline{S}}\|}x,y\bigg)\bigg]\\
    &=\frac{1}{1-p}\bigg[L^{\mathrm{unh}}_{\overline{D}}\bigg(\frac{\mu^{unh}_{\overline{S}}}{\|\mu^{unh}_{\overline{S}}\|}\bigg)-p b_C\bigg]. 
\end{aligned}\]
For any $0<\sigma<1$ and with $m=dC$, applying Theorem \ref{combinedThmLin} with $\beta=0$ and $\ell=b_C/M$, we obtain that with probability at least \( 1 - \delta \) over the draw \( \overline{S} \sim \overline{D}^n \), 
\[\displaystyle
\begin{aligned}
L^{0\text{-}1}_{D}(\mu^{unh}_{\overline{S}}) 
&\leq 
\frac{1}{1-p}\bigg[L^{\mathrm{unh}}_{\overline{S}}\bigg(\frac{\mu^{unh}_{\overline{S}}}{\|\mu^{unh}_{\overline{S}}\|}\bigg)-p b_C\bigg]  
+ \frac{b_C\sqrt{dC+1}}{2(1-p)\sigma\sqrt{n}}  \\
&\quad+
\frac{b_C\sigma\sqrt{dC+1}}{(1-p)\sqrt{n}}
\left(\ln\frac{2}{\delta}+\frac12\right).
\end{aligned}
\]
The proof is completed by exploiting
\[
L_{\overline{S}}^{\mathrm{unh}}(W) = -\frac{b_C}{M}\sqrt{\frac{C}{C-1}}\tr(\mu^{unh}_{\overline{S}}W^T) + b_C
\]
with \(W=\frac{\mu^{unh}_{\overline{S}}}{\|\mu^{unh}_{\overline{S}}\|}\) and simplifying. \Endproof

We can also apply our results to the problem of linear regression.

\begin{coro}[Linear regression with absolute loss]
\label{cor:absolute-loss-linear}
Let 
\(h_w(x)=\langle w,x\rangle\), \(\|x\|\le M\) and
\[
L(z,y)=|z-y|.
\]
Assume that \(|y|\le Y_{\max}\) almost surely.
Let \(P=\mathcal N(0,\sigma^2 I_d)\) with \(\sigma>0\).
Then, with probability at least \(1-\delta\) over \(S\sim D^n\), for all
probability distributions \(Q\) on \(\mathbb R^d\),
\[
\begin{aligned}
\E_{w\sim Q}L_D(h_w)
&\le
\E_{w\sim Q}L_S(h_w)
+
\frac{M\sigma\sqrt{d+1}}{\sqrt n}
\left[
\KL(Q\|P)
+
\ln\frac{2}{\delta}
+
\frac12
\right]
+
Y_{\max}\sqrt{\frac{\ln(2/\delta)}{2n}}.
\end{aligned}
\]
In particular, the explicit Gaussian-prior complexity term has a square root dependence on the dimension \(d\).
\end{coro}

\noindent\textbf{Proof.}
The absolute loss is \(1\)-Lipschitz in its prediction argument and Proposition~\ref{ah} applies to $\widetilde{L}(z,y)=L(z,y)-L(0,y)$ with
\(a_w=M\|w\|\). Applying Theorem~\ref{Alquier} and
Lemma~\ref{lem:pacbayes-gaussian-term} to \(\widetilde L\), with confidence parameter \(\delta/2\), we obtain that, with probability at least \(1-\delta/2\) over \(S\sim D^n\), for all probability distributions \(Q\) on \(\mathbb R^d\),
\[
\begin{aligned}
\E_{w\sim Q}L_D(h_w)
&\le
\E_{w\sim Q}L_S(h_w)
+
\frac{M\sigma\sqrt{d+1}}{\sqrt n}
\left[
\KL(Q\|P)
+
\ln\frac{2}{\delta}
+
\frac12
\right]
+
\E_D[|y|]-\frac1n\sum_{i=1}^n |y_i|.
\end{aligned}
\]
Applying Hoeffding's inequality with a union bound completes the proof. \Endproof

Since the absolute loss is not smooth, our derandomized framework does not apply to it directly. Nevertheless, it can be
approached by the following smooth family of Huber losses, which converges uniformly to the absolute loss as the smoothing parameter
tends to zero.

\begin{coro}[Linear regression with Huber losses]
\label{cor:normalized-huber-loss-linear}
Let \(h_w(x)=\langle w,x\rangle\), assume that \(\|x\|\le M\), and fix
\(\rho>0\). Define the Huber loss by
\[
L_\rho(z,y)
:=
\begin{cases}
\dfrac{(z-y)^2}{2\rho},
& \text{if } |z-y|\le \rho,\\[1.2ex]
|z-y|-\dfrac{\rho}{2},
& \text{if } |z-y|>\rho.
\end{cases}
\]
Assume that \(|y|\le Y_{\max}\) almost surely and define
\[
Y_{\rho,\max}
:=
\sup_{|y|\le Y_{\max}} L_\rho(0,y).\]
Let
\[
\Theta:=\{w\in\mathbb R^d:\|w\|\le R\}
\]
and let \(P=\mathcal N(0,\sigma^2 I_d)\) with \(\sigma>0\).
Then, with probability at least \(1-\delta\) over \(S\sim D^n\), we have
uniformly for all \(w\in\Theta\),
\[
\begin{aligned}
L_{\rho,D}(h_w)
\le\;&
L_{\rho,S}(h_w)
+
\frac{2\,R\sigma M^2}{\rho\sqrt n}
+
\frac{4M^2\sigma^2}{\rho}
\sqrt{\frac{2\ln(16/\delta)}{n}}
+
Y_{\rho,\max}\sqrt{\frac{\ln(2/\delta)}{2n}}
\\
&+
\frac{M\sqrt{d+1}}{2\sigma\sqrt n}\|w\|^2
+
\frac{M\sigma\sqrt{d+1}}{\sqrt n}
\left(\ln\frac{4}{\delta}+\frac12\right).
\end{aligned}
\]
\end{coro}

\noindent \textbf{Proof:}
For every \(y\), the map \(z\mapsto L_\rho(z,y)\) is \(1\)-Lipschitz and
\(1/\rho\)-smooth. We can apply
Theorem~\ref{combinedThmLin} with \(C=1\) (see Remark \ref{RemOnCequal1}), \(\ell=1\) and \(\beta=1/\rho\) to $\widetilde{L}_{\rho}(z,y)=L_{\rho}(z,y)-L_{\rho}(0,y)$. We then obtain that, with probability at least \(1-\delta/2\) over \(S\sim D^n\), uniformly for all \(w\in\Theta\),
\[
\begin{aligned}
L_{\rho,D}(h_w)
\le\;&
L_{\rho,S}(h_w)
+
\frac{2\,R\sigma M^2}{\rho\sqrt n}
+
\frac{4M^2\sigma^2}{\rho}
\sqrt{\frac{2\ln(16/\delta)}{n}}
+
\E_D[L_{\rho}(0,y)]-\frac1n\sum_{i=1}^n L_{\rho}(0,y_i)
\\
&+
\frac{M\sqrt{d+1}}{2\sigma\sqrt n}\|w\|^2
+
\frac{M\sigma\sqrt{d+1}}{\sqrt n}
\left(\ln\frac{4}{\delta}+\frac12\right).
\end{aligned}
\]
Applying Hoeffding's inequality with a union bound completes the proof.
\Endproof

\section{Generalization bounds for bounded loss functions from the PAC-Bayes framework}
\label{sec:pb-bounded}
When the loss is uniformly bounded, one can instead invoke standard bounded-loss PAC-Bayes inequalities, leading to a cleaner square root complexity term. Using the same proofs as in the case of unbounded loss functions, but replacing Theorem \ref{Alquier} with Theorem $31.1$ from \cite{ShalevShwartz2014}, we get the following result for bounded loss functions:
\begin{thm}
\label{bounded}
Assume that the loss \(L(z,y)\) takes values in an interval \([A',A]\subset\mathbb{R}\) on the domain \(\mathbb{R}^C\), and that, for every \(y\), the map \(z\mapsto L(z,y)\) is \(\ell\)-Lipschitz and \(\beta\)-smooth.
Let \(D\) be a distribution over \(X\times Y\) and let \(S=\{(x_i,y_i)\}_{i=1}^n\sim D^n\).
Let \(\mathcal H=\{h_w:\ w\in\R^m\}\) be a hypothesis class.
Let
\[
\Theta := \{\overline w\in\R^m:\ \|\overline w\|_2\le R\},
\]
and fix \(\sigma>0\). For any \(\overline w\in\Theta\), define a posterior $Q:=\mathcal N(\overline w,\sigma^2 I_m)$.
Assume that the hypotheses of Corollary~\ref{cor:rad-jensen-gap} hold,
and assume moreover that there exists \(B>0\) such that, for all \((x,y)\) and all
\(\overline w\in\Theta\),
\[
|\jen_{Q}\!\big[L(h_w(x),y)\big]|\le\sigma^2 B .
\]
Let \(P\) be any prior distribution on \(\R^m\), and let \(\delta\in(0,1)\).
Then, with probability at least \(1-\delta\) over \(S\sim D^n\), we have uniformly for all
\(\overline w\in\Theta\),
\[
\begin{aligned}
L_D(h_{\overline w})
\;\le\;&
L_S(h_{\overline w})
\;+\;
2c\,\frac{R\sigma}{\sqrt n}\,
\sqrt{\beta^2 J_{S,\sigma}+\ell^2 H_{S,\sigma}}
\;+\;
8\sigma^2 B\,\sqrt{\frac{2\ln(8/\delta)}{n}}
\\
&\;+\;
(A-A')\sqrt{\frac{\KL(Q\|P)+\ln\!\big(\tfrac{2n}{\delta}\big)}{2(n-1)}}
\;,
\end{aligned}
\]
where \(c>0\) is an absolute constant, and \(J_{S,\sigma}\) and \(H_{S,\sigma}\) are defined as in
Corollary~\ref{cor:rad-jensen-gap}.
\end{thm}
Similarly to Theorem \ref{combinedThmLin}, we can obtain a straightforward specialization to the linear case:
\begin{thm}[Linear case of Theorem~\ref{bounded}]
\label{boundedLin}
Assume that each hypothesis $h_W\in\mathcal H$ is given by $h_W(x)=Wx$ with
$W\in\R^{C\times d}$, and let $m:=Cd$.
Assume that $\|x\|\le M$ for all $x\in X$.
Let $L(z,y)$ be a loss function taking values in an interval $[A',A]\subset\mathbb{R}$ on the domain $\mathbb{R}^C$, and assume that, for every $y$, the map $z\mapsto L(z,y)$ is $\ell$-Lipschitz and $\beta$-smooth.
Let
\[
\Theta:=\{\,W\in\R^{C\times d}:\ \|W\|_F\le R\,\}.
\]
Fix $\sigma>0$ and let $\delta\in(0,1)$.
Then, with probability at least $1-\delta$ over $S\sim D^n$, we have uniformly
for all $W\in\Theta$,
\[
\begin{aligned}
L_D(h_W)
\;\le\;
L_S(h_W)
&\;+\;
2\sqrt{2}\,\frac{\beta R\sigma C M^2}{\sqrt n}
\;+\;
4\beta C M^2\sigma^2\sqrt{\frac{2\ln(8/\delta)}{n}}
\\
&\;+\;
(A-A')\sqrt{
\frac{
\frac{\|W\|_F^2}{2\sigma^2}
+\ln\!\big(\tfrac{2n}{\delta}\big)
}{
2(n-1)
}
}\;.
\end{aligned}
\]
\end{thm}

 One can alternatively invoke Theorem~5 of \citet{Maurer2004}, which provides an
inverse-\(\klb\) PAC-Bayesian bound for bounded losses. This may lead to a
sharper complexity term, but at the cost of an implicit final expression that is
less transparent.  We obtain the following variant of Theorem~\ref{bounded}.
\begin{thm}[Bounded-loss PAC-Bayes bound in \(\klb^{-1}\) form, uniform over \(\{\sigma_i\}_{i\ge1}\)]
\label{thm:bounded-kl-regrouped-rademacher-sigmai}
Assume that the loss \(L(z,y)\) takes values in an interval
\([A',A]\subset\mathbb{R}\) on the domain \(\mathbb{R}^C\), and that, for every
\(y\), the map \(z\mapsto L(z,y)\) is \(\ell\)-Lipschitz and \(\beta\)-smooth.
Let \(D\) be a distribution over \(X\times Y\), let
\(S=\{(x_j,y_j)\}_{j=1}^n\sim D^n\), with \(n\ge 8\), and let
\[
\Theta:=\{\overline w\in\mathbb R^m:\ \|\overline w\|_2\le R\}.
\]
Fix \(\sigma_{\max}>0\) and for \(i\ge 1\) define
\[
\sigma_i:=2^{\,1-i}\sigma_{\max}.
\]
For \(\overline w\in\Theta\), let
\[
Q_{\overline w,\sigma_i}:=\mathcal N(\overline w,\sigma_i^2 I_m).
\]
Assume the same conditions as in Corollary~\ref{cor:rad-jensen-gap}. Assume
moreover that there exists \(B>0\) such that, for all \((x,y)\), all
\(\overline w\in\Theta\) and all \(i\ge 1\),
\[
\left|\jen_{Q_{\overline w,\sigma_i}}\!\big[L(h_w(x),y)\big]\right|
\le \sigma_i^2 B.
\]
Let \(P\) be any prior distribution on \(\mathbb R^m\) and let
\(\delta\in(0,1)\). Then, with probability at least \(1-\delta\) over
\(S\sim D^n\), we have uniformly for all \(i\ge 1\) and all
\(\overline w\in\Theta\),
\[
\begin{aligned}
L_D(h_{\overline w})
\le\;&
2c\,\frac{R\sigma_i}{\sqrt n}\,
\sqrt{\beta^2 J_{S,\sigma_i}+\ell^2 H_{S,\sigma_i}}
+
8\sigma_i^2 B\,
\sqrt{\frac{4\ln(i)+2\ln(16/\delta)}{n}}
+
A'
+
\jen_{Q_{\overline w,\sigma_i}}\!\big[L_S(h_w)\big]
\\
&\;+\;
(A-A')\,
\klb^{-1}\!\left(
\frac{
\E_{w\sim Q_{\overline w,\sigma_i}}L_S(h_w)
-
A'
}{
A-A'
}
\;\middle|\;
\frac{
\KL(Q_{\overline w,\sigma_i}\|P)
+
\ln\!\big(\tfrac{4\sqrt n}{\delta}\big)
}{
n
}
\right),
\end{aligned}
\]
where
\[
\klb^{-1}(q\mid \varepsilon)
:=
\sup\Bigl\{p\in[q,1]:\ \klb(q\|p)\le \varepsilon\Bigr\}
\]
with
\[
\klb(q\|p)
:=
q\ln\!\frac{q}{p}
+
(1-q)\ln\!\frac{1-q}{1-p}.
\]
\end{thm}
\noindent\textbf{Proof:}
Apply Theorem~5 of \citet{Maurer2004} to the rescaled loss
\((L-A')/(A-A')\) and then use Proposition~\ref{prop:jensen-gap-genbound-dyadic} with a union bound
after rearranging terms.
\Endproof

Our Rademacher complexity bounds for the Jensen gap class require a uniform
upper bound on the Jensen gap. In the next section, we provide such bounds for smooth neural networks, treating separately the cases of bounded and unbounded activation functions.
 
\section{Uniform upper bounds for the Jensen gap for smooth neural networks}\label{sec:jg-uniform-upper}
We start with a general result for bounding the Jensen gap of the composition of a function $f$ with a neural network $h_w(x)$ exploiting the Lipschitz constant and the $\beta$-smoothness of the function $f$.

\begin{lem}
\label{JensenLemmaNN}
Assume that the hypothesis class is parameterized by $w\in \mathbb{R}^m$ and \( Q \) is a probability distribution on \( \mathbb{R}^m \). Denote by $h_w$ a neural network mapping inputs \( x \in \mathbb{R}^d \) to  vectors in $\mathbb{R}^H$. Let \( f : \mathbb{R}^H \to \mathbb{R}^K \) (or \( f : \mathbb{R}^H \times Y \to \mathbb{R}^K \) if $f$ is a loss function) be differentiable. Furthermore, assume that \( f \) is \( \ell \)-Lipschitz and \( \beta \)-smooth (for all $y\in Y$ if there is a dependency on $y$). We will omit the variable $y$ in order to simplify notations in what follows. The property of $\beta$-smoothness means that the Jacobian \( \nabla f\in \mathbb{R}^{K\times H} \) is \( \beta \)-Lipschitz with respect to the spectral norm $\|\cdot\|_2$, that is,
    \[
    \| \nabla f(z) - \nabla f(z') \|_2 \leq \beta \| z - z' \|, \quad \text{for all } z, z' \in \mathbb{R}^H.
    \]
For any fixed input \( x \in \mathbb{R}^d \), $f(h_w(x))$ is a function of $w$. Then, we have that the Jensen gap of $f(h_w(x))$ with respect to $Q$ satisfies
\[
\|\jen_Q[f(h_w(x))]\|
\leq \ell  \|\jen_Q[h_w(x)]\| + \frac{\beta}{2}  \var_Q[h_w(x)],
\]
where $\var_Q[h_w(x)]$ denotes the total variance of the vector $h_w(x)$ under distribution $Q$ (trace of the covariance matrix).
\end{lem}

\noindent \textbf{Proof:} Let $\overline{w}=\mathbb{E}_{w \sim Q}[w]$ and $\overline{z}=\mathbb{E}_{w \sim Q}[h_w(x)]$. Then,
\begin{align*}
 \|\jen_Q[f(h_w(x))]\|&= \| f(h_{\overline{w}}(x)) - \mathbb{E}_{w \sim Q}[f(h_w(x))] \|\\
 &\leq \| f(h_{\overline{w}}(x)) - f(\overline{z}) \|+\|f(\overline{z})-  \mathbb{E}_{w \sim Q}[f(h_w(x))] \|. 
\end{align*}
For the first term, the \( \ell \)-Lipschitz property of \( f \) gives
\[
\| f(h_{\overline{w}}(x)) - f(\overline{z}) \| \leq \ell  \| h_{\overline{w}}(x) - \overline{z} \| = \ell  \|\jen_Q[h_w(x)]\|.
\]
For the second term, we apply the integral form of first-order Taylor's theorem to \( f(z) \) around \( \overline{z} \):
\[
f(z) = f(\overline{z}) +\nabla f(\overline{z}) (z- \overline{z})+ \int_0^1 \left( \nabla f(\overline{z} + t(z - \overline{z})) - \nabla f(\overline{z}) \right)(z - \overline{z}) \, dt.
\]
Taking expectation over $w$ on both sides with $z=h_w(x)$, we obtain
\[
\mathbb{E}_Q[f(h_w(x))] = f(\overline{z}) + \mathbb{E}_{w\sim Q}\left[ \int_0^1 \left( \nabla f(\overline{z} + t(h_w(x) - \overline{z})) - \nabla f(\overline{z}) \right)(h_w(x) - \overline{z}) \, dt \right],
\]
Using the \( \beta \)-smoothness of \( f \), we get
\[
\| \nabla f(\overline{z} + t(h_w(x) - \overline{z})) - \nabla f(\overline{z}) \|_2 \leq \beta t \| h_w(x) - \overline{z} \|.
\]
Thus,
\[
\| f(\overline{z}) - \mathbb{E}_Q[f(h_w(x))] \| \leq \mathbb{E}_{w\sim Q}\int_0^1 \beta t  \| h_w(x) - \overline{z} \|^2  \, dt = \frac{\beta}{2} \var_Q[h_w(x)].
\]
\Endproof

Our proof for bounding the variance term for neural networks will be in the form of an inductive argument on the number of layers and will exploit the following lemma.
\begin{lem}
\label{varLemma}
    Assume that $f:\mathbb{R}^m\xrightarrow{} \mathbb{R}^k$ is \( \ell \)-Lipschitz. Then, 
    \[\var[f(x)]\leq \ell^2 \var[x].\]
\end{lem}
\noindent \textbf{Proof:} On one hand, 
\[
\begin{aligned}
 \var[f(x)]=\mathbb{E}_x [\|f(x)-\mathbb{E}_xf(x)\|^2]&=\mathbb{E}_x [\|f(x)\|^2-2(f(x)\cdot \mathbb{E}_xf(x))+\|\mathbb{E}_xf(x)\|^2]\\
&=\mathbb{E}_x [\|f(x)\|^2]- \|\mathbb{E}_xf(x)\|^2.
\end{aligned}\]
On the other hand, if $x'$ follows the same distribution as $x$ but is independent from $x$,
\[
\begin{aligned}
 \mathbb{E}_x \mathbb{E}_{x'}[\|f(x)-f(x')\|^2]&=\mathbb{E}_{x'}\mathbb{E}_x [\|f(x)\|^2-2\mathbb{E}_{x'}\mathbb{E}_{x}(f(x)\cdot f(x'))+\mathbb{E}_{x'}\mathbb{E}_{x}\|f(x')\|^2]\\
&=2\mathbb{E}_x [\|f(x)\|^2]- 2\|\mathbb{E}_xf(x)\|^2.
\end{aligned}\]
Therefore, \[\var[f(x)]=\frac{1}{2}\mathbb{E}_x \mathbb{E}_{x'}[\|f(x)-f(x')\|^2].\]
We can now exploit the Lipschitz property to conclude the proof:
\[
\begin{aligned}
\var[f(x)]=\frac{1}{2}\mathbb{E}_x \mathbb{E}_{x'}[\|f(x)-f(x')\|^2]\leq \frac{\ell^2}{2}\mathbb{E}_x \mathbb{E}_{x'}[\|x-x'\|^2]=\ell^2 \var[x].
\end{aligned}
\]
\Endproof
We are now ready to upper bound the variance of the outputs of a neural network.
\begin{prop}[Bounded activation functions]
\label{varbound}
    Consider a feedforward fully connected neural network with $T$ layers  $h_{W_{1:T}}(x)=W_Tg_{T-1}(W_{T-1}g_{T-2}(\cdots g_1(W_1x)))$ where each non-linear function $g_t:\mathbb{R}^{d_t}\xrightarrow[]{}\mathbb{R}^{d_t}$ is \( \ell_t \)-Lipschitz. We further assume that these functions  are bounded: $\|g_t(z)\|\leq M_t$ for all $z\in \mathbb{R}^{d_t}$. We let $d_T:=C$ be the number of classes. Assume that $\|x\|\leq M_0$ and let $Q$ be an isotropic Gaussian distribution centered at $\overline{W}_1,\cdots,\overline{W}_T$ with covariance matrix given by $\sigma^2$ times the identity matrix. Then,
    \[\var_Q[h_{W_{1:T}}(x)]\leq \sigma^2\sum_{t=1}^Td_tM_{t-1}^2\prod_{t'=t}^{T-1}\ell_{t'}^2\|\overline{W}_{t'+1}\|_2^2,\]
    where the product is equal to $1$ if $t=T$.
\end{prop}
\noindent \textbf{Proof:} The proof is done by induction on the number of layers $T$. For $T=1$, we have $h_{W_1}(x)=W_1x$. If $w_i$ denotes the $i^{\text{th}}$ row of $W_1$, then $w_i\sim\mathcal N(\overline w_i,\sigma^2 I)$ and therefore $\var(w_i^Tx)=\sigma^2\|x\|^2$. Summing over the $d_1$ output coordinates gives
\[
\var_Q[h_{W_1}(x)]=\sum_{i=1}^{d_1}\var(w_i^Tx)=d_1\sigma^2\|x\|^2\le d_1\sigma^2 M_0^2,
\]
which matches the claimed bound since the product in the statement equals $1$ when $T=1$. Assume that the result is true for $T-1$ layers.
We have
\[\var_Q[h_{W_{1:T}}(x)]= \mathbb{E}_{W_{1:T-1}}\mathbb{E}_{W_T}[\|W_T g_{T-1}(h_{W_{1:T-1}}(x))-\overline{W}_T \mathbb{E}_{W_{1:T-1}}[g_{T-1}(h_{W_{1:T-1}}(x))]\|^2].\]
To simplify notations, write $W:=W_T$, $u:=g_{T-1}(h_{W_{1:T-1}}(x))$, $\overline{W}:=\overline{W}_T$ and \\$\overline{u}:=\mathbb{E}_{W_{1:T-1}}[g_{T-1}(h_{W_{1:T-1}}(x))]$. Then, 
\[\mathbb{E}_{W}[\|W u-\overline{W} \overline{u}\|^2]=\sum_{i}\mathbb{E}_{w_i}[(w_i^T u-\overline{w}_i^T \overline{u})^2]=\sum_{i}\mathbb{E}_{w_i}[(w_i^T u)^2-2(w_i^T u)(\overline{w}_i^T \overline{u})+(\overline{w}_i^T \overline{u})^2],\] where $w_i$ is the vector corresponding to the $i^{th}$ row of $W$ and $\overline{w}_i$ is the vector corresponding to the $i^{th}$ row of $\overline{W}$. Since $\mathbb{E}_{w_i}[(w_i^T u)^2]=\sigma^2\|u\|^2+(\overline{w}_i^T u)^2$, we get
\[\begin{aligned}\mathbb{E}_{W}[\|W u-\overline{W} \overline{u}\|^2]&=\sum_{i}\big[\sigma^2\|u\|^2+(\overline{w}_i^T u)^2-2(\overline{w}_i^T u)(\overline{w}_i^T \overline{u})+(\overline{w}_i^T \overline{u})^2\big]\\
&= d_T\sigma^2\|u\|^2+\sum_i(\overline{w}_i^T u-\overline{w}_i^T \overline{u})^2\\
&= d_T\sigma^2\|u\|^2+\|\overline{W}u-\overline{W}\overline{u}\|^2\\
&\leq d_T\sigma^2M_{T-1}^2 +\|\overline{W}\|_2^2\|u-\overline{u}\|^2.\end{aligned} \]
Therefore,
\[\var_Q[h_{W_{1:T}}(x)]\leq d_T\sigma^2M_{T-1}^2 +\|\overline{W}_T\|_2^2\var_Q[u]\leq d_T\sigma^2M_{T-1}^2 +\|\overline{W}_T\|_2^2 \ell_{T-1}^2\var_Q[h_{W_{1:T-1}}(x)],\]
where we used Lemma \ref{varLemma} for the last inequality. Using the induction hypothesis, we get
\[\begin{aligned}\var_Q[h_{W_{1:T}}(x)]&\leq d_T\sigma^2M_{T-1}^2 +\sigma^2\|\overline{W}_T\|_2^2 \ell_{T-1}^2\sum_{t=1}^{T-1}d_tM_{t-1}^2\prod_{t'=t}^{T-2}\ell_{t'}^2\|\overline{W}_{t'+1}\|_2^2\\
&=d_T\sigma^2M_{T-1}^2 +\sigma^2\sum_{t=1}^{T-1}d_tM_{t-1}^2\prod_{t'=t}^{T-1}\ell_{t'}^2\|\overline{W}_{t'+1}\|_2^2\\
&=\sigma^2\sum_{t=1}^{T}d_tM_{t-1}^2\prod_{t'=t}^{T-1}\ell_{t'}^2\|\overline{W}_{t'+1}\|_2^2.\end{aligned}\]
\Endproof 

\begin{prop}[Unbounded activation functions]
\label{varbound-unbounded}
Consider a feedforward fully connected neural network with \(T\) layers
\[
h_{W_{1:T}}(x)=W_Tg_{T-1}(W_{T-1}g_{T-2}(\cdots g_1(W_1x))),
\]
where each non-linear function \(g_t:\mathbb{R}^{d_t}\to\mathbb{R}^{d_t}\) is \(\ell_t\)-Lipschitz and satisfies \(g_t(0)=0\). We let \(d_T:=C\) be the number of classes. Assume that \(\|x\|\leq M_0\), and let \(Q\) be an isotropic Gaussian distribution centered at \(\overline{W}_1,\cdots,\overline{W}_T\) with covariance matrix given by \(\sigma^2\) times the identity matrix. Then,
\[
\var_Q[h_{W_{1:T}}(x)]
\leq
\sigma^2 M_0^2
\sum_{t=1}^{T} d_t
\left(
\prod_{t'=1}^{t-1}\ell_{t'}^2\bigl(\sigma^2 d_{t'}+\|\overline{W}_{t'}\|_2^2\bigr)
\right)
\left(
\prod_{t'=t}^{T-1}\ell_{t'}^2\|\overline{W}_{t'+1}\|_2^2
\right),
\]
where the first product is equal to $1$ if $t=1$ and the second product is equal to $1$ if $t=T$.
\end{prop}

\noindent \textbf{Proof:} Define recursively \[
u_0:=x,
\qquad
u_t:=g_t(W_tu_{t-1}).
\] 
The proof is almost identical to that of Proposition~\ref{varbound}, but \(M_{t-1}^2\) is replaced by
\[
\mathbb{E}\|u_{t-1}\|^2
\le
M_0^2
\prod_{t'=1}^{t-1}\ell_{t'}^2\bigl(\sigma^2 d_{t'}+\|\overline{W}_{t'}\|_2^2\bigr).
\]
This inequality is true since
\[
\begin{aligned}
\mathbb{E}\|u_{t-1}\|^2
&\le
\ell_{t-1}^2\,
\mathbb{E}\Big[\mathbb{E}_{W_{t-1}}\big[\|W_{t-1}u_{t-2}\|^2 \big]\Big] \\
&=
\ell_{t-1}^2\,
\mathbb{E}\Big[
d_{t-1}\sigma^2\|u_{t-2}\|^2+\|\overline{W}_{t-1}u_{t-2}\|^2
\Big] \\
&\le
\ell_{t-1}^2\,
\mathbb{E}\Big[
\bigl(\sigma^2 d_{t-1}+\|\overline{W}_{t-1}\|_2^2\bigr)\|u_{t-2}\|^2
\Big] \\
&=
\ell_{t-1}^2\bigl(\sigma^2 d_{t-1}+\|\overline{W}_{t-1}\|_2^2\bigr)\,
\mathbb{E}\|u_{t-2}\|^2 \\
&\le \cdots \\
&\le
M_0^2
\prod_{t'=1}^{t-1}\ell_{t'}^2
\bigl(\sigma^2 d_{t'}+\|\overline{W}_{t'}\|_2^2\bigr).
\end{aligned}
\]
\Endproof
Propositions~\ref{varbound} and \ref{varbound-unbounded} control the variance term appearing in Lemma~\ref{JensenLemmaNN} for deep networks. The next two theorems combine this with layerwise smoothness to yield an explicit $\sigma^2$-scaling upper bound on the Jensen gap of the loss.

\begin{thm}[Bounded activation functions]
\label{jenUpperBound}
Consider the same assumptions as in Proposition \ref{varbound} and also assume that each non-linear function $g_t:\mathbb{R}^{d_t}\xrightarrow[]{}\mathbb{R}^{d_t}$ is \( \beta_t \)-smooth. Then, we have
\[|\jen_Q[L(h_{W_{1:T}}(x),y)]|\leq \frac{\sigma^2}{2}\sum_{t=1}^{T}\sum_{u=1}^t\beta_td_uM_{u-1}^2\bigg(\prod_{t'=u}^{t-1}\ell_{t'}^2\|\overline{W}_{t'+1}\|_2^2\bigg)\bigg(\prod_{t'=t}^{T-1}\ell_{t'+1}\|\overline{W}_{t'+1}\|_2\bigg),\]  
where the first product is equal to $1$ if $u=t$ and the second product is equal to $1$ if $t=T$.
\end{thm}
 \noindent \textbf{Proof:} Let $g_T(z):=L(z,y)$. The dependency on $y$ does not matter for the proof. The $\beta$-smoothness of the loss function $L(z,y)$ will be denoted by $\beta_T$ and the Lipschitz constant by $\ell_T$. We will show by induction on the number of layers that
 \[|\jen_Q[g_T(h_{W_T}(x))]|\leq \sum_{t=1}^T\frac{\beta_t}{2}\var_Q[h_{W_{1:t}}(x)]\prod_{t'=t}^{T-1}\ell_{t'+1}\|\overline{W}_{t'+1}\|_2,\]
 where the product is equal to $1$ if $t=T$. For $T=1$, the result follows immediately from Lemma \ref{JensenLemmaNN}. Assume that the result is true for $T-1$ layers. From Lemma \ref{JensenLemmaNN} and the induction hypothesis, we have
\[\begin{aligned}
|\jen_Q[g_T(h_{W_{1:T}}(x))]|&\leq \ell_{T}\|\jen_Q[h_{W_{1:T}}(x)]\|+\frac{\beta_T}{2}\var_Q[h_{W_{1:T}}(x)] \\
&\leq \ell_{T}\|\overline{W}_{T}\|_2\|\jen[g_{T-1}(h_{W_{1:T-1}}(x))]\|+\frac{\beta_T}{2}\var_Q[h_{W_{1:T}}(x)]\\
&\leq  \ell_{T}\|\overline{W}_{T}\|_2\sum_{t=1}^{T-1}\frac{\beta_t}{2}\var_Q[h_{W_{1:t}}(x)]\prod_{t'=t}^{T-2}\ell_{t'+1}\|\overline{W}_{t'+1}\|_2+\frac{\beta_T}{2}\var_Q[h_{W_{1:T}}(x)]\\
&=\sum_{t=1}^{T-1}\frac{\beta_t}{2}\var_Q[h_{W_{1:t}}(x)]\prod_{t'=t}^{T-1}\ell_{t'+1}\|\overline{W}_{t'+1}\|_2+\frac{\beta_T}{2}\var_Q[h_{W_{1:T}}(x)]\\
&=\sum_{t=1}^{T}\frac{\beta_t}{2}\var_Q[h_{W_{1:t}}(x)]\prod_{t'=t}^{T-1}\ell_{t'+1}\|\overline{W}_{t'+1}\|_2.
\end{aligned}
\]
The proof of the Theorem is now completed by using the upper bound on $\var_Q[h_{W_{1:t}}(x)]$ obtained in Proposition \ref{varbound}. \Endproof

\begin{thm}[Unbounded activation functions]
\label{jenUpperBound-unbounded}
Consider the same assumptions as in Proposition~\ref{varbound-unbounded} and also assume that each non-linear function
\(g_t:\mathbb{R}^{d_t}\to\mathbb{R}^{d_t}\) is \(\beta_t\)-smooth. Then, we have
\[
\begin{aligned}
|\jen_Q[L(h_{W_{1:T}}(x),y)]|
\le\;&
\frac{\sigma^2 M_0^2}{2}
\sum_{t=1}^{T}\sum_{u=1}^t
\beta_t d_u \\
&\times
\bigg(
\prod_{t'=1}^{u-1}\ell_{t'}^2
\bigl(\sigma^2 d_{t'}+\|\overline{W}_{t'}\|_2^2\bigr)
\bigg)
\bigg(
\prod_{t'=u}^{t-1}\ell_{t'}^2\|\overline{W}_{t'+1}\|_2^2
\bigg)
\bigg(
\prod_{t'=t}^{T-1}\ell_{t'+1}\|\overline{W}_{t'+1}\|_2
\bigg),
\end{aligned}
\]
where the first product is equal to \(1\) if \(u=1\), the second product is equal to \(1\) if \(u=t\), and the third product is equal to \(1\) if \(t=T\).
\end{thm}
\noindent \textbf{Proof:} The proof is the same as that of Theorem~\ref{jenUpperBound}, but using Proposition~\ref{varbound-unbounded} instead of Proposition~\ref{varbound}. \Endproof
We now turn to generalization bounds for the Jensen gap itself, viewed as a function class indexed by $\overline w$.

\section{Generalization bound for the Jensen gap from Rademacher complexity}
\label{sec:jg-rademacher}
In the linear case, we can obtain a result from a simple contraction argument. We state a well known contraction lemma in the multi-class case next.

\begin{lem}[Contraction inequality in the multi-class case {\citep{maurer2016vector}}]
\label{lem:vector-contraction-maurer}
Let $x_1,\ldots,x_n$ be fixed and let $\mathcal F$ be a class of functions
$f:\mathcal X\to\R^C$ with components $f_k$.
For each $i\in\{1,\ldots,n\}$, let $\psi_i:\R^C\to\R$ satisfy
\[
|\psi_i(u)-\psi_i(v)| \le L_i \|u-v\|_2,
\qquad \forall u,v\in\R^C,
\]
for some constants $L_i> 0$.
Then,
\[
\E_{\varepsilon}
\Bigg[
\sup_{f\in\mathcal F}
\sum_{i=1}^n \varepsilon_i\,\psi_i(f(x_i))
\Bigg]
\;\le\;
\sqrt{2}\,
\E_{\varepsilon_{ik}}
\Bigg[
\sup_{f\in\mathcal F}
\sum_{i=1}^n\sum_{k=1}^C
\varepsilon_{ik}\,L_i\,f_k(x_i)
\Bigg],
\]
where $(\varepsilon_i)$ and $(\varepsilon_{ik})$ are independent Rademacher
variables.
\end{lem}
\noindent \textbf{Proof:} This is a slight extension to Corollary $4$ of \cite{maurer2016vector} where we allow the Lipschitz constants to depend on $i$.
Assume $L_i>0$ for all $i$. Define $h_i:\R^C\to\R$ by
$h_i(z):=\psi_i(z/L_i)$. Then, each $h_i$ is $1$-Lipschitz.
Let
\[
\mathcal G:=\{g:\mathcal X\to\R^C:\ \exists f\in\mathcal F \text{ s.t. } g(x_i)=L_i f(x_i)\ \forall i\}.
\]
For any $f\in\mathcal F$, letting $g\in\mathcal G$ satisfy $g(x_i)=L_i f(x_i)$ gives
$h_i(g(x_i))=\psi_i(f(x_i))$, hence
\[
\sup_{f\in\mathcal F}\sum_{i=1}^n \varepsilon_i\,\psi_i(f(x_i))
=
\sup_{g\in\mathcal G}\sum_{i=1}^n \varepsilon_i\,h_i(g(x_i)).
\]
Applying Maurer’s contraction inequality \citep{maurer2016vector} to the $h_i$'s yields
\[
\E_{\varepsilon}\Big[\sup_{f\in\mathcal F}\sum_{i=1}^n \varepsilon_i\,\psi_i(f(x_i))\Big]
\le
\sqrt{2}\,
\E_{\varepsilon_{ik}}\Big[\sup_{g\in\mathcal G}\sum_{i=1}^n\sum_{k=1}^C \varepsilon_{ik}\,g_k(x_i)\Big].
\]
Finally, by the definition of $\mathcal G$,
\[
\sup_{g\in\mathcal G}\sum_{i=1}^n\sum_{k=1}^C \varepsilon_{ik}\,g_k(x_i)
=
\sup_{f\in\mathcal F}\sum_{i=1}^n\sum_{k=1}^C \varepsilon_{ik}\,L_i\,f_k(x_i).
\]
\Endproof
We can now bound the Rademacher complexity of the Jensen gap class in the linear case.
\begin{thm}
\label{thm:rad-jensen-gap-linear-multiclass}
Let $S=\{(x_i,y_i)\}_{i=1}^n$ with $x_i\in\R^{d}$ and $y_i\in\{1,\cdots,C\}$.
Consider the linear multiclass score function $h_W(x)=Wx$ with $W\in\R^{C\times d}$ and
$\Theta:=\{W\in\R^{C\times d}:\ \|W\|_F\le R\}$.
Assume that for any $y$, the map $z\mapsto L(z,y)$ is differentiable and $\beta$-smooth.
For $\sigma>0$, let \( Q \) be an isotropic Gaussian distribution with mean $\overline{W}$ and covariance matrix \( \sigma^2 I_{Cd \times Cd} \). Define
\[
F_{\sigma}(\overline{W};x,y)
:=L(h_{\overline{W}}(x),y)-\E_{W\sim Q}\big[L(h_{W}(x),y)\big].
\]
and let $\mathcal{F}_{\sigma}:=\{(x,y)\mapsto F_{\sigma}(\overline{W};x,y):\, \overline{W}\in\Theta\}$.
Then, the empirical Rademacher complexity of $\mathcal{F}_{\sigma}$ satisfies
\[
\mathfrak R_S(\mathcal F_{\sigma})
\;\le\;
\frac{\sqrt{2}\beta R \sigma C}{n}\sqrt{\sum_{i=1}^n \|x_i\|^4}.
\]
In particular, if $\|x_i\|\le M$ for all $i$, then
\[
\mathfrak R_S(\mathcal F_\sigma)
\le
\frac{\sqrt{2}\beta R\sigma C\,M^2}{\sqrt{n}}.
\]
\end{thm}

\noindent\textbf{Proof:}
Let $\overline W\in\Theta$ and let $W=\overline W+U$, where
$U\sim\mathcal N(0,\sigma^2 I_{Cd\times Cd})$.
For each $i$, set $z_i=\overline W x_i\in\R^{C}$ and $U_i=Ux_i\in\R^{C}$, so that
$h_W(x_i)=z_i+U_i$.
By the first-order integral theorem,
\[
L(z_i+U_i,y_i)-L(z_i,y_i)
=
\int_0^1 \langle \nabla_z L(z_i+tU_i,y_i),\,U_i\rangle\,dt,
\]
and therefore
\[
F_\sigma(\overline W;x_i,y_i)
=
-\E_U\int_0^1 \langle \nabla_z L(z_i+tU_i,y_i),\,U_i\rangle\,dt.
\]
Using $\sup\E\le\E\sup$ and $\sup\int\le\int\sup$, we obtain
\[
\mathfrak R_S(\mathcal F_\sigma)
\le
\int_0^1
\E_{\varepsilon,U}
\Bigg[
\sup_{\overline W\in\Theta}
\frac1n\sum_{i=1}^n
\varepsilon_i
\Big(
-\langle \nabla_z L(z_i+tU_i,y_i),\,U_i\rangle
\Big)
\Bigg]dt,
\]
where $\varepsilon=(\varepsilon_1,\cdots,\varepsilon_n)\in\{\pm 1\}^n$ are Rademacher random variables.
For fixed $t$ and $U_1,\dots,U_n$, define
$\psi_i(z):=-\langle \nabla_z L(z+tU_i,y_i),U_i\rangle$.
By $\beta$-smoothness of $L(\cdot,y_i)$,
\[
|\psi_i(z)-\psi_i(z')|
\le
\beta\,\|U_i\|\,\|z-z'\|,
\]
so $\psi_i$ is $(\beta\|U_i\|)$-Lipschitz. Applying Maurer’s vector contraction inequality
(Lemma~\ref{lem:vector-contraction-maurer}) to the class
$f_{\overline W}(x)=\overline W x$ with Lipschitz constants
$L_i=\beta\|U_i\|$, we obtain
\[
\mathfrak R_S(\mathcal F_\sigma)
\le
\frac{\sqrt{2}\beta}{n}\,
\E_{\varepsilon_{ik},U}
\Bigg[
\sup_{\|\overline W\|_F\le R}
\sum_{i=1}^n\sum_{k=1}^C
\varepsilon_{ik}\,\|U_i\|\,(\overline W x_i)_k
\Bigg].
\]
The supremum over $\overline W$ can be obtained explicitly:
\[
\sup_{\|\overline W\|_F\le R}
\sum_{i=1}^n\sum_{k=1}^C
\varepsilon_{ik}\,\|U_i\|\,(\overline W x_i)_k
=
R\Big\|
\sum_{i=1}^n
\|U_i\|\,\varepsilon_i^{(C)} x_i^\top
\Big\|_F,
\]
where $\varepsilon_i^{(C)}=(\varepsilon_{i1},\ldots,\varepsilon_{iC})^\top\in\R^C$.
This is true since 
\begin{align*}
\sum_{i=1}^n\sum_{k=1}^C \varepsilon_{ik}\,\|U_i\|\,(\overline W x_i)_k
&=\sum_{i=1}^n \|U_i\|\,(\varepsilon_i^{(C)})^\top \overline W x_i \\
&=\sum_{i=1}^n\Big\langle \overline W, \|U_i\|\,\varepsilon_i^{(C)} x_i^\top\Big\rangle_F \\
&= \Big\langle \overline W,\sum_{i=1}^n \|U_i\|\,\varepsilon_i^{(C)} x_i^\top\Big\rangle_F
,
\end{align*}
where $\Big\langle \cdot,\cdot\Big\rangle_F$ denotes the Frobenius inner product.
Therefore, the empirical Rademacher complexity is upper bounded as follows:
\[
\mathfrak R_S(\mathcal F_\sigma)
\le
\frac{\sqrt{2}\beta R}{n}
\E_{\varepsilon_{ik},U}
\Big\|
\sum_{i=1}^n
\|U_i\|\,\varepsilon_i^{(C)} x_i^\top
\Big\|_F.
\]

By Jensen’s inequality and independence of Rademacher random variables,
\[
\E_{\varepsilon_{ik},U}
\Big\|
\sum_{i=1}^n
\|U_i\|\,\varepsilon_i^{(C)} x_i^\top
\Big\|_F
\le
\sqrt{
C\sum_{i=1}^n
\E_U\!\left[\|U_i\|^2\right]\|x_i\|^2
}.
\]
This is true since
\begin{align*}
\E_{\varepsilon}\Big\|
\sum_{i=1}^n \|U_i\|\,\varepsilon_i^{(C)} x_i^\top
\Big\|_F^2
&=
\sum_{k=1}^C
\E_{\varepsilon}
\sum_{i,j}
\|U_i\|\|U_j\|
\varepsilon_{ik}\varepsilon_{jk}
\langle x_i,x_j\rangle \\
&=
\sum_{k=1}^C
\sum_{i=1}^n
\|U_i\|^2 \|x_i\|^2
=
C\sum_{i=1}^n \|U_i\|^2 \|x_i\|^2 .
\end{align*}
Since $U_i=Ux_i\sim\mathcal N(0,\sigma^2\|x_i\|^2 I_C)$, we have
\[
\E_U\|U_i\|^2 = C\sigma^2\|x_i\|^2,
\]
and therefore
\[
\E_{\varepsilon_{ik},U}
\Big\|
\sum_{i=1}^n
\|U_i\|\,\varepsilon_i^{(C)} x_i^\top
\Big\|_F
\le
\sigma C\sqrt{\sum_{i=1}^n\|x_i\|^4}.
\]

Substituting into the previous bound gives
\[
\mathfrak R_S(\mathcal F_\sigma)
\le
\frac{\sqrt{2}\beta R\sigma C}{n}
\sqrt{\sum_{i=1}^n\|x_i\|^4}.
\]

If $\|x_i\|\le M$ for all $i$, then $\sum_{i=1}^n\|x_i\|^4\le nM^4$, and thus
\[
\mathfrak R_S(\mathcal F_\sigma)
\le
\frac{\sqrt{2}\beta R\sigma C\,M^2}{\sqrt{n}}.
\]
\Endproof

For bounding the Rademacher complexity of the Jensen gap class for non-linear neural networks, we will need more advanced techniques than a contraction argument. The next lemma is the key technical result. It bounds the pseudo-metric induced from the Jensen gap (defined in the lemma) with a pseudo-metric involving the properties of the loss function and measures of sensitivity to noise (first order and second order) of the neural networks in the class.

\begin{lem}[Bounding the metric induced by the Jensen gap]
\label{metric}
Let $S=\{(x_i,y_i)\}_{i=1}^n$ and let $h_w:\mathcal X\to\R^C$ be a neural network
parameterized by $w\in \Theta\subset \R^m$. Assume that \(\Theta\) is convex. Fix $\sigma>0$ and let $u\sim\mathcal N(0,I_m)$.
For each $i\in\{1,\dots,n\}$ define
\[
F_i(w)
:=
L(h_w(x_i),y_i)
-
\E_u\!\left[L(h_{w+\sigma u}(x_i),y_i)\right],
\]
and the sample-dependent pseudo-metric
\[
d_S(w_1,w_2)
:=
\Bigg(\sum_{i=1}^n (F_i(w_1)-F_i(w_2))^2\Bigg)^{1/2}.
\]

Assume that for each $y$, the map $z\mapsto L(z,y)$ is differentiable,
$\ell$-Lipschitz and $\beta$-smooth. Assume moreover that, for each $i$, the map
$w\mapsto h_w(x_i)$ is twice continuously differentiable on \(\R^m\), and that the Gaussian expectations below are finite and may be differentiated under the expectation sign.
Write
\[
J_i(w):=\nabla_w h_w(x_i)\in\R^{C\times m},
\qquad
H_{i,k}(w):=\nabla_w^2\big(h_w(x_i)\big)_k\in\R^{m\times m}.
\]
Assume that there exist sample-dependent positive semi-definite matrices $G_i\succeq 0$ such that,
for all $w\in\Theta$,
\[
J_i(w)^\top J_i(w)\preceq G_i.
\]
Define
\[
A_{i,\sigma}
:=
\sup_{w\in\Theta}
\int_0^1
\E_u\!\left[
 u^\top J_i(w+t\sigma u)^\top J_i(w+t\sigma u)u
\right]dt,
\]
and assume that $A_{i,\sigma}<\infty$ for all $i$.
For each $w\in\Theta$, define the Gaussian averaged Hessian matrix
\[
B_{i,\sigma}(w)
:=
\int_0^1
\E_u\!\left[
\sum_{k=1}^C
H_{i,k}(w+t\sigma u)^\top u u^\top H_{i,k}(w+t\sigma u)
\right]dt.
\]
Assume that there exist sample-dependent positive semi-definite matrices
$B_{i,\sigma}\succeq0$ such that, for all $w\in\Theta$,
\[
B_{i,\sigma}(w)\preceq B_{i,\sigma}.
\]
Define the sample-dependent matrix
\[
M_{S,\sigma}
:=
2\sigma^2\sum_{i=1}^n
\Big(\beta^2 A_{i,\sigma}G_i+\ell^2 B_{i,\sigma}\Big).
\]
Then, for all $w_1,w_2\in\Theta$,
\[
d_S(w_1,w_2)^2
\le
(w_1-w_2)^\top M_{S,\sigma}(w_1-w_2).
\]
\end{lem}

\noindent\textbf{Proof:}
Fix $i$ and abbreviate $z_i(w):=h_w(x_i)$ and
$g_i(w):=\nabla_z L(z_i(w),y_i)\in\R^C$.
By the chain rule,
\[
\nabla_w L(z_i(w),y_i)=J_i(w)^\top g_i(w).
\]
Differentiating under the expectation gives
\[
\nabla F_i(w)
=
J_i(w)^\top g_i(w)
-
\E_u\!\left[J_i(w+\sigma u)^\top g_i(w+\sigma u)\right].
\]
Add and subtract $\E_u[J_i(w)^\top g_i(w+\sigma u)]$ to obtain
\[
\nabla F_i(w)
=
J_i(w)^\top\Big(g_i(w)-\E_u g_i(w+\sigma u)\Big)
+
\E_u\!\left[(J_i(w)-J_i(w+\sigma u))^\top g_i(w+\sigma u)\right].
\]
Fix $v\in\R^m$ and write $\langle \nabla F_i(w),v\rangle = T_{1,i}(w,v)+T_{2,i}(w,v)$ with
\[
\begin{aligned}
T_{1,i}(w,v)
&:=
\langle g_i(w)-\E_u g_i(w+\sigma u),\; J_i(w)v\rangle,\\
T_{2,i}(w,v)
&:=
\E_u\langle g_i(w+\sigma u),\; (J_i(w)-J_i(w+\sigma u))v\rangle.
\end{aligned}
\]
Then, $(a+b)^2\le 2a^2+2b^2$ gives
\[
\langle \nabla F_i(w),v\rangle^2\le 2T_{1,i}(w,v)^2+2T_{2,i}(w,v)^2.
\]
We first bound $T_{1,i}$. By Cauchy--Schwarz and $\beta$-smoothness of $L(z,y)$,
\[
|T_{1,i}(w,v)|
\le
\|J_i(w)v\|\,\E_u\|g_i(w)-g_i(w+\sigma u)\|
\le
\beta\,\|J_i(w)v\|\,\E_u\|z_i(w)-z_i(w+\sigma u)\|.
\]
Using the fundamental theorem of calculus,
\[
z_i(w+\sigma u)-z_i(w)=\sigma\int_0^1 J_i(w+t\sigma u)\,u\,dt,
\]
so Jensen's inequality gives
\[
\E_u\|z_i(w+\sigma u)-z_i(w)\|^2
\le
\sigma^2\int_0^1
\E_u\!\left[
 u^\top J_i(w+t\sigma u)^\top J_i(w+t\sigma u)u
\right]dt
\le
\sigma^2 A_{i,\sigma}.
\]
Thus,
\[
\E_u\|z_i(w+\sigma u)-z_i(w)\|
\le
\sigma\sqrt{A_{i,\sigma}}.
\]
Since $\|J_i(w)v\|^2=v^\top J_i(w)^\top J_i(w)v\le v^\top G_i v$, we obtain
\[
2T_{1,i}(w,v)^2
\le
2\beta^2\sigma^2 A_{i,\sigma}\, v^\top G_i v.
\]
We now bound $T_{2,i}$. Since $L(z,y_i)$ is $\ell$-Lipschitz and differentiable, $\|g_i(w)\|\le \ell$,
so by Cauchy--Schwarz and Jensen,
\[
|T_{2,i}(w,v)|
\le
\E_u\big[\|g_i(w+\sigma u)\|\,\|(J_i(w+\sigma u)-J_i(w))v\|\big]
\le
\ell\,\sqrt{\E_u\|(J_i(w+\sigma u)-J_i(w))v\|^2}.
\]
Applying the fundamental theorem of calculus yields
\[
J_i(w+\sigma u)-J_i(w)=\sigma\int_0^1 \nabla_w^2 h_{w+t\sigma u}(x_i)[u]\,dt,
\]
where
\[
\nabla_w^2 h_{w'}(x_i)[u]
:=
\begin{pmatrix}
(H_{i,1}(w')u)^\top\\
\vdots\\
(H_{i,C}(w')u)^\top
\end{pmatrix}.
\]
By Jensen's inequality,
\[
\|(J_i(w+\sigma u)-J_i(w))v\|^2
\le
\sigma^2\int_0^1 \big\|\nabla_w^2 h_{w+t\sigma u}(x_i)[u]v\big\|^2\,dt.
\]
For any $w'$, the vector $\nabla_w^2 h_{w'}(x_i)[u]v\in\R^C$ has $k$-th coordinate
$u^\top H_{i,k}(w')v$. Therefore,
\[
\big\|\nabla_w^2 h_{w'}(x_i)[u]v\big\|^2
=
\sum_{k=1}^C (u^\top H_{i,k}(w')v)^2
=
 v^\top\left(\sum_{k=1}^C H_{i,k}(w')^\top u u^\top H_{i,k}(w')\right)v.
\]
Consequently,
\[
\E_u\|(J_i(w+\sigma u)-J_i(w))v\|^2
\le
\sigma^2 v^\top B_{i,\sigma}(w)v
\le
\sigma^2 v^\top B_{i,\sigma}v.
\]
Therefore,
\[
2T_{2,i}(w,v)^2
\le
2\ell^2\,\E_u\|(J_i(w+\sigma u)-J_i(w))v\|^2
\le
2\ell^2\sigma^2\, v^\top B_{i,\sigma} v.
\]
Combining the bounds gives, for all $w\in\Theta$ and all $v\in\R^m$,
\[
\langle \nabla F_i(w),v\rangle^2
\le
v^\top\Big(2\beta^2\sigma^2 A_{i,\sigma}G_i+2\ell^2\sigma^2 B_{i,\sigma}\Big)v.
\]
Now fix $w_1,w_2\in\Theta$, set $v=w_1-w_2$, and define $w_t=w_2+t v$.
By the fundamental theorem of calculus,
\[
F_i(w_1)-F_i(w_2)=\int_0^1 \langle \nabla F_i(w_t),v\rangle\,dt,
\]
and Jensen's inequality implies
\[
(F_i(w_1)-F_i(w_2))^2\le \int_0^1 \langle \nabla F_i(w_t),v\rangle^2\,dt.
\]
Using the bound above yields
\[
(F_i(w_1)-F_i(w_2))^2
\le
v^\top\Big(2\beta^2\sigma^2 A_{i,\sigma}G_i+2\ell^2\sigma^2 B_{i,\sigma}\Big)v.
\]
Summing over $i$ gives
\[
d_S(w_1,w_2)^2
\le
v^\top\Bigg(
2\sigma^2\sum_{i=1}^n
\Big(\beta^2 A_{i,\sigma}G_i+\ell^2 B_{i,\sigma}\Big)
\Bigg)v,
\]
which is the claimed bound.
\Endproof

It follows from Talagrand-type comparison inequalities \citep{talagrand2005generic} that we can bound the Rademacher complexity of the Jensen gap class from the trace of the matrix $M_S$. The next two results are what we need to prove this fact.

\begin{prop}[Talagrand comparison inequality {\citep[Cor.~8.5.6]{Vershynin2018HDP}}]
\label{cor:talagrand-comparison}
Let $(X_t)_{t\in T}$ be a mean-zero random process and $(Y_t)_{t\in T}$ a mean-zero Gaussian
process. Assume that for all $t,s\in T$,
\[
\|X_t-X_s\|_{\psi_2}\le K\|Y_t-Y_s\|_{L_2}.
\]
Then,
\[
\mathbb{E}\sup_{t\in T} X_t \le c K\,\mathbb{E}\sup_{t\in T} Y_t,
\]
where $c>0$ is an absolute constant. For a real-valued random variable $X$, the $L_2$ norm is defined as
\[
\|X\|_{L_2}
:=
\big(\E[X^2]\big)^{1/2}.
\]
The $\psi_2$ norm of $X$ is defined by
\[
\|X\|_{\psi_2}
:=
\inf\Big\{
t>0 \;:\;
\E\!\left[\exp\!\left(\frac{X^2}{t^2}\right)\right]\le 2
\Big\}.
\]
\end{prop}

\begin{prop}
\label{prop:rad-talagrand-MS}
Let $S=\{(x_i,y_i)\}_{i=1}^n$ be a sample and let
\[
\Theta := \{\, w\in\mathbb R^m : \|w\|_2 \le R \,\}.
\]
Let $f:\Theta\times(\mathcal X\times\mathcal Y)\to\mathbb R$ and define
\[
\mathcal F
:=
\{\, (x,y)\mapsto f(w,(x,y)) : w\in\Theta \,\}.
\]
Define the sample-dependent pseudo-metric
\[
d_S(w_1,w_2)
:=
\Bigg(
\sum_{i=1}^n
\big(
f(w_1,(x_i,y_i)) - f(w_2,(x_i,y_i))
\big)^2
\Bigg)^{1/2}.
\]
Assume there exists a sample-dependent positive semi-definite matrix
$M_S\in\mathbb R^{m\times m}$ such that
\[
d_S(w_1,w_2)
\;\le\;
\sqrt{(w_1-w_2)^\top M_S (w_1-w_2)}
\qquad\forall w_1,w_2\in\Theta.
\]
Then, the empirical Rademacher complexity of $\mathcal F$ satisfies
\[
\mathfrak R_S(\mathcal F)
\;\le\;
c\,\frac{R}{n}\,\sqrt{\tr(M_S)},
\]
where $c>0$ is a constant.
\end{prop}
\noindent{\textbf{Proof:}}
The proof is a direct application of Proposition \ref{cor:talagrand-comparison}
together with a standard computation of the Gaussian width of an ellipsoid
(see e.g.\ \citep[Chapter~7]{Vershynin2018HDP}).
Let $\varepsilon_1,\ldots,\varepsilon_n$ be independent Rademacher random variables and define
\[
X_w:=\sum_{i=1}^n \varepsilon_i\, f\big(w,(x_i,y_i)\big).
\]
Then, $\E_\varepsilon[X_w]=0$ for every $w$ and
\[
\mathfrak R_S(\mathcal F)=\frac{1}{n}\E_\varepsilon\Big[\sup_{w\in\Theta} X_w\Big].
\]
For $w_1,w_2\in\Theta$, set $\Delta_i:=f(w_1,(x_i,y_i))-f(w_2,(x_i,y_i))$ and \(
Z:=\sum_{i=1}^n \varepsilon_i \Delta_i.
\) The random variable $Z$ is subgaussian with variance proxy $\sum_{i=1}^n \Delta_i^2=d_S(w_1,w_2)^2$. This is true since
the moment generating function of $Z$ satisfies 
\[
\E_\varepsilon e^{\lambda Z}
=
\prod_{i=1}^n \E_\varepsilon e^{\lambda \varepsilon_i \Delta_i}
=
\prod_{i=1}^n \cosh(\lambda \Delta_i)
\le
\prod_{i=1}^n \exp\!\Big(\frac{\lambda^2\Delta_i^2}{2}\Big)
=
\exp\!\Big(\frac{\lambda^2}{2}\sum_{i=1}^n \Delta_i^2\Big),
\]
where we used $\cosh(u)\le e^{u^2/2}$ for all $u\in\R$. It follows that there exists a constant $c_0>0$ such that $\|Z\|_{\psi_2}\leq c_0 \sqrt{\sum_{i=1}^n \Delta_i^2}=c_0d_S(w_1,w_2)$ (see for example exercice $2.40$ from \cite{Vershynin2018HDP}). From our assumption on $d_S(w_1,w_2)$, we therefore have \[\|X_{w_1}-X_{w_2}\|_{\psi_2}\leq c_0\sqrt{(w_1-w_2)^\top M_S (w_1-w_2)}.\]
In order to apply the Talagrand comparison inequality, we need to specify a mean-zero Gaussian process $(Y_w)_{w\in\Theta}$. Let $g\sim\mathcal N(0,I_m)$ and define the Gaussian process
\[
Y_w:=\langle g, M_S^{1/2}w\rangle,
\qquad w\in\Theta.
\]
For $\Delta:=w_1-w_2$ and using $\E[gg^\top]=I_m$, we get
\[
\|Y_{w_1}-Y_{w_2}\|_{L_2}^2
=
\E\big(g^\top M_S^{1/2}\Delta\big)^2
=
\E\big[\Delta^\top M_S^{1/2} g g^\top M_S^{1/2}\Delta\big]
=
\Delta^\top M_S \Delta.
\]
Hence, $\|Y_{w_1}-Y_{w_2}\|_{L_2}=\sqrt{(w_1-w_2)^\top M_S (w_1-w_2)}$. We can therefore apply the Talagrand comparison inequality with $K=c_0$. Thus, there exists a constant $c_1>0$ such that
\[
\E_\varepsilon\sup_{w\in\Theta} X_w
\le
c_1 c_0\, \E_g\sup_{w\in\Theta} Y_w.
\]
Furthermore, \[
\E_g\sup_{w\in\Theta} Y_w
=
R\,\E\|M_S^{1/2}g\|_2
\le
R\,\sqrt{\E\|M_S^{1/2}g\|_2^2}
=
R\,\sqrt{\tr(M_S)},
\] 
Combining the above inequalities and dividing by $n$ yields
\[
\mathfrak R_S(\mathcal F)
=
\frac{1}{n}\E_\varepsilon\sup_{w\in\Theta} X_w
\le
\frac{c_1 c_0}{n}\,R\,\sqrt{\tr(M_S)}
=
c\,\frac{R}{n}\,\sqrt{\tr(M_S)},
\]
with $c:=c_1c_0>0$ a constant.
\Endproof
We can now provide an upper bound on the Rademacher complexity of the Jensen gap class for neural networks.

\begin{coro}[Rademacher complexity of the Jensen-gap class]
\label{cor:rad-jensen-gap}
Let $S=\{(x_i,y_i)\}_{i=1}^n$ and let
\[
\Theta := \{\, w\in\mathbb R^m : \|w\|_2 \le R \,\}.
\]
Fix $\sigma>0$ and let $\mathcal F_\sigma$ be the Jensen gap class for the loss function
$L(h_w(x),y)$. Assume that for each $y$, the map $z\mapsto L(z,y)$ is differentiable,
$\ell$-Lipschitz and $\beta$-smooth. Assume moreover that the hypotheses of
Lemma~\ref{metric} hold. Define
\[
J_{S,\sigma}
:=
\frac{1}{n}\sum_{i=1}^n A_{i,\sigma}\tr(G_i),
\qquad
H_{S,\sigma}
:=
\frac{1}{n}\sum_{i=1}^n \tr(B_{i,\sigma}).
\]
Then, the empirical Rademacher complexity of $\mathcal F_\sigma$ satisfies
\[
\mathfrak R_S(\mathcal F_\sigma)
\le
c\,\frac{R\sigma}{\sqrt{n}}\,
\sqrt{
\beta^2 J_{S,\sigma}
+
\ell^2H_{S,\sigma}
},
\]
where $c>0$ is a universal constant.
\end{coro}
\noindent \textbf{Proof:}
By Lemma~\ref{metric}, the pseudo-metric induced by the Jensen gap class is dominated by the
quadratic form associated with
\[
M_{S,\sigma}
=
2\sigma^2\sum_{i=1}^n
\Big(\beta^2 A_{i,\sigma}G_i+\ell^2 B_{i,\sigma}\Big).
\]
Applying Proposition~\ref{prop:rad-talagrand-MS} gives
\[
\mathfrak R_S(\mathcal F_\sigma)
\le
c\frac{R}{n}\sqrt{\tr(M_{S,\sigma})}.
\]
Since
\[
\tr(M_{S,\sigma})
=
2\sigma^2 n
\left(
\beta^2 J_{S,\sigma}+\ell^2 H_{S,\sigma}
\right),
\]
the result follows after absorbing the factor $\sqrt2$ into the universal constant $c$.
\Endproof

\begin{rem}
\label{remlin}
    As a sanity check of the result, consider the case of linear hypothesis classes ($h_W(x)=Wx$). Then, the Hessian term $H_{S,\sigma}$ is zero and the Jacobian term $J_{S,\sigma}$ is easily seen to be bounded by $C^2M^4$ if $||x||\leq M$. Therefore, Corollary \ref{cor:rad-jensen-gap} implies
    \[\mathfrak R_S(\mathcal F_\sigma)
\;\le\;
c\,\frac{R\sigma\beta C M^2}{\sqrt{n}}.
\] This is exactly the same result as the one obtained in Theorem \ref{thm:rad-jensen-gap-linear-multiclass} up to a constant.  
\end{rem}

\begin{rem}[Explicit constant]
\label{rem:explicit-constant}
Corollary \ref{cor:rad-jensen-gap} leaves us with an unspecified universal constant $c$. In the present setting, it is possible to alternatively obtain the explicit constant
\(c=\sqrt{\pi}\), which is useful for numerical evaluation.
First, we have the following comparison between the Rademacher and Gaussian complexity:
\[
\E_\varepsilon \sup_{w\in\Theta} X_w
\le
\sqrt{\frac{\pi}{2}}\,
\E_g \sup_{w\in\Theta} \sum_{i=1}^n g_i f(w,(x_i,y_i)).
\]
This follows by writing \(g_i=\varepsilon_i|g_i|\) and applying Jensen's inequality together with \(\E|g_i|=\sqrt{2/\pi}\).
Second, for every \(w_1,w_2\in\Theta\), the \(L_2\) norm satisfies
\[
\Big\|\sum_{i=1}^n g_i\big(f(w_1,(x_i,y_i))-f(w_2,(x_i,y_i))\big)\Big\|_{L_2}
=
d_S(w_1,w_2).
\]
Third, by assumption, this quantity is upper bounded by the \(L_2\) norm of
$
Y_{w_1}-Y_{w_2}
$.
Finally, by the Gaussian comparison inequality (Sudakov--Fernique; see \cite[Theorem~7.2.8]{Vershynin2018HDP}),
\[
\E_g \sup_{w\in\Theta} \sum_{i=1}^n g_i f(w,(x_i,y_i))
\le
\E_g \sup_{w\in\Theta} Y_w,
\]
which yields the result after considering the factor $\sqrt2$ that was absorbed into the universal constant $c$.
\end{rem}


We are now ready to state and prove the main result of this section.
\begin{prop}[Uniform generalization bound for the Jensen-gap class]
\label{cor:jensen-gap-genbound}
Let $S=\{(x_i,y_i)\}_{i=1}^n\sim D^n$ and let
\[
\Theta := \{\, \overline w\in\mathbb R^m : \|\overline w\|_2 \le R \,\}.
\]
Fix $\sigma>0$ and let $Q$ denote an isotropic Gaussian distribution
\[
Q := \mathcal N(\overline w,\sigma^2 I_m),
\qquad \overline w\in\Theta .
\]
Consider the same assumptions as in Corollary \ref{cor:rad-jensen-gap}. Assume moreover that there exists $B>0$ such that for all $(x,y)$ and all $\overline w\in\Theta$,
\[
 |\jen_{Q}\!\big[L(h_w(x),y)\big]| \le \sigma^2 B.
\]
Then, for any $\delta\in(0,1)$, with probability at least $1-\delta$ over $S\sim D^n$, we have
uniformly for all $\overline w\in\Theta$,
\[
\begin{aligned}
\mathbb E_{(x,y)\sim D}\!\Big[\jen_{Q}\!(L(h_w(x),y))\Big]
&\le
\frac1n\sum_{i=1}^n \jen_{Q}\!(L(h_w(x_i),y_i)) \\
&\quad
+\;
2c\,\frac{R\sigma}{\sqrt{n}}\,
\sqrt{\beta^2 J_{S,\sigma}+\ell^2 H_{S,\sigma}}
\;+\;
8\sigma^2 B\,\sqrt{\frac{2\ln(4/\delta)}{n}},
\end{aligned}
\]
where $c>0$ is an absolute constant. 
\end{prop}
\noindent \textbf{Proof:} Using Theorem $26.5$ from \cite{ShalevShwartz2014} and Corollary \ref{cor:rad-jensen-gap}, for any $\delta\in(0,1)$, with probability at least $1-\delta$ over $S\sim D^n$, we have uniformly for all $\overline w\in\Theta$,
\[
\begin{aligned}
\mathbb E_{(x,y)\sim D}\!\Big[\jen_{Q}\!(L(h_w(x),y))\Big]
&\le
\frac1n\sum_{i=1}^n \jen_{Q}\!(L(h_w(x_i),y_i)) \\
&\quad
+\;
2\,\mathfrak R_S(\mathcal F_\sigma)
\;+\;
8\sigma^2 B\,\sqrt{\frac{2\ln(4/\delta)}{n}} \\
&\le
\frac1n\sum_{i=1}^n \jen_{Q}\!\big[L(h_w(x_i),y_i)\big] \\
&\quad
+\;
2c\,\frac{R\sigma}{\sqrt{n}}\,
\sqrt{\beta^2 J_{S,\sigma}+\ell^2 H_{S,\sigma}}
\;+\;
8\sigma^2 B\,\sqrt{\frac{2\ln(4/\delta)}{n}},
\end{aligned}
\]
concluding the proof. \Endproof
If needed, a standard argument can be used to obtain a result holding uniformly over a set of standard deviations $\{\sigma_i\}$.
\begin{prop}[Uniform generalization bound on $\{\sigma_i\}_{i\ge 1}$]
\label{prop:jensen-gap-genbound-dyadic}
Let $S=\{(x_j,y_j)\}_{j=1}^n\sim D^n$ and let
\[
\Theta := \{\, \overline w\in\mathbb R^m : \|\overline w\|_2 \le R \,\}.
\]
Fix $\sigma_{\max}>0$ and for $i\ge 1$ define
\[
\sigma_i := 2^{\,1-i}\sigma_{\max}.
\]
For $\overline w\in\Theta$ and $i\ge 1$, let
\[
Q_{\overline w,\sigma_i} := \mathcal N(\overline w,\sigma_i^2 I_m).
\]
Assume the same conditions as in Corollary~\ref{cor:rad-jensen-gap}. Assume moreover that there exists $B>0$ such that
for all $(x,y)$, all $\overline w\in\Theta$ and all $i\ge 1$,
\[
 |\jen_{Q_{\overline w,\sigma_i}}\!\big[L(h_w(x),y)\big]| \le \sigma_i^2 B.
\]
Then, for any $\delta\in(0,1)$, with probability at least $1-\delta$ over $S\sim D^n$, we have uniformly for all
$i\ge 1$ and all $\overline w\in\Theta$,
\[
\begin{aligned}
\mathbb E_{(x,y)\sim D}\!\Big[\jen_{Q_{\overline w,\sigma_i}}\!\big(L(h_w(x),y)\big)\Big]
&\le
\frac1n\sum_{j=1}^n \jen_{Q_{\overline w,\sigma_i}}\!\big(L(h_w(x_j),y_j)\big)
+
2c\,\frac{R\sigma_i}{\sqrt{n}}\,
\sqrt{\beta^2 J_{S,\sigma_i}+\ell^2 H_{S,\sigma_i}} \\
&\quad
+
8\sigma_i^2 B\,
\sqrt{
\frac{
4\ln(i)
+2\ln(8/\delta)
}{n}
}.
\end{aligned}
\]
\end{prop}

\noindent\textbf{Proof:} This is a standard union bound argument.
For fixed $\sigma>0$ and $\eta\in(0,1)$, define
\[
\mathcal B(\sigma,\eta;\overline w)
:=
\frac1n\sum_{j=1}^n \jen_{Q_{\overline w,\sigma}}\!\big(L(h_w(x_j),y_j)\big)
+
2c\,\frac{R\sigma}{\sqrt{n}}\,
\sqrt{\beta^2 J_{S,\sigma}+\ell^2 H_{S,\sigma}}
+
8\sigma^2 B\sqrt{\frac{2\ln(4/\eta)}{n}}.
\]
By Proposition~\ref{cor:jensen-gap-genbound}, for every fixed $\sigma$ and $\eta$,
\[
\mathbb P_S\!\left(
\forall \overline w\in\Theta:\;
\mathbb E\!\Big[\jen_{Q_{\overline w,\sigma}}\!\big(L(h_w(x),y)\big)\Big]
\le
\mathcal B(\sigma,\eta;\overline w)
\right)
\ge 1-\eta .
\]
For $i\ge 1$, choosing $\eta=\delta_i:=\delta/(2i^2)$ and $\sigma=\sigma_i$, we then have
\[
\mathbb P_S\!\left(
\forall \overline w\in\Theta:\;
\mathbb E\!\Big[\jen_{Q_{\overline w,\sigma_i}}\!\big(L(h_w(x),y)\big)\Big]
\le
\mathcal B(\sigma_i,\delta_i;\overline w)
\right)
\ge 1-\delta_i .
\]
Taking a union bound over $i\ge 1$ and using $\sum_{i\ge 1}\delta_i \le \delta$ gives that with probability at least
$1-\delta$, the inequality holds simultaneously for all $i\ge 1$ and all $\overline w\in\Theta$.
It remains to simplify the confidence term:
\[
\ln\!\Big(\frac{4}{\delta_i}\Big)
=
\ln\!\Big(\frac{8i^2}{\delta}\Big)
=
\ln\!\Big(\frac{8}{\delta}\Big) + 2\ln (i).
\]
\Endproof

\begin{rem}[From Gaussian averaged controls to simpler global controls]
\label{rem:global-controls-from-gaussian-averages}
The quantities \(A_{i,\sigma}\) and \(B_{i,\sigma}\) appearing above are
averaged along the perturbation path \(w+t\sigma u\). This is the appropriate
form for ordinary Gaussian posteriors, since the perturbation is not confined
to the original parameter set \(\Theta\). However, in settings where the posterior
perturbation has bounded support, we can obtain a simpler expression. Suppose that the perturbed parameter is supported in a larger parameter
set \(\widetilde\Theta\), in the sense that
\[
w+t\sigma u\in\widetilde\Theta
\qquad
\text{for all }w\in\Theta,\ t\in[0,1],
\]
almost surely. Assume also that \(\mathbb E[uu^\top]\preceq I_m\),
and that, for each \(i\), there exists a positive semidefinite matrix
\(G_i\succeq0\) such that
\[
J_i(v)^\top J_i(v)\preceq G_i
\qquad
\text{for all }v\in\widetilde\Theta .
\]
Then, the Gaussian averaged Jacobian quantity satisfies
\[
\begin{aligned}
A_{i,\sigma}
&=
\sup_{w\in\Theta}
\int_0^1
\mathbb E_u\!\left[
u^\top J_i(w+t\sigma u)^\top J_i(w+t\sigma u)u
\right]dt  \\
&\le
\mathbb E_u[u^\top G_i u]
=
\operatorname{Tr}\!\big(G_i\,\mathbb E[uu^\top]\big)
\le
\operatorname{Tr}(G_i).
\end{aligned}
\]
Consequently,
\[
J_{S,\sigma}
=
\frac1n\sum_{i=1}^n A_{i,\sigma}\operatorname{Tr}(G_i)
\le
\frac1n\sum_{i=1}^n \operatorname{Tr}(G_i)^2 .
\]
Similarly, suppose that there exist positive semidefinite matrices
\(M_i\succeq0\) such that
\[
\sum_{k=1}^C H_{i,k}(v)^\top H_{i,k}(v)\preceq M_i
\qquad
\text{for all }v\in\widetilde\Theta.
\]
Then, the proof of Lemma~\ref{metric} can be modified by using this global
Hessian control directly in the estimate of the second order term. In that setting, the Hessian contribution in the final Rademacher bound becomes
\[
H_S
:=
\frac1n\sum_{i=1}^n\operatorname{Tr}(M_i).
\]
Therefore, under bounded support perturbations, Corollary~\ref{cor:rad-jensen-gap}
implies the simpler bound
\[
\mathfrak R_S(\mathcal F_\sigma)
\le
c\,\frac{R\sigma}{\sqrt n}
\sqrt{
\beta^2 J_S+\ell^2 H_S
},
\]
where
\[
J_S
:=
\frac1n\sum_{i=1}^n \operatorname{Tr}(G_i)^2,
\qquad
H_S
:=
\frac1n\sum_{i=1}^n \operatorname{Tr}(M_i).
\]

For ordinary Gaussian posteriors, the perturbation has unbounded support, so
the bounded support simplification does not apply directly. A possible
extension would be to use a localization argument: control the Jacobian and
Hessian on a larger set \(\widetilde\Theta\) that contains the perturbation
path with high probability and then bound the contribution of the
complementary event separately. This would lead to a global 
bound as above together with a residual term.

The practical regularizer introduced in Section~\ref{sec:jh-bn-regularizer}
takes a much simpler route. Instead of keeping the supremum over parameters and the Gaussian average along the perturbation path, it evaluates empirical
Jacobian and Hessian quantities at the current effective parameters. Thus, it
should be viewed as a convenient proxy inspired by the structure of the
bound, rather than as a direct evaluation of \(J_{S,\sigma}\) and
\(H_{S,\sigma}\). In the SHEL analysis of Section~\ref{sec:shel}, we instead retain the
quantities averaged over Gaussian perturbations, while replacing each
supremum by the largest realization observed during training.
\end{rem}

\section{Bound comparisons for linear classifiers}
\label{sec:linear-bounded-experiments}
When the loss function is bounded, both the bounded and unbounded variants of the derandomized PAC-Bayes bounds can be applied. We first introduce a convenient bounded smooth multi-class loss obtained by extending the sigmoid loss.

\begin{lem}[Surrogate multiclass sigmoid loss function]
\label{lem:surrogate-multiclass-sigmoid-loss}
Fix $R>0$ and $C>2$, and let $\sigma(z):=(1+e^{-z})^{-1}$.
For $z\in\R^C$ and $y\in\{1,\dots,C\}$, define
\[
L(z,y):=\sum_{k\neq y}\sigma(z_k-z_y).
\]
Let
\[
m_R:=(C-1)\sigma\!\left(-R\sqrt{\frac{C}{C-1}}\right),
\qquad
q_R:=\frac12+(C-2)\sigma\!\left(-R\sqrt{\frac{C}{2(C-2)}}\right),
\]
and
\[
a(R):=\frac{1}{q_R-m_R},
\qquad
b(R):=-a(R)\,m_R.
\]
Define the surrogate loss
\[
\Phi_R(z,y):=a(R)L(z,y)+b(R).
\]
Then the following hold: 
\begin{enumerate}
\item[(i)] On the domain $\|z\|\le R$:
\[
L^{0\text{-}1}(z,y)\le \Phi_R(z,y).
\]
\item[(ii)] The infimum of $\Phi_R$ on the domain $\|z\|\le R$ equals $0$, whereas on the domain $\mathbb{R}^C$ it equals $b(R)$.
\item[(iii)] The supremum of $\Phi_R$ on the domain $\|z\|\le R$ equals \[a(R)(C-1)\bigl(2\sigma\!\left(R\sqrt{\frac{C}{C-1}}\right)-1\bigr),\]
whereas on the domain $\mathbb{R}^C$ it equals \[a(R)(C-1)\sigma\!\left(R\sqrt{\frac{C}{C-1}}\right).\]
\item[(iv)] $\Phi_R$ is $\ell_R$-Lipschitz on the domain $\mathbb{R}^C$ with
\[
\ell_R=\frac{a(R)}{4}\sqrt{C(C-1)}.
\]
\item[(v)] 
$\Phi_R$ is $\beta_R$-smooth on the domain $\mathbb{R}^C$ with
\[
\beta_R\leq\frac{a(R)\,C}{6\sqrt3}.
\]
\end{enumerate}
\end{lem}
\noindent \textbf{Proof:} Fix $y$. We start by proving that the infimum of $\Phi_R$ on the domain $\|z\|\le R$ equals $0$. At a minimizer of $L(z,y)$, we must have
$
z_y \ge z_k$ for $k\neq y$
(this can be seen by a simple perturbation argument). We may therefore restrict
the optimization to the set
\[
B(0,R)\cap \{z:\, z_y \ge z_k \ \text{for all } k\neq y\}.
\]
This set is convex. Moreover, on this domain the arguments of $\sigma$ satisfy
$z_k-z_y\le 0$, and since $\sigma$ is convex on $(-\infty,0]$, the function
\[
L(z,y)=\sum_{k\neq y}\sigma(z_k-z_y)
\]
is convex on this domain. Both the objective and the constraint are invariant under permutations of the coordinates $\{z_k\}_{k\neq y}$. Therefore, if $z$ is a
minimizer, any permutation of the incorrect-class coordinates is also a
minimizer. Averaging over all such permutations gives another feasible point satisfying
\[
z_k=t \qquad (\forall k\neq y)
\]
for some $t\in\mathbb{R}$. By convexity of $L(\cdot,y)$ on the feasible set,
this averaged point achieves a value of the objective no larger than the
original one. Therefore, there exists a minimizer with all incorrect-class
coordinates equal. The optimization problem thus reduces to \[ \min_{s,t\in\R}\ (C-1)\sigma(t-s) \qquad \text{subject to} \qquad s^2+(C-1)t^2\le R^2, \quad s\ge t. \] Since $\sigma$ is increasing, this is equivalent to \[ \max_{s,t\in\R}\ (s-t) \qquad \text{subject to} \qquad s^2+(C-1)t^2\le R^2, \quad s\ge t. \]
The solution of this problem can be obtained by using Lagrange multipliers. The infimum of $L(z,y)$ is then found to be given by 
\[(C-1)\sigma\bigg(-R\sqrt{\frac{C}{C-1}}\bigg)=m_R\]
and $(ii)$ for the bounded domain follows. For the unbounded domain, the result is obtained by observing that the infimum of $L(z,y)$ is $0$ on $\mathbb{R}^C$. Applying the same argument to $-L(z,y)$ allows to find the supremum of the function $L(z,y)$ on the domain $\|z\|\leq R$. It is given by
\[(C-1)\sigma\bigg(R\sqrt{\frac{C}{C-1}}\bigg).\]
Substituting into the expression for $\Phi_R(z,y)$ and simplifying proves $(iii)$ for the bounded domain. For the unbounded domain $\mathbb{R}^C$, the result follows from observing that the supremum of $L(z,y)$ is then $C-1$.
We now prove $(i)$.
If $L^{0\text{-}1}(z,y)=1$, then $z_k\ge z_y$ for some $k\neq y$, so
\[
\sigma(z_k-z_y)\ge \tfrac12 .
\]
Minimizing the remaining terms under $\|z\|\le R$ gives
\[
L(z,y)\ge \frac12+(C-2)\sigma\!\left(-R\sqrt{\frac{C}{2(C-2)}}\right)=q_R 
\]
(this can be obtained by the same argument as in the proof of $(ii)$). 
Therefore,
\[
\Phi_R(z,y)=a(R)(L(z,y)-m_R)\ge 1, 
\]
concluding the proof of $(i)$.
In order to prove $(iv)$, we compute the gradient of $L(z,y)$:
\[
\nabla_z L(z,y)=\sum_{k\neq y}\sigma'(z_k-z_y)(e_k-e_y).
\]
On one hand, the norm of the gradient evaluated at $0$ is equal to $\frac14\sqrt{C(C-1)}$. On the other hand, 
\[
\|\nabla_z L(z,y)\|^2
=
\sum_{k\neq y}\sigma'(z_k-z_y)^2
+
\left(\sum_{k\neq y}\sigma'(z_k-z_y)\right)^2
\le
(C-1)\frac{1}{16}
+
\left((C-1)\frac{1}{4}\right)^2
=
\frac{C(C-1)}{16},
\]
where we used that $0\leq\sigma'(x)\leq 1/4.$
We conclude that the Lipschitz constant of $L(z,y)$ is equal to $\frac14\sqrt{C(C-1)}$. Multiplying by $a(R)$ then gives the Lipschitz constant of $\Phi_R$. 

Finally, in order to prove $(v)$, we compute the Hessian of $L(z,y)$:
\[
\nabla_z^2 L(z,y)
=
\sum_{k\neq y}\sigma''(z_k-z_y)(e_k-e_y)(e_k-e_y)^\top .
\]
This Hessian has a convenient simple form. We get the following upper bound on its spectral norm:
\[
\|\nabla_z^2 L(z,y)\|_2
\leq 
\max_{k\neq y}|\sigma''(z_k-z_y)|\cdot\| \sum_{k\neq y}(e_k-e_y)(e_k-e_y)^\top\|_2 .
\]
 The eigenvalues and eigenvectors of $\sum_{k\neq y}(e_k-e_y)(e_k-e_y)^\top$ are straightforward to obtain. We have the eigenvalue $0$ with eigenvector $\textbf{1}$, an eigenvalue $C$ with eigenvector $v = (1,\ldots,1,\underbrace{-(C-1)}_{y\text{-th coordinate}},1,\ldots,1)$. Also,
 the subspace defined by
\[
v_y = 0,
\qquad
\sum_{k\neq y} v_k = 0
\]
is an eigenspace of dimension \(C-2\) associated with an eigenvalue of \(1\). We conclude that the spectral norm of $\sum_{k\neq y}(e_k-e_y)(e_k-e_y)^\top$ is equal to $C$. Multiplying by $a(R)$ and using $\sup_{x\in \mathbb{R}}|\sigma''(x)|=\frac{1}{6\sqrt{3}}$ concludes the proof. \hfill$\blacksquare$

In Figure~\ref{fig:bounded-vs-unbounded-sigmoid}, we compare the
generalization bound terms obtained from the unbounded derandomized
PAC-Bayes bound of Theorem~\ref{combinedThmLin} with those obtained from
the bounded derandomized PAC-Bayes bound of Theorem~\ref{boundedLin}, both
applied to the multi-class sigmoid surrogate of
Lemma~\ref{lem:surrogate-multiclass-sigmoid-loss}. We set
\(n=100000\), \(d=50\) and \(\delta=0.05\), and vary \(R\) over the
interval \([0.1,1]\). The bounded version gives a looser bound for very
small values of \(R\), but becomes substantially tighter than the unbounded
version as \(R\) increases. This behavior reflects, in part, the less
favorable dimensional dependence of the unbounded bound. The unbounded bound is tighter only in a restricted regime of
very small parameter norms and low dimension. Outside this regime, when
both results are applicable, the bounded version should be preferred.

We next compare our bounded derandomized PAC-Bayes bounds with a
Rademacher complexity bound for the same class of linear predictors. The
Rademacher complexity bound is given by the following proposition.
\begin{prop}
\label{lem:mc-linear-bounded}
Let $S=\{(x_i,y_i)\}_{i=1}^n$. Assume that $\|x\|\le M$ and consider the hypothesis class
\[
\mathcal F
:=
\big\{\,x\mapsto Wx \in \R^C:\ \|W\|_F\le R\,\big\}.
\]
Let $L(z,y)$ take values in $[0,A]$ and assume that, for every $y$, the map
$z\mapsto L(z,y)$ is $\ell$-Lipschitz.
Then, with probability at least $1-\delta$ over $S\sim D^n$, we have uniformly for all
$W$ with $\|W\|_F\le R$,
\[
L_D(h_W)
\;\le\;L_S(h_W)+
2\sqrt{2}\,\ell\,\frac{RM\sqrt{C}}{\sqrt{n}}
\;+\;
4A\,\sqrt{\frac{2\ln(4/\delta)}{n}}.
\]
\end{prop}
\noindent\textbf{Proof:}
Immediate from Maurer’s vector-contraction inequality applied to $L(\cdot,y)$,
Theorem~26.5 (item~2) of \citet{ShalevShwartz2014}, and the bound
$\mathfrak R_S(\mathcal F)\le RM\sqrt{C/n}$ for the linear class with $\|W\|_F\le R$.
\hfill$\blacksquare$

The relative improvement of our bound in Theorem~\ref{boundedLin} over the
Rademacher complexity bound is shown in Figure~\ref{fig:relative-improvement}.
Our bound is tighter for moderate values of \(R\) and becomes comparable to
the Rademacher bound for larger values of \(R\). Figure~\ref{fig:linear-ce-bounds}
reports a complementary experiment on MNIST. In this case, the
\(\klb^{-1}\) version of Theorem~\ref{thm:bounded-kl-regrouped-rademacher-sigmai} remains tighter over the displayed range.

\begin{figure}[t]
\centering

\begin{subfigure}{0.48\linewidth}
\centering
\includegraphics[width=\linewidth]{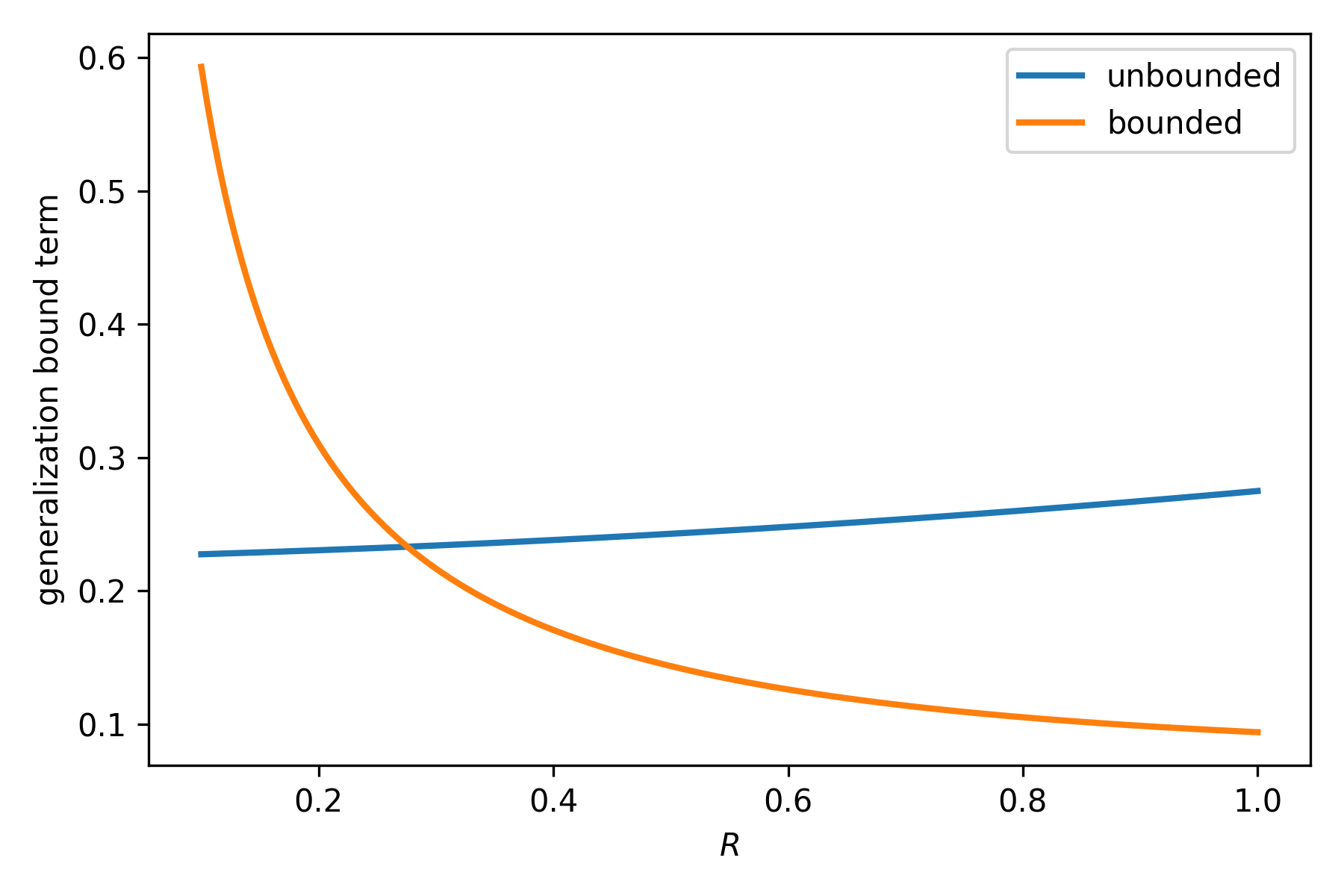}
\caption{$C=3$}
\end{subfigure}
\hspace{0.02\linewidth}
\begin{subfigure}{0.48\linewidth}
\centering
\includegraphics[width=\linewidth]{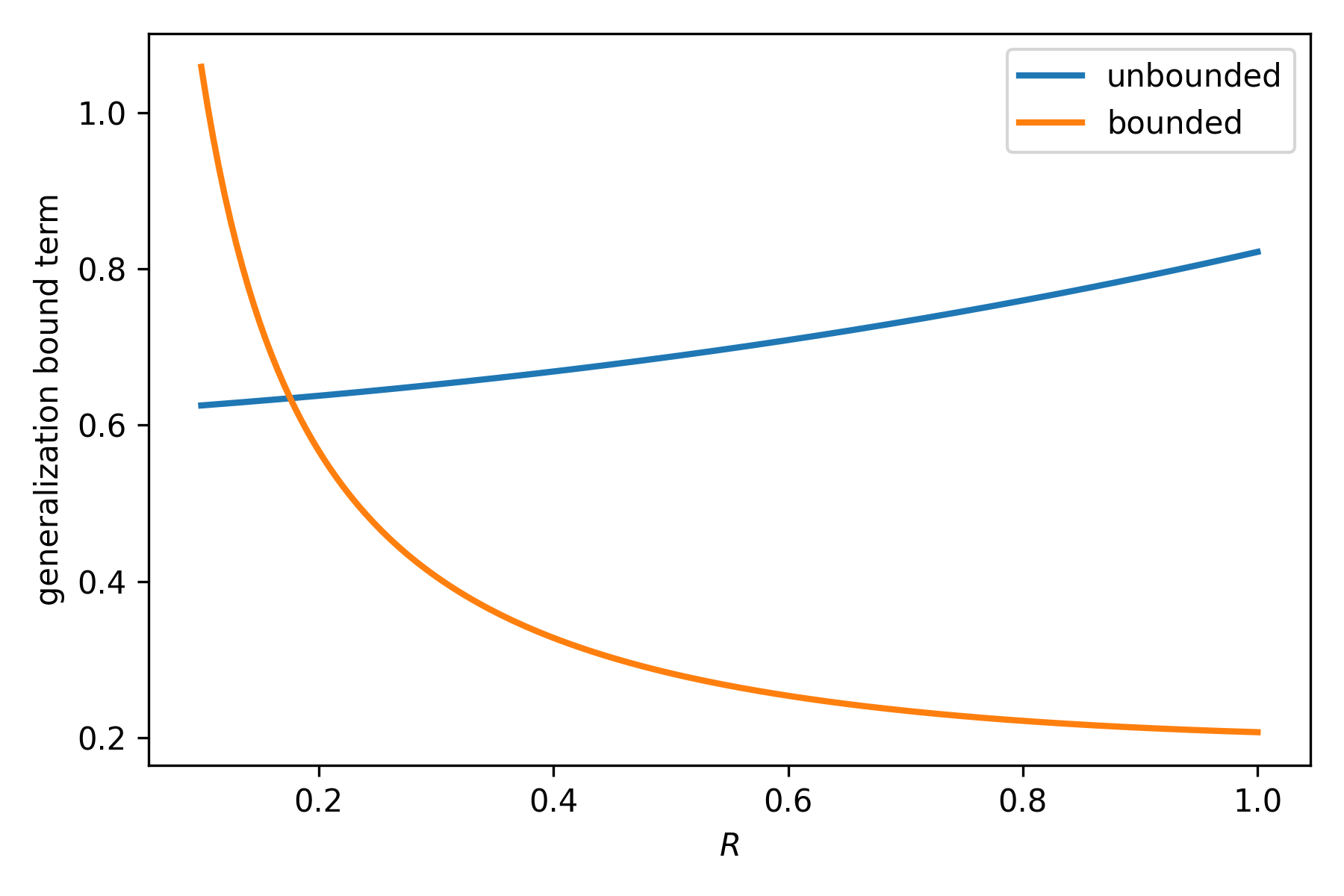}
\caption{$C=10$}
\end{subfigure}

\caption{Comparison between the bounded and unbounded derandomized PAC-Bayes generalization bound terms for the multi-class sigmoid loss.}
\label{fig:bounded-vs-unbounded-sigmoid}
\end{figure}

\begin{figure}[t]
\centering

\begin{subfigure}{0.48\linewidth}
\centering
\includegraphics[width=\linewidth]{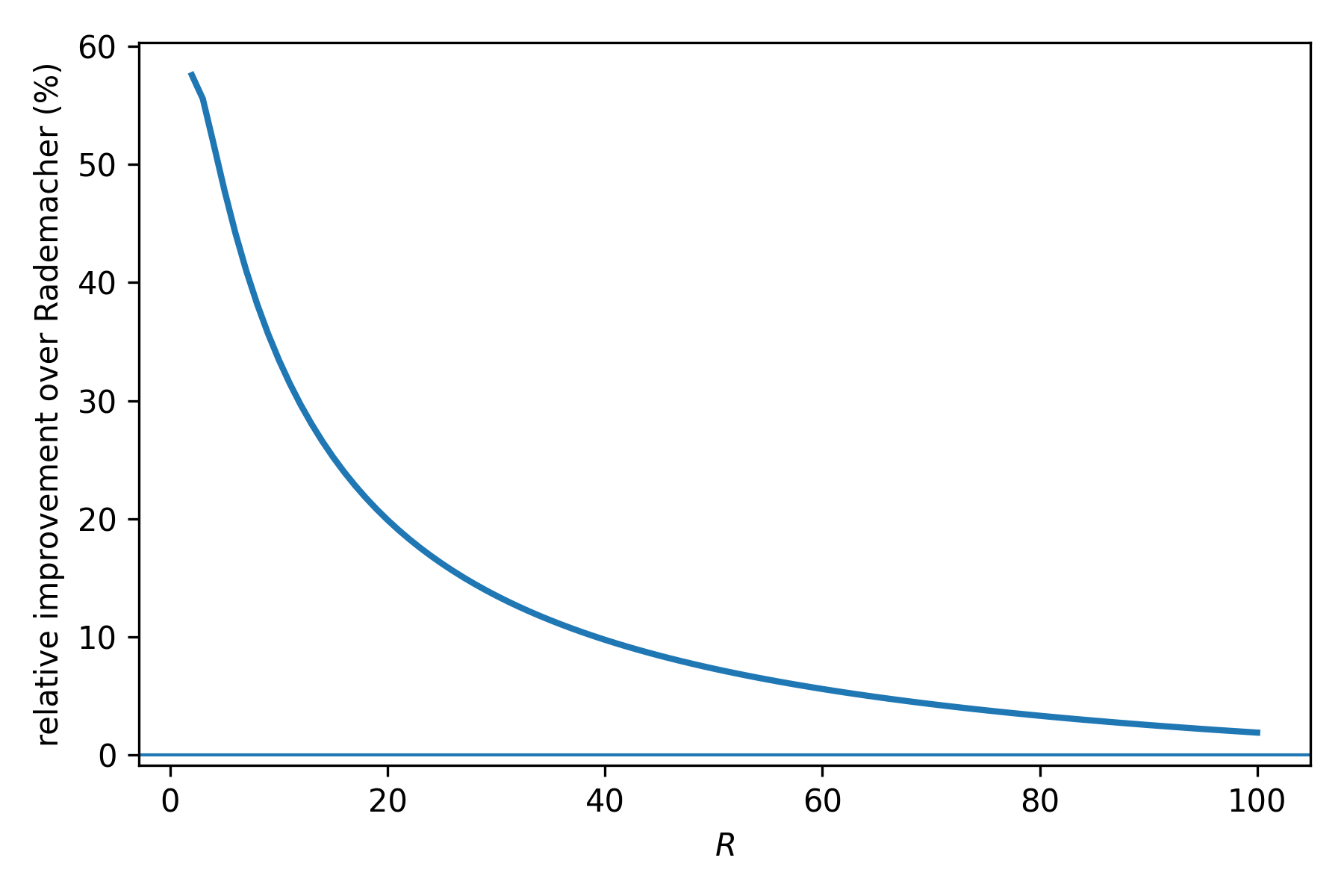}
\caption{$C=3$}
\end{subfigure}
\hspace{0.02\linewidth}
\begin{subfigure}{0.48\linewidth}
\centering
\includegraphics[width=\linewidth]{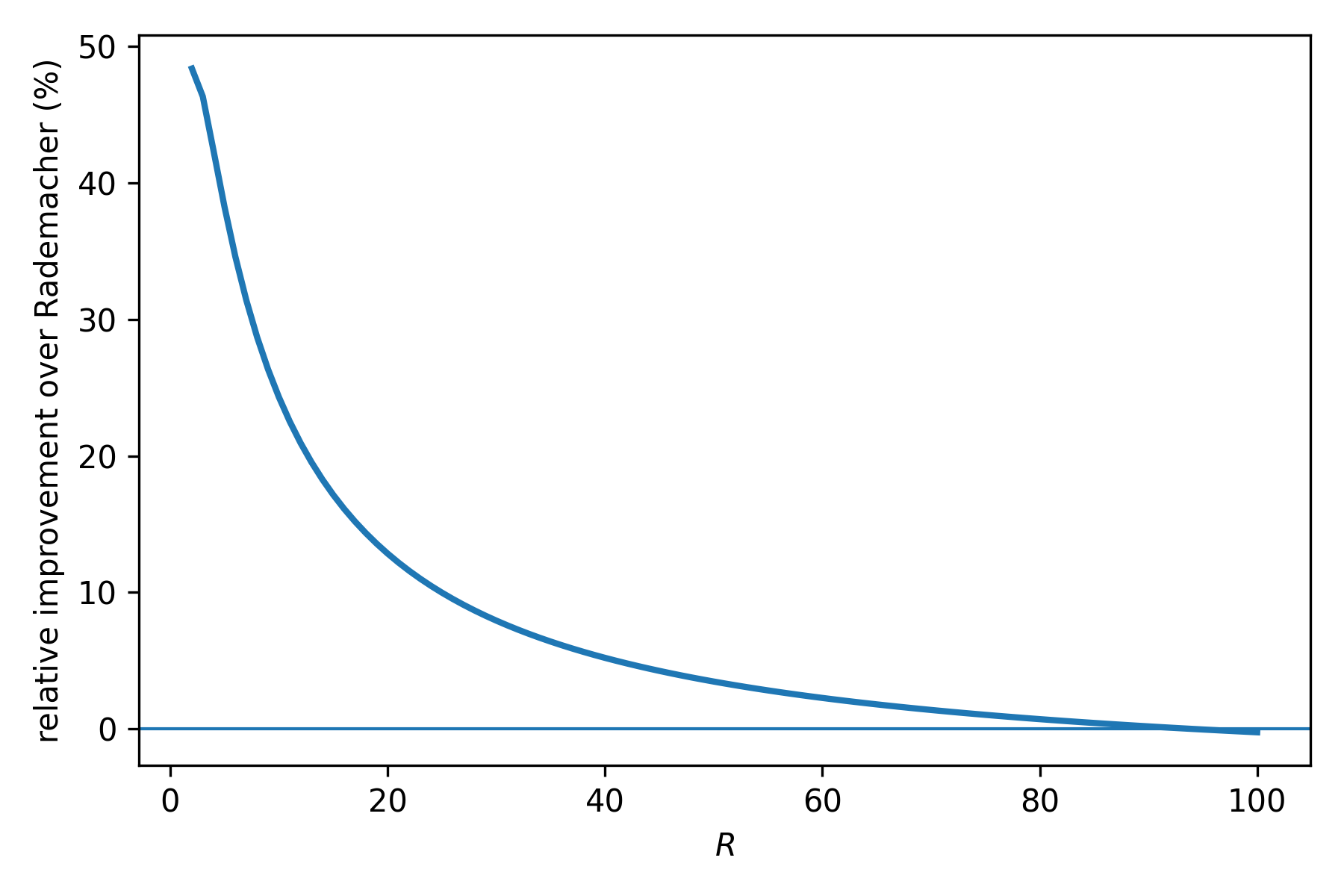}
\caption{$C=10$}
\end{subfigure}

\caption{Relative improvement of our derandomized PAC-Bayes bound over the Rademacher complexity bound as a function of $R\in [2,100]$ for the multi-class sigmoid loss.}
\label{fig:relative-improvement}
\end{figure}

\begin{figure}[t]
    \centering
    \includegraphics[width=0.48\linewidth]
    {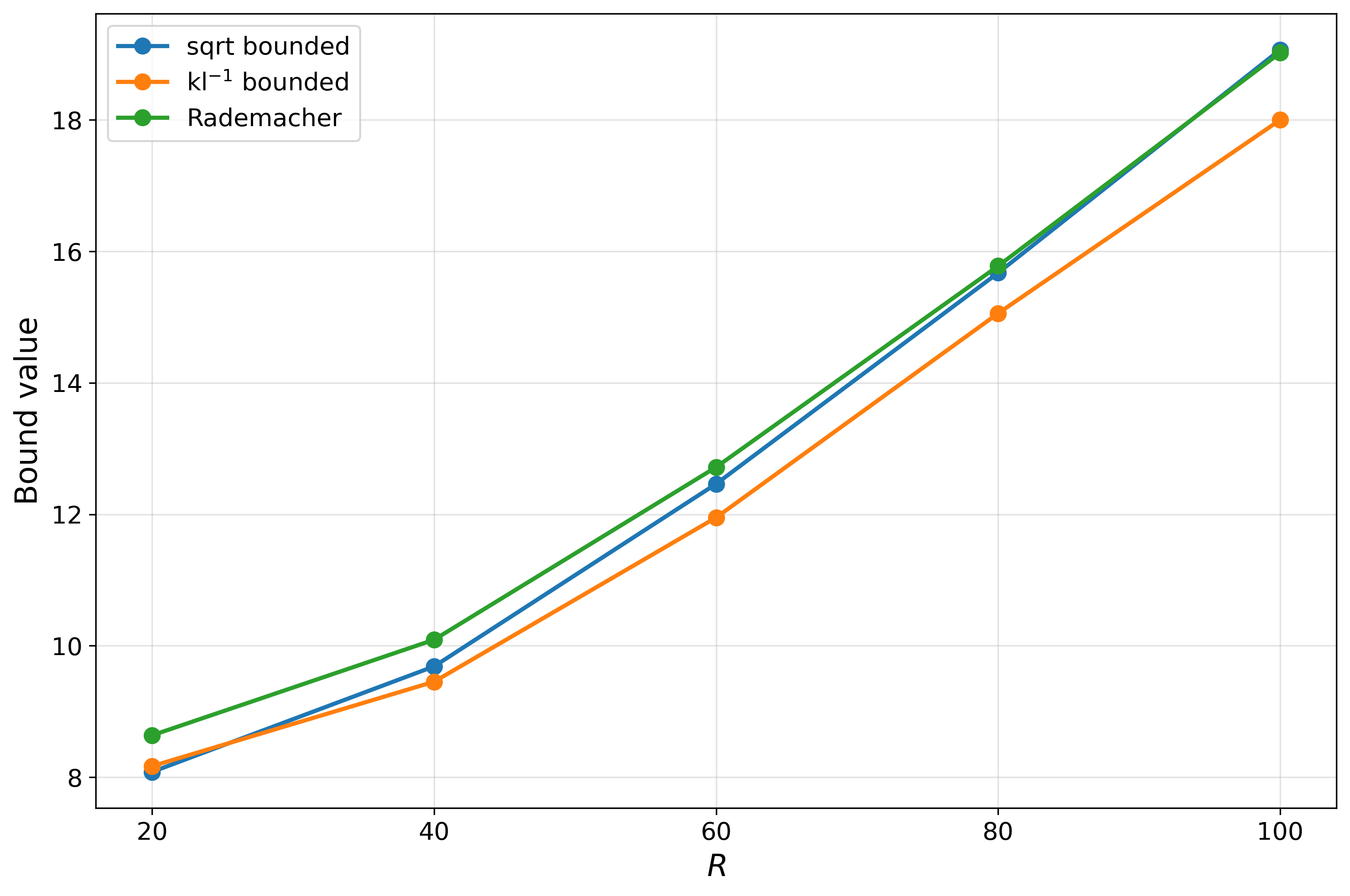}
    \caption{
    Comparison of the square-root and \(\klb^{-1}\) derandomized PAC-Bayes
    bounds with the Rademacher complexity bound for linear classifiers
    trained with the cross-entropy loss on MNIST, as functions of the
    parameter radius \(R\). For each \(R\), the classifier is trained for \(40\) epochs using projected
SGD with momentum, with a linear warm-up followed by cosine annealing.
    }
    \label{fig:linear-ce-bounds}
\end{figure}

\FloatBarrier

\section{Single hidden layer erf (SHEL) networks}
\label{sec:shel}

We now turn to the specialization of our results to a single hidden layer neural network with smooth activation functions. For the activation functions, we consider the error function, which may be seen as a Gaussian smoothing of the sign function \citep{germain2009pacbayesian,letarte2019dichotomize}. Our goal is to compare the resulting bounds with Theorem~5 of \citet{biggs2022margins}. We begin by collecting in the next lemma some quantities needed to evaluate our bounds for the SHEL network.

\begin{lem}[Specialization to the SHEL network]
\label{lem:shel-specialization}
For $x\neq 0$, consider
\[
F_{U,V}(x)
=
V\,\operatorname{erf}\!\left(\frac{Ux}{\sqrt{2}\,\|x\|}\right),
\qquad
U\in\mathbb{R}^{K\times d},
\quad
V\in\mathbb{R}^{C\times K},
\]
where \(\operatorname{erf}\) is applied element-wise and is defined by
\[
\operatorname{erf}(x):=\frac{2}{\sqrt{\pi}}\int_0^x e^{-t^2}\,dt.
\]
The Lipschitz constant and $\beta$-smoothness of the activation function $\operatorname{erf}$ are given respectively by 
\[
\ell_{\operatorname{erf}}:=\frac{2}{\sqrt{\pi}},
\qquad
\beta_{\operatorname{erf}}:=2\sqrt{\frac{2}{\pi e}}.
\]
Assume that
$\|U\|_F\le R_U$,
$\|V\|_F\le R_V$ and that, for every \(y\), the map \(z\mapsto L(z,y)\) is
\(\ell_L\)-Lipschitz and \(\beta_L\)-smooth. Then,

\begin{enumerate}

\item[(i)] (Jensen-gap upper bound for bounded activation functions) From Theorem~\ref{jenUpperBound} with
\[
\begin{aligned}
&T=2,\quad M_0^2=\frac12,\quad M_1^2=K,\quad d_1=K,\quad d_2=C,\\
&\ell_1=\ell_{\operatorname{erf}},\quad \beta_1=\beta_{\operatorname{erf}},\quad
\ell_2=\ell_L,\quad \beta_2=\beta_L.
\end{aligned}
\]
we obtain
\[
|\jen_Q[L(F_{U,V}(x),y)]|
\le
\sigma^2 K\left(
\frac{\ell_L}{\sqrt{2\pi e}}\,R_V
+\frac{\beta_L}{\pi}R_V^2
+\frac{\beta_L C}{2}
\right).
\]
\item[(ii)] (Jensen-gap upper bound for unbounded activation functions) From Theorem~\ref{jenUpperBound-unbounded} with
\[
T=2,\quad
M_0^2=\frac12,\quad
d_1=K,\quad
d_2=C,\quad
\ell_1=\ell_{\operatorname{erf}},\quad
\beta_1=\beta_{\operatorname{erf}},\quad
\ell_2=\ell_L,\quad
\beta_2=\beta_L,
\]
we obtain
\[
|\jen_Q[L(F_{U,V}(x),y)]|
\le
\sigma^2\left[
K\left(
\frac{\ell_L}{\sqrt{2\pi e}}\,R_V
+\frac{\beta_L}{\pi}R_V^2
\right)
+
\frac{\beta_L C}{\pi}R_U^2
+
\frac{\beta_L C}{\pi}K\sigma^2
\right].
\]

\item[(iii)] (KL upper bound for isotropic Gaussian prior and posterior) If
\[
P=\mathcal N(0,\sigma^2 I_m),
\qquad
Q=\mathcal N\!\bigl((\operatorname{vec}(U),\operatorname{vec}(V)),\sigma^2 I_m\bigr),
\qquad
m=K(d+C),
\]
then
\[
\KL(Q\|P)
=
\frac{\|U\|_F^2+\|V\|_F^2}{2\sigma^2}
\le
\frac{R_U^2+R_V^2}{2\sigma^2}.
\]

\end{enumerate}
\end{lem}

Since both parts~(i) and~(ii) apply to the SHEL network, the Jensen gap may be bounded by the minimum of the two corresponding upper bounds. We will compare our results to the following margin based bound.

\begin{thm}[Explicit version of Theorem~5 of \citet{biggs2022margins}]
\label{thm:biggs-guedj-thm5-klinv}
Consider the SHEL network
\[
F_{U,V}(x)=V\,\erf\!\left(\frac{Ux}{\sqrt{2}\,\|x\|}\right),
\qquad
U\in\mathbb{R}^{K\times d},
\quad
V\in\mathbb{R}^{C\times K},
\]
and let
\[
V_\infty:=\max_{i,j}|V_{ij}|.
\]
Then, with probability at least \(1-\delta\) over \(S=\{(x_j,y_j)\}_{j=1}^n\sim D^n\), we have uniformly for all
\(\gamma\in(0,V_\infty K)\),
\[
L_D^{0\text{-}1}(F_{U,V})
\le
\frac1n
+
\klb^{-1}\!\left(
 L_{S,\gamma}(F_{U,V})+\frac1n
\;\middle|\;
\varepsilon_{U,V,\gamma,\delta,n}
\right),
\]
where
\[
\begin{aligned}
\varepsilon_{U,V,\gamma,\delta,n}
:=\;&
\frac{17\ln n}{n}
\left(\frac{2V_\infty K}{\gamma}\right)^2
\left(
\frac{\|U-U^0\|_F^2}{2K}
+
\frac{\ln 2}{V_\infty^2K}\,\|V\|_F^2
\right)
\\
&\;+\;
\frac{1}{n}\ln\frac{4\sqrt n}{\delta}
\;+\;
\frac{2}{n}\ln\!\left(
\frac{\ln(4V_\infty K/\gamma)}{\ln 2}
\right).
\end{aligned}
\]
\end{thm}

\noindent\textbf{Proof:}
This follows directly from the proof of Theorem~5 in \citet{biggs2022margins}, but applying \(\klb^{-1}\) instead of Pinsker's inequality at the end of the proof in order to keep a tighter result. \Endproof

In our experiments, the SHEL network is trained for $100$ epochs on MNIST with Adam and with a cosine annealing learning rate schedule preceded by a linear warm-up phase. The training objective is the cross-entropy loss. We use the following temperature-scaled multi-class sigmoid loss for evaluating our bounds:
\[
L_{\tau}(z,y)
=
2\sum_{j\neq y}
s\!\left(\tau(z_j-z_y)\right),
\qquad
s(t):=\frac{1}{1+e^{-t}},
\]
where \(\tau>0\).
 Our reported bounds use the bounded
derandomized PAC-Bayes bound in \(\klb^{-1}\) form from
Theorem~\ref{thm:bounded-kl-regrouped-rademacher-sigmai}. The temperature parameter is optimized over
a dyadic grid of the form
\[
\tau_i = 2^{1-i}\tau_{\max},
\qquad
\tau_{\max}=2,
\]
and a union bound is used over the \(\tau\)-grid. The posterior scale
parameter \(\sigma\) is similarly optimized over a dyadic grid with
\(\sigma_{\max}=1\).
 Finally, in order to compare the relative behavior of the bounds outside the vacuous regime, we evaluate the bounds with an effective sample size \(n^\ast\) much larger than the true sample size \(n\). This allows the optimization over the margin parameter \(\gamma\) in the Biggs bound and over \(\tau\) in our bounds to avoid the trivial value \(1\). Indeed, at the true sample size \(n\), the optimized bounds are equal to \(1\) and therefore do not allow a meaningful comparison between the different methods. The evaluations based on \(n^\ast\) should thus be interpreted only as comparisons of the relative behavior of the bounds, and not as PAC guarantees for the true sample size.

Figures~\ref{fig:ce-training-batch-size}
--\ref{fig:ce-training-width} report the resulting total bounds and test
error as functions of either the hidden width \(K\) or the batch size, for
learning rates \(0.01\) and \(0.05\). The reported results are averaged over
three random seeds. Whenever the quantities \(J\) and \(H\) involve a
supremum over the parameter class, we replace it by the largest value
observed during training. The Gaussian averaged quantities are estimated by
Monte Carlo sampling. These empirical approximations are not certified upper
bounds on the corresponding theoretical quantities and therefore lead to an
optimistic evaluation. Consequently, our main objective is not to compare
the absolute tightness of the bounds, but rather to study their relative
behavior as functions of the width \(K\) and the batch size.

As can be seen in Figure~\ref{fig:ce-training-batch-size}, the test error
appears essentially constant as a function of the batch size in these
experiments. To quantify how close each quantity is to being constant with
respect to the batch size \(B\), we fit
\[
\log q = a+\alpha\log B,
\]
where \(q\) denotes the quantity under consideration. The fitted
exponents for the test error are \(-0.026\) and \(-0.024\) for learning rates
\(0.01\) and \(0.05\), respectively. The
bounds exhibit a more noticeable dependence on the batch size. For learning
rate \(0.01\), the Biggs bound decreases with exponent \(-0.136\), whereas
the smoothness-based derandomized bound increases mildly with exponent
\(0.078\). This opposite behavior may reflect the greater sensitivity of our bound to the
empirical flatness quantities, which can increase with the batch size. For learning rate \(0.05\), both bounds
decrease, with exponents \(-0.237\) for the Biggs bound and \(-0.152\) for
our bound. The decrease of our bound is, however, mainly concentrated
between batch sizes \(200\) and \(500\). Indeed, over the range from \(500\) to \(2000\), its fitted exponent is \(0.009\), indicating essentially constant behavior.

As shown in Figure~\ref{fig:ce-training-width}, the test error also appears
essentially constant as the width \(K\) increases. This indicates that the larger
number of parameters does not lead to increased overfitting over the range
considered. This observation is confirmed by the slightly negative fitted
exponents reported in Table~\ref{tab:width-global-slopes}. By contrast, both
bounds increase with \(K\), although the smoothness-based derandomized bound
grows more slowly than the Biggs bound for both learning rates. Its estimated
exponents are \(0.274\) and \(0.378\), compared with \(0.759\) and \(0.457\)
for the Biggs bound.

It is difficult to decompose the optimized bound cleanly and isolate the
contribution of each term, since the infimum over the bound parameters ties
the different terms together. To obtain some insight into the source of the
observed width dependence, we therefore fix \(\sigma\) and \(\tau\) across
the different values of \(K\). We then report in Table~\ref{tab:fixed-sigma-tau-jh-slopes} the global
exponents of different terms in our bound. The Jacobian--Hessian term
grows more slowly than the remainder of the bound. Moreover, removing the
factor \(\|w\|\) reduces its exponent from \(0.232\) to \(0.110\). These
results suggest that the empirical Jacobian and Hessian quantities contribute to the weaker dependence of our bound on \(K\).

The results of this section suggest that, in the derandomization process, we
should seek to retain as many data-dependent quantities as possible beyond
the empirical surrogate used to upper-bound the \(0\)-\(1\) loss. This appears
to be beneficial even for arguably small neural networks such as the SHEL
networks. In the next section, we move from this correlational
evidence toward a more causal investigation of the role played by these quantities.

\begin{figure}[t]
    \centering

    \begin{subfigure}[t]{0.48\textwidth}
        \centering
        \includegraphics[width=\textwidth]{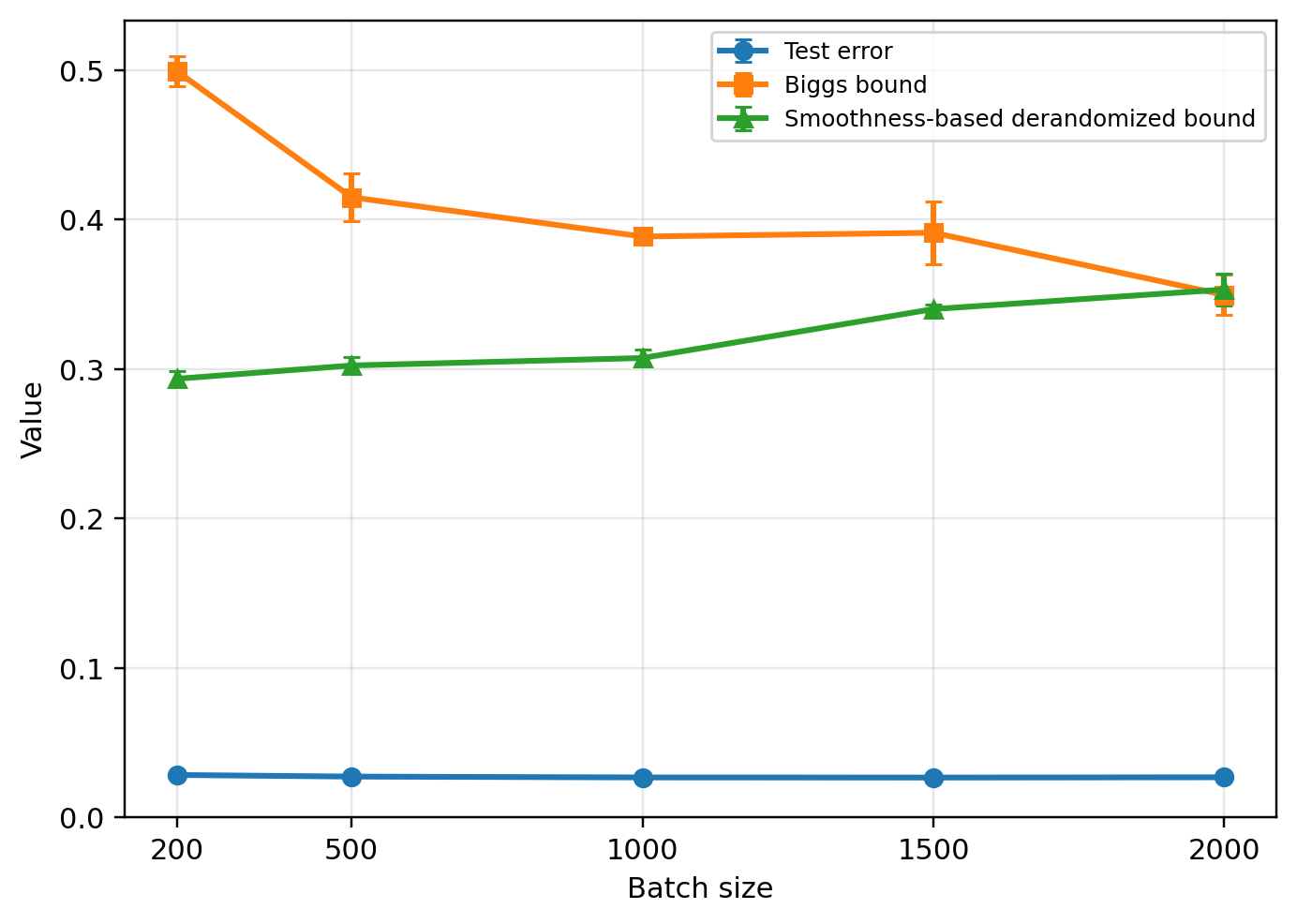}
        \caption{Learning rate \(0.01\).}
        \label{fig:ce-batch-lr001}
    \end{subfigure}
    \hfill
    \begin{subfigure}[t]{0.48\textwidth}
        \centering
        \includegraphics[width=\textwidth]{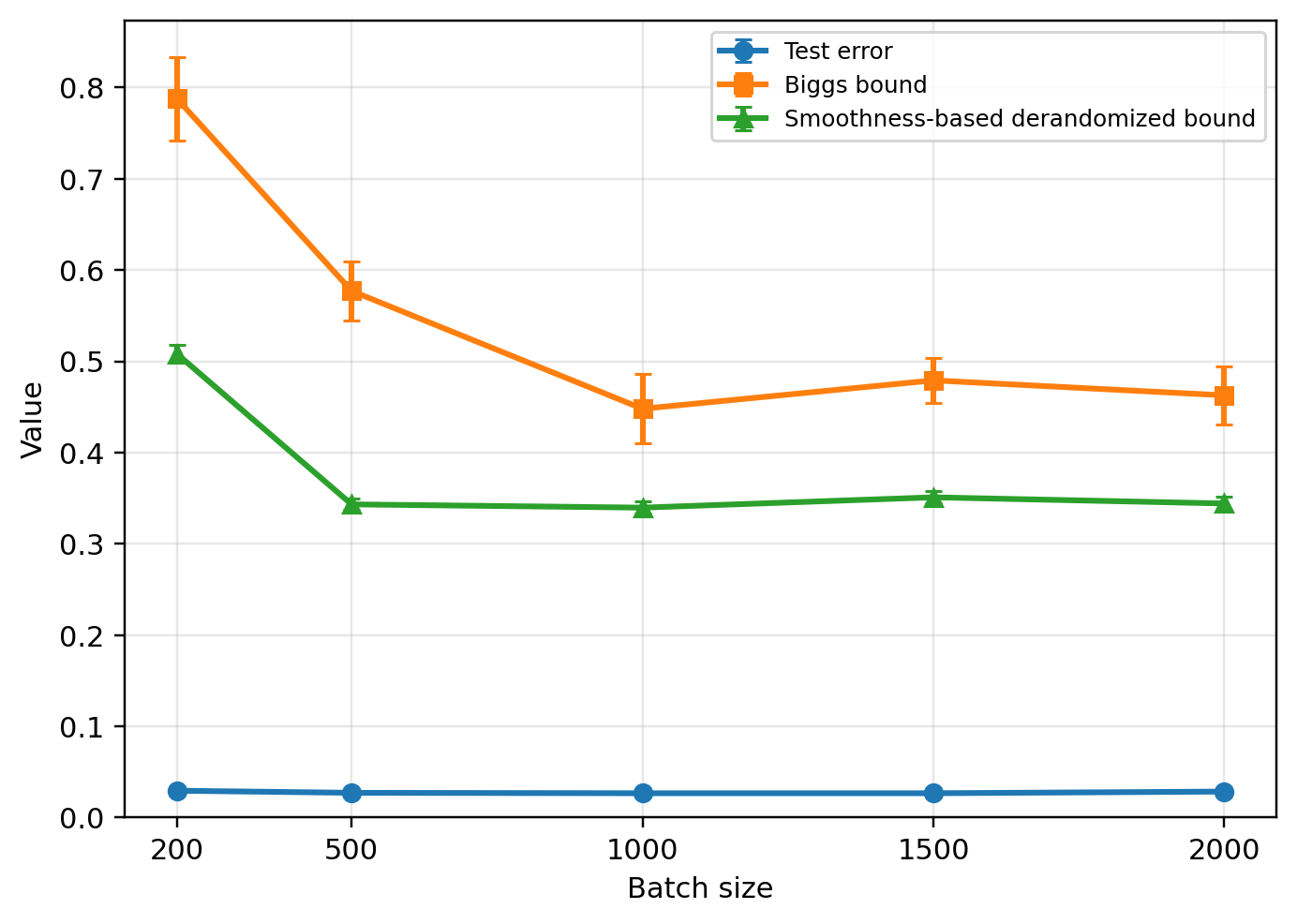}
        \caption{Learning rate \(0.05\).}
        \label{fig:ce-batch-lr005}
    \end{subfigure}

    \caption{
    Generalization bounds after training on the cross-entropy loss as a function
    of the batch size (with \(n^\ast=6\times 10^9\)).
    }
    \label{fig:ce-training-batch-size}
\end{figure}

\begin{figure}[t]
    \centering

    \begin{subfigure}[t]{0.48\textwidth}
        \centering
        \includegraphics[width=\textwidth]{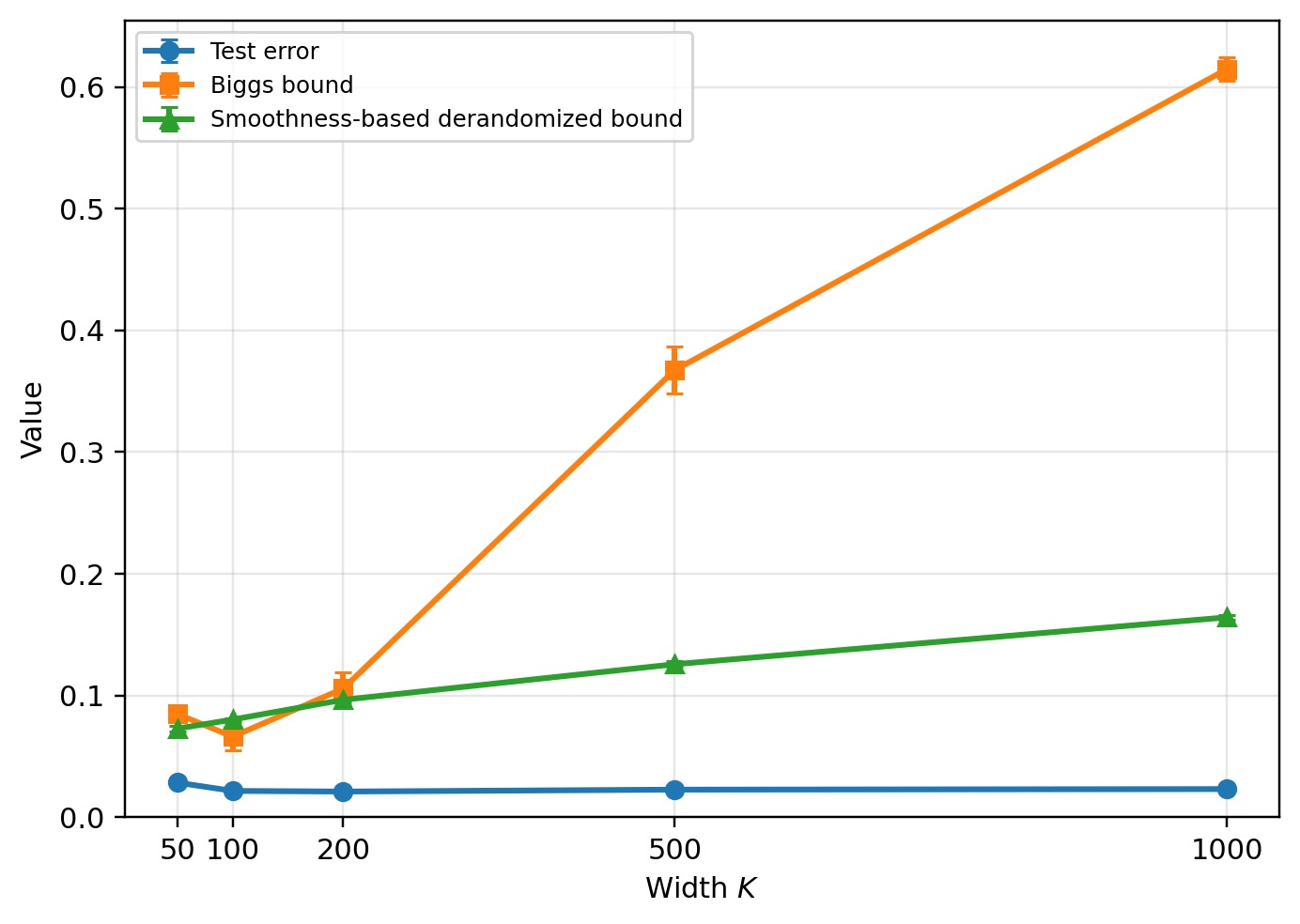}
        \caption{Learning rate \(0.01\).}
        \label{fig:ce-width-lr001}
    \end{subfigure}
    \hfill
    \begin{subfigure}[t]{0.48\textwidth}
        \centering
        \includegraphics[width=\textwidth]{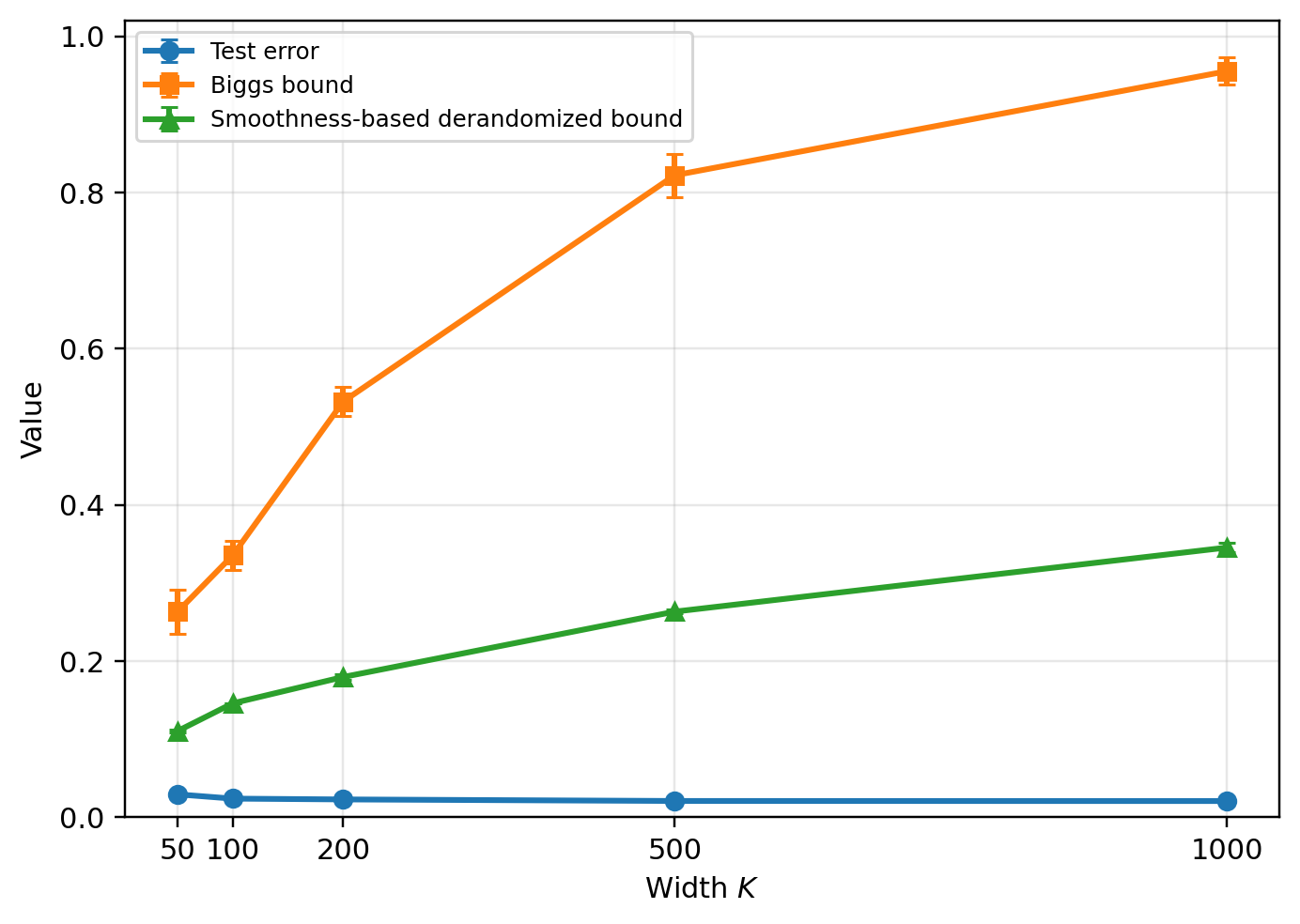}
        \caption{Learning rate \(0.05\).}
        \label{fig:ce-width-lr005}
    \end{subfigure}

    \caption{
    Generalization bounds after training on the cross-entropy loss as a function
    of the width \(K\) (with \(n^\ast=10^{11}\)).
    }
    \label{fig:ce-training-width}
\end{figure}

\begin{table}[t]
\centering
\caption{Estimated global exponents of the bounds and test error with respect to the width \(K\).}
\label{tab:width-global-slopes}
\begin{tabular}{lcc}
\toprule
Quantity & lr \(=0.01\) & lr \(=0.05\) \\
\midrule
Biggs margin bound& \(0.759\) & \(0.457\) \\
Smoothness-based derandomized bound& \(0.274\) & \(0.378\) \\
Test error & \(-0.045\) & \(-0.107\) \\
\bottomrule
\end{tabular}
\end{table}

\begin{table}[t]
\centering
\caption{Estimated global exponents with respect to the width \(K\) for lr \(=0.01\)
 with \(\sigma=0.03125\) and \(\tau=0.5\) fixed across
all widths. The exponent \(\alpha\) is obtained by fitting
\(\log y=a+\alpha\log K\) by ordinary least squares, and \(R^2\) denotes
the coefficient of determination of this linear fit in log-log space.}
\label{tab:fixed-sigma-tau-jh-slopes}
\begin{tabular}{lcc}
\toprule
Quantity & Exponent \(\alpha\) & \(R^2\) \\
\midrule
Full bound
& \(0.293\) & \(0.973\) \\[1mm]
\(\displaystyle

\|w\|\sqrt{\beta_\tau^2J+\ell_\tau^2H}
\)
& \(0.232\) & \(0.972\) \\[2mm]
\(\displaystyle
\sqrt{\beta_\tau^2J+\ell_\tau^2H}\)
& \(0.110\) & \(0.871\) \\[2mm]
Remainder after subtracting the full Jacobian-Hessian term
& \(0.378\) & \(0.868\) \\
\bottomrule
\end{tabular}
\end{table}

\FloatBarrier

\section{A Regularizer Inspired by Our Theory for Batch-Normalized Networks}
\label{sec:jh-bn-regularizer}
The Rademacher complexity bound for the Jensen-gap class obtained in our
theory (see Corollary~\ref{cor:rad-jensen-gap} and Remark~\ref{rem:global-controls-from-gaussian-averages}) motivates data-dependent
quantities that can be used as regularizers for neural networks. Rather than
attempting to compute a fully faithful version of the quantity appearing in
the bound, we evaluate a local and computationally tractable proxy at the
current network parameters during training. Let \(S_B\) denote the current training minibatch of size \(B\), and let \(S_{B'}\subseteq S_B\) denote a possibly smaller minibatch of size \(B'\) used to estimate the regularizer for computational efficiency. We define
\[
\mathrm{JH}(w;S_{B'})
:=
\sqrt{
\beta^2 J(w;S_{B'})^2
+
\ell^2 H(w;S_{B'})
},
\]
where 
\[
J(w;S_{B'})
:=
\frac{1}{B'}
\sum_{i=1}^{B'}
\|J_i(w)\|_F^2
\]
and
\[
H(w;S_{B'})
:=
\frac{1}{B'}
\sum_{i=1}^{B'}
\sum_{k=1}^C
\|H_{i,k}(w)\|_F^2.
\]
 The constants \(\ell\) and \(\beta\) denote the Lipschitz and smoothness constants of the loss with respect to the network output.

 For batch-normalized networks, we assume that each batch-normalization layer is
applied before the non-linearity. To remain close to the setting of the
theoretical results, we work with an effective parametrization in which each
batch-normalization layer is folded into the preceding affine layer. Indeed,
for fixed batch-normalization statistics, batch normalization acts as an affine
transformation of the pre-activation. Its composition with the preceding affine
map can therefore be represented as a single effective affine map. More precisely, consider an affine layer followed by batch normalization,
\[
z_t = W_t u_{t-1}+b_t,
\qquad
u_t = g_t(\mathrm{BN}_t(z_t)).
\]
Equivalently, we absorb the bias into an augmented weight matrix by defining
\[
\bar u_{t-1}
:=
\begin{pmatrix}
u_{t-1}\\
1
\end{pmatrix},
\qquad
\bar W_t
:=
\begin{pmatrix}
W_t & b_t
\end{pmatrix},
\]
so that
\[
z_t
=
W_tu_{t-1}+b_t
=
\bar W_t\bar u_{t-1}.
\]
Let \(\mu_{t,B'}\) and \(v_{t,B'}\) denote the
empirical mean and variance of \(z_t\).
For convolutional layers, these statistics are computed channelwise over both the batch and spatial dimensions. If the batch-normalization layer has affine
parameters \(\gamma_t\) and \(\delta_t\), then
\[
\mathrm{BN}_t(z)
=
\gamma_t\odot
\frac{z-\mu_{t,B'}}{\sqrt{v_{t,B'}+\varepsilon_t}}
+
\delta_t,
\]
where \(\odot\) denotes componentwise multiplication, and the division and
square root are also taken componentwise.
Define 
\[
s_{t,B'}
:=
\frac{\gamma_t}{\sqrt{v_{t,B'}+\varepsilon_t}}.
\]
Then, the affine map followed by batch normalization can be written as a single effective affine map
\[
\mathrm{BN}_t(z_t)
=
\bar W_{t,\mathrm{eff},B'}\bar u_{t-1},
\]
where
\[
\bar W_{t,\mathrm{eff},B'}
:=
\begin{pmatrix}
\operatorname{Diag}(s_{t,B'})W_t
&
s_{t,B'}\odot(b_t-\mu_{t,B'})+\delta_t
\end{pmatrix}.
\]
For convolutional layers, \(\operatorname{Diag}(s_{t,B'})W_t\) means that each
output channel of the convolutional kernel is multiplied by the corresponding
batch-normalization scale. In the augmented notation, the final column of
\(\bar W_{t,\mathrm{eff},B'}\) represents the corresponding effective bias,
which is shared across spatial locations.
We denote by
\[
w_{\mathrm{eff},B'}
=
\bigl\{
\bar W_{t,\mathrm{eff},B'}
\bigr\}_{t}
\]
the collection of all such effective affine parameters. Layers that are
not followed by batch normalization are included unchanged in this collection. The regularizer is computed with respect to the effective parametrization
\(w_{\mathrm{eff},B'}\). In other words, the Jacobian and Hessian terms are
formed in the effective coordinates and evaluated at the effective parameters
induced by the current network weights. The parameters being optimized remain
the original weights \(w\). The gradient of the regularized objective is backpropagated through the map \(w\mapsto w_{\mathrm{eff},B'}(w)\). We emphasize that the batch statistics \(\mu_{t,B'}\) and \(v_{t,B'}\) are
not treated as constants: they depend also on \(w\) and gradients are propagated through this
dependence as well.

Since, for any matrix A,
\[
\|A\|_F^2
=
\operatorname{Tr}(AA^\top)
=
\mathbb{E}_{v}\, v^\top AA^\top v
=
\mathbb{E}_{v}\,\|A^\top v\|_2^2,
\]
for any random vector \(v\) satisfying \(\mathbb{E}[vv^\top]=I\), the squared
Frobenius norms in \(J(w_{\mathrm{eff},B'},S_{B'})\) and \(H(w_{\mathrm{eff},B'},S_{B'})\) can be estimated using random matrix-vector products. In practice, this only requires Jacobian-vector and Hessian-vector products, and avoids explicitly forming or storing the full Jacobian or Hessian matrices. 

To estimate the Jacobian
term, we sample \(M\) independent random vectors \(v_1,\ldots,v_M\) in the
output space of the minibatch score vector
\[
\bigl(h_{w_{\mathrm{eff},B'}}(x_1),\ldots,
h_{w_{\mathrm{eff},B'}}(x_{B'})\bigr),
\]
and compute
\[\nabla_{w_{\mathrm{eff},B'}}
\left[
\frac{1}{\sqrt{B'}}
\left\langle
v_m,\,
\bigl(h_{w_{\mathrm{eff},B'}}(x_1),\ldots,
h_{w_{\mathrm{eff},B'}}(x_{B'})\bigr)
\right\rangle
\right],
\qquad m=1,\ldots,M.\]
The average of the squared norms of these quantities gives a stochastic estimate
of the corresponding squared Frobenius norm.

To estimate the Hessian term, for each \(m\), we additionally sample an
independent random vector \(u_m\) in parameter space and differentiate, with
respect to \(w_{\mathrm{eff},B'}\), the scalar product between the
vector-Jacobian product above and \(u_m\). This yields the required
Hessian-vector product without explicitly forming the Hessian. During training, these stochastic estimates are kept in the computational graph, so PyTorch
autograd can differentiate through the vector-Jacobian and Hessian-vector
products and optimize the resulting regularized objective.

We train an eight-layer convolutional neural network on CIFAR-10, where the
eight layers consist of six convolutional layers and two fully connected
layers. Each convolutional layer and the first fully connected layer are
followed by a BatchNorm layer and a centered Softplus activation. The loss function is the cross-entropy and training is
performed with Adam, using a cosine annealing learning-rate schedule preceded by
a linear warm-up phase of \(3\) epochs. The total number of epochs is set to
\(40\), except in the experiments that adjust the training time as a function of
the batch size. In that case, for a batch size \(B\in\{256,512,1024,2048,4096\}\), we use
\[
40\sqrt{B/256}
\]
epochs. For each batch size, we first tune the
learning rate for the baseline model. We then keep this learning rate fixed and
tune only the regularization coefficient for the \(\mathrm{JH}\)-regularized
model. The learning rate grid consists of \(10\) logarithmically spaced values
between \(10^{-3}\) and \(10^{-1}\), generated using
\(\texttt{np.geomspace}\). The regularization grid consists of \(10\)
logarithmically spaced values between \(10^{-5}\) and \(5\times10^{-4}\), except for the
\(\mathrm{JH}(w/\|w\|;S_{B'})\) method in the first experiment. For this method,
the appropriate search grid consisted of \(10\) logarithmically spaced values between
\(10^{-7}\) and \(5\times10^{-6}\). Hyperparameters are
selected using validation accuracy on \(5000\) examples and the final selected configurations are rerun over three seeds. In all experiments, the regularizer is computed on a subset minibatch of size \(B'=50\) with \(M=4\) random vectors except for the experiments reported in Table~\ref{tab:jh-bprime-probes-sqrt-epoch-comparison}.

In the first experiment, we compare four parametrizations of the
\(\mathrm{JH}\) regularizer with the unregularized cross-entropy baseline. The
first version computes \(\mathrm{JH}\) with respect to the effective weights
\(w_{\mathrm{eff},B'}\), as described in this section. The second computes
\(\mathrm{JH}\) with respect to the original weights \(w\), while the third also
includes the affine parameters of the BatchNorm layers. The fourth computes \(\mathrm{JH}\) with respect to layerwise normalized weights, where the weights of every layer of the network are replaced by \(W_t/\|W_t\|_F\). We denote this
method by \(\mathrm{JH}(w/\|w\|;S_{B'})\), where \(w/\|w\|\) is used as
shorthand for the collection \(\{W_t/\|W_t\|_F\}_t\).
Figure~\ref{fig:batchsize-jh-effective-comparison} shows the results across
batch sizes. Only the effective-weight version,
\(\mathrm{JH}(w_{\mathrm{eff},B'};S_{B'})\), and the layerwise normalized-weight version, \(\mathrm{JH}(w/\|w\|;S_{B'})\), meaningfully improve over the baseline. This highlights the importance of choosing a parametrization that is compatible with the effective
function implemented by networks with BatchNorm layers.

\begin{rem}
 Jacobian and Hessian regularizers have been studied as training objectives, including regularization in input space
\citep{sokolic2017robust,varga2017gradient,cui2022jacobianhessian}
and Hessian trace regularization with respect to the model parameters
\citep{liu2022hessiantrace}. To our knowledge, previous approaches based on curvature or sharpness for BatchNorm networks define these quantities with respect to the original weights or to normalized or scale invariant parameterizations of these weights
\citep{rangamani2019scale,tsuzuku2020normalized,yi2021bn,lyu2022understanding}. In contrast, we compute the regularizer with respect to the
BatchNorm effective weights \(w_{\mathrm{eff},B'}\). Such folding of the BatchNorm layer into the adjacent affine weights is used in
the quantization and inference literature
\citep{jacob2018quantization,nagel2021white}, where the goal is to deploy the
same trained function more efficiently. In our case, the folding is used for a
different purpose: it defines the coordinate system in which the regularizer is
computed during training. Moreover, the effective weights are batch dependent
because they are constructed using the current regularizer minibatch statistics rather than
fixed test time BatchNorm statistics.
\end{rem}

In the second experiment, we further investigate whether
\(\mathrm{JH}(w_{\mathrm{eff},B'};S_{B'})\) can mitigate the degradation in
performance sometimes observed when increasing the batch size. We compare the
baseline and the \(\mathrm{JH}\)-regularized model under two training-time
protocols: the fixed \(40\)-epoch protocol and the variable-epoch protocol
described above, in which the number of epochs is scaled by \(\sqrt{B/256}\).
We perform this comparison both on clean CIFAR-10 and under \(20\%\) uniform
label noise (with probability \(0.2\), each label is replaced by a uniformly sampled incorrect class). The results are shown in Figure~\ref{fig:batchsize-clean-fixed-vs-sqrt} for the clean-label setting and in Figure~\ref{fig:batchsize-noise20-fixed-vs-sqrt} for the label-noise setting. 

We first observe that the regularizer leads to larger absolute gains over the baseline in the label-noise setting than in the clean-label setting. This is expected since label noise exacerbates overfitting, thereby creating more room for the regularizer to improve generalization. In most cases, increasing the number of epochs benefits both the baseline and the \(\mathrm{JH}\)-regularized model. The main exception occurs under label noise at batch size \(4096\), where the fixed-epoch \(\mathrm{JH}\) run performs unusually well. We conjecture that this corresponds to an early-stopping-like regime, which may prevent excessive adaptation to the corrupted labels. Under label noise, when the number of epochs is increased according to the variable-epoch protocol, the larger batch sizes even outperform batch size \(256\). On clean data, the baseline performance deteriorates more rapidly as the batch size increases compared to the \(\mathrm{JH}\)-regularized model. Overall, these results suggest that \(\mathrm{JH}(w_{\mathrm{eff},B'};S_{B'})\) is especially beneficial in regimes where large batch training or noisy labels deteriorate performance.

\begin{table}[t]
\centering
\caption{CIFAR-10 test accuracy under the variable-epoch protocol with batch size \(B=4096\). We compare the baseline with \(\mathrm{JH}\)-regularized models using different regularizer minibatch sizes \(B'\) and different numbers \(M\) of random vectors. Results are reported as mean \(\pm\) standard deviation over three seeds.}
\label{tab:jh-bprime-probes-sqrt-epoch-comparison}
\begin{tabular}{lccc}
\toprule
Method & \(B'\) & \(M\) & Test accuracy (\%) \\
\midrule
Baseline & -- & -- & \(78.84 \pm 0.21\) \\
\(\mathrm{JH}(w_{\mathrm{eff},B'};S_{B'})\) & \(50\) & \(1\)  & \(81.19 \pm 0.43\) \\
\(\mathrm{JH}(w_{\mathrm{eff},B'};S_{B'})\) & \(50\) & \(4\)  & \(81.46 \pm 0.46\) \\
\(\mathrm{JH}(w_{\mathrm{eff},B'};S_{B'})\) & \(50\) & \(16\) & \(81.40 \pm 0.18\) \\
\(\mathrm{JH}(w_{\mathrm{eff},B'};S_{B'})\) & \(256\) & \(4\) & \(81.68 \pm 0.37\) \\
\(\mathrm{JH}(w_{\mathrm{eff},B'};S_{B'})\) & \(1024\) & \(4\) & \(81.73 \pm 0.11\) \\
\bottomrule
\end{tabular}
\end{table}

\begin{figure}[t]
    \centering
    \includegraphics[width=0.95\textwidth]{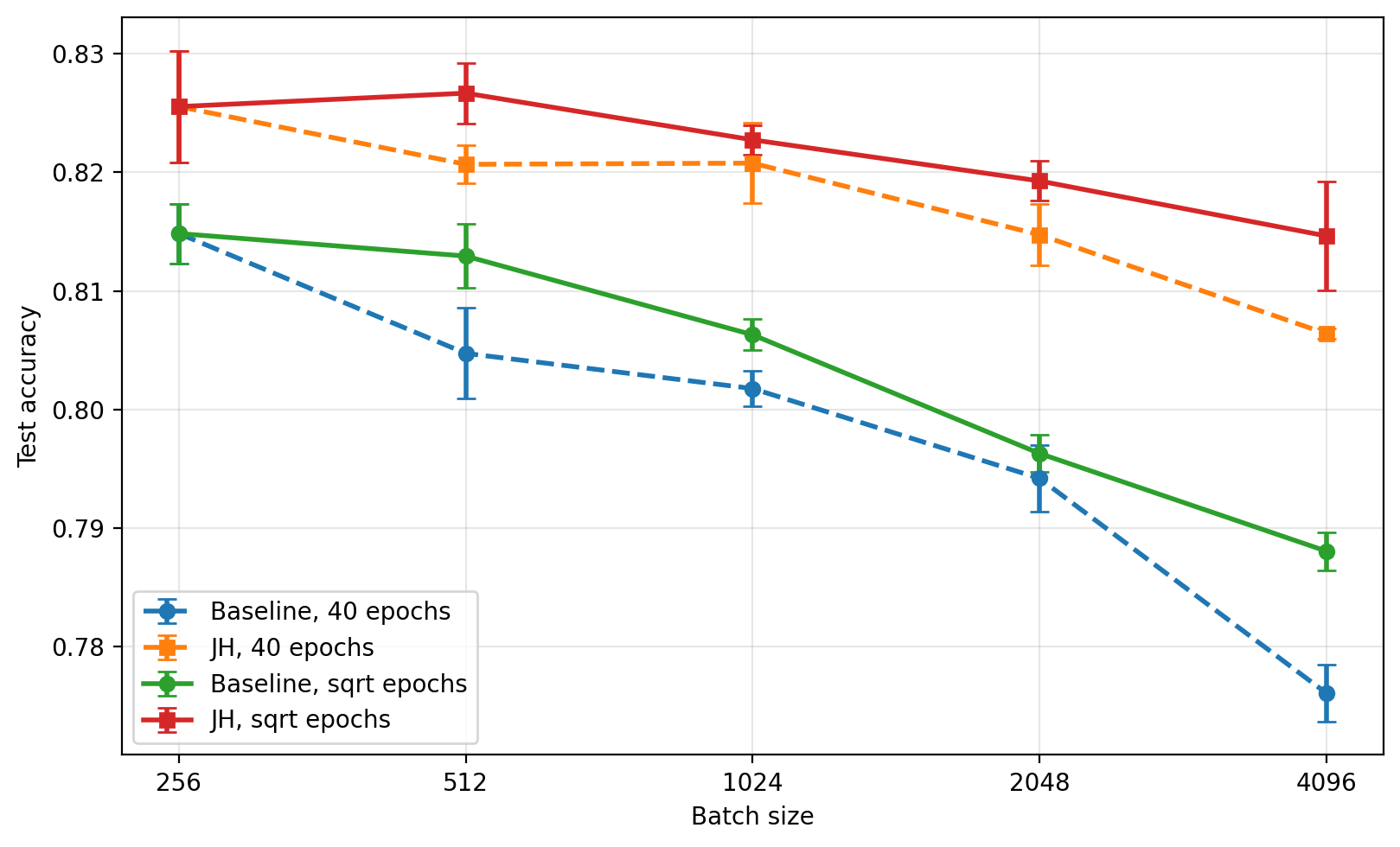}
    \caption{
    CIFAR-10 test accuracy as a function of the batch size for the baseline and the \(\mathrm{JH}\)-regularized model. The fixed \(40\)-epoch protocol is compared with the variable-epoch protocol where the number of epochs is scaled by \(\sqrt{B/256}\). Error bars indicate the standard deviation over three seeds.
    }
    \label{fig:batchsize-clean-fixed-vs-sqrt}
\end{figure}

\begin{figure}[t]
    \centering
    \includegraphics[width=0.95\textwidth]{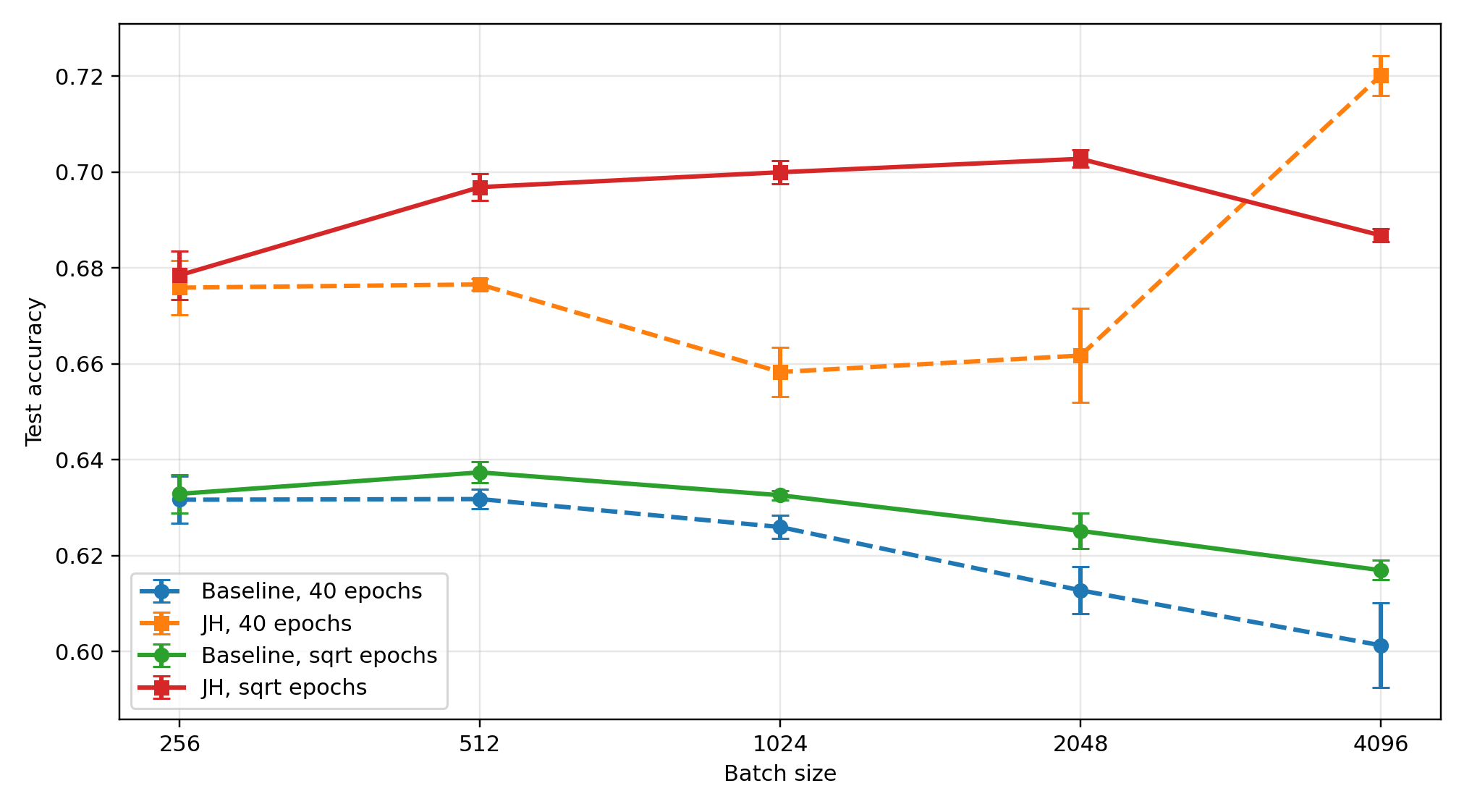}
    \caption{
    Clean CIFAR-10 test accuracy under \(20\%\) uniform label noise in the training and validation labels. The \(\mathrm{JH}\)-regularized model is compared with the baseline under both the fixed \(40\)-epoch protocol and the variable-epoch protocol where the number of epochs is scaled by \(\sqrt{B/256}\). Error bars indicate the standard deviation over three seeds.
    }
    \label{fig:batchsize-noise20-fixed-vs-sqrt}
\end{figure}

\begin{figure}[t]
    \centering
    \includegraphics[width=0.95\textwidth]{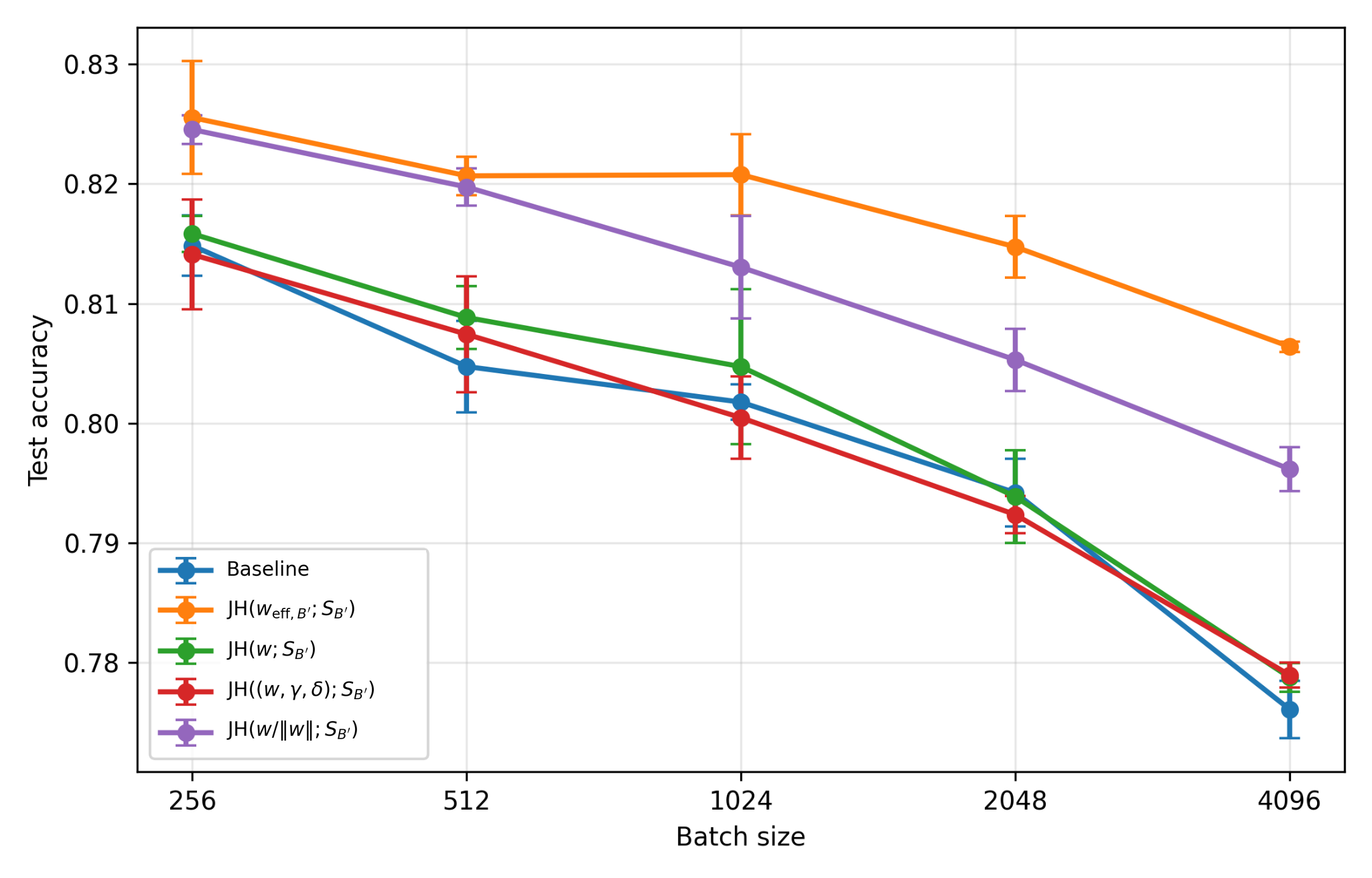}
    \caption{
    CIFAR-10 test accuracy as a function of the training batch size \(B\) for the baseline and four variants of the \(\mathrm{JH}\) regularizer, estimated using a regularizer minibatch of size \(B'=50\). Among the tested variants, only when the regularizer is
computed with respect to the effective weights or with respect to the normalized weights does the performance consistently and meaningfully improve over the baseline. Error bars indicate the standard deviation over three seeds.
    }
    \label{fig:batchsize-jh-effective-comparison}
\end{figure}

Finally, we investigate the effect of the regularizer minibatch size \(B'\) and the number \(M\) of random vectors used to estimate the Jacobian and Hessian terms. Table~\ref{tab:jh-bprime-probes-sqrt-epoch-comparison} reports the results under the variable-epoch protocol with \(B=4096\). Increasing \(B'\) led to a small improvement in test accuracy, while increasing \(M\) beyond \(4\) did not provide a clear additional benefit.

\begin{rem}
A natural idea suggested by the generalization bounds would be to combine the \(\mathrm{JH}\) regularizer with a norm penalty on the effective weights \(w_{\mathrm{eff},B}\). However, based on our preliminary experiments, this did not meaningfully improve performance compared with using only the \(\mathrm{JH}\) regularizer in our experimental setting. In some sense, the effective parameters \(w_{\mathrm{eff},B}\) are not arbitrary because of the BatchNorm layers. Indeed, BatchNorm induces some control over the scale of the inputs fed to the nonlinearities, although this control is data-dependent and comes with the caveat that the affine BatchNorm parameters can partially undo it. This type of control may already be sufficient in some practical settings to make additional norm control unnecessary. We leave a more thorough investigation of these considerations to future work.
\end{rem}

\begin{rem}
Unlike a regularizer built from the Hessian of the labeled loss, the \(\mathrm{JH}\) regularizer is label-free: it is
computed from the Jacobian and Hessian of the score map with respect
to the (effective) parameters. Consequently, \(\mathrm{JH}\) can also be evaluated on unlabeled
inputs, making it potentially suitable as an auxiliary regularizer in settings
where unlabeled data are available, such as semi-supervised or unsupervised
learning. Future research could investigate this direction.
\end{rem}

\clearpage

\section{Conclusion}
We introduced a smoothness based approach to PAC-Bayes derandomization. The
main message is that the passage from a randomized Gibbs predictor to the
deterministic predictor at the posterior mean can be isolated as a separate
generalization problem, namely bounding the generalization gap of the Jensen
gap class. This viewpoint makes it possible to combine PAC-Bayes control of
the randomized predictor with data dependent control of the derandomization
penalty. Controlling this derandomization term through Rademacher complexity
naturally brings in Jacobian and Hessian quantities of the score map.

An important direction for future work is to obtain sharper high probability bounds for the Jensen gap class. In particular, local Rademacher complexity bounds could lead to tighter bounds than those obtained here. Another promising direction is an algorithmic stability analysis of the Jensen gap class, which could provide a different route to controlling the derandomization penalty and clarify its behavior under a given optimization algorithm.

\section*{Acknowledgements}

The authors thank Mathieu Bazinet and Pascal Germain for helpful discussions on PAC-Bayes theory. This research was enabled in part by support provided by Calcul Québec and the Digital Research Alliance of Canada.
 
\bibliographystyle{plainnat}
\bibliography{iclr2024_conference}

@InProceedings{clerico2025deterministic,
  title = {Generalisation under gradient descent via deterministic {PAC}-{Bayes}},
  author = {Clerico, Eugenio and Farghly, Tyler and Deligiannidis, George and Guedj, Benjamin and Doucet, Arnaud},
  booktitle = {Proceedings of The 36th International Conference on Algorithmic Learning Theory},
  pages = {349--389},
  year = {2025},
  editor = {Kamath, Gautam and Loh, Po-Ling},
  volume = {272},
  series = {Proceedings of Machine Learning Research},
  month = {24--27 Feb},
  publisher = {PMLR},
  url = {https://proceedings.mlr.press/v272/clerico25a.html}
}

@inproceedings{casado2024pacbayeschernoff,
  title = {{PAC}-{Bayes}-{Chernoff} bounds for unbounded losses},
  author = {Casado, Ioar and Ortega, Luis A. and P{\'e}rez, Aritz and Masegosa, Andr{\'e}s R.},
  booktitle = {Advances in Neural Information Processing Systems},
  volume = {37},
  year = {2024},
  url = {https://openreview.net/forum?id=CyzZeND3LB}
}

@inproceedings{rivasplata2020pacbayesbeyond,
  title = {{PAC}-{Bayes} analysis beyond the usual bounds},
  author = {Rivasplata, Omar and Kuzborskij, Ilja and Szepesv{\'a}ri, Csaba and Shawe-Taylor, John},
  booktitle = {Advances in Neural Information Processing Systems},
  volume = {33},
  year = {2020},
  url = {https://arxiv.org/abs/2006.13057}
}

@InProceedings{haddouche2025flatminima,
  title = {A {PAC}-{Bayesian} Link Between Generalisation and Flat Minima},
  author = {Haddouche, Maxime and Viallard, Paul and Simsekli, Umut and Guedj, Benjamin},
  booktitle = {Proceedings of The 36th International Conference on Algorithmic Learning Theory},
  pages = {481--511},
  year = {2025},
  editor = {Kamath, Gautam and Loh, Po-Ling},
  volume = {272},
  series = {Proceedings of Machine Learning Research},
  month = {24--27 Feb},
  publisher = {PMLR},
  url = {https://proceedings.mlr.press/v272/haddouche25a.html}
}

@article{liu2022hessiantrace,
  title = {Regularizing Deep Neural Networks with Stochastic Estimators of {Hessian} Trace},
  author = {Liu, Yucong and Yu, Shixing and Lin, Tong},
  journal = {arXiv preprint arXiv:2208.05924},
  year = {2022},
  url = {https://arxiv.org/abs/2208.05924}
}

@article{cui2022jacobianhessian,
  title = {Generalizing and Improving {Jacobian} and {Hessian} Regularization},
  author = {Cui, Chenwei and Yan, Zehao and Liu, Guangshen and Lu, Liangfu},
  journal = {arXiv preprint arXiv:2212.00311},
  year = {2022},
  url = {https://arxiv.org/abs/2212.00311}
}

@article{leblanc2025deterministicrisk,
  title = {A Framework for Bounding Deterministic Risk with {PAC}-{Bayes}: Applications to Majority Votes},
  author = {Leblanc, Benjamin and Germain, Pascal},
  journal = {arXiv preprint arXiv:2510.25569},
  year = {2025},
  url = {https://arxiv.org/abs/2510.25569}
}

@inproceedings{rangamani2019scale,
  title     = {A Scale Invariant Measure of Flatness for Deep Network Minima},
  author    = {Rangamani, Akshay and Nguyen, Nam H. and Kumar, Abhishek
               and Phan, Dzung and Chin, Sang H. and Tran, Trac D.},
  booktitle = {2021 IEEE International Conference on Acoustics, Speech and
               Signal Processing (ICASSP)},
  pages     = {1680--1684},
  year      = {2021},
  publisher = {IEEE},
  doi       = {10.1109/ICASSP39728.2021.9413771},
  eprint    = {1902.02434},
  archivePrefix = {arXiv},
  primaryClass  = {stat.ML}
}

@inproceedings{yi2021bn,
  title     = {{BN}-Invariant Sharpness Regularizes the Training Model to
               Better Generalization},
  author    = {Yi, Mingyang and Zhang, Huishuai and Chen, Wei
               and Ma, Zhi-Ming and Liu, Tie-Yan},
  booktitle = {Proceedings of the Twenty-Eighth International Joint Conference
               on Artificial Intelligence},
  pages     = {4164--4170},
  year      = {2019},
  publisher = {International Joint Conferences on Artificial Intelligence
               Organization},
  eprint    = {2101.02944},
  archivePrefix = {arXiv},
  primaryClass  = {cs.LG}
}

@article{shoham2021exploration,
  title   = {An Exploration into Why Output Regularization Mitigates Label Noise},
  author  = {Shoham, Neta and Avidor, Tomer and Israel, Nadav},
  journal = {arXiv preprint arXiv:2104.12477},
  year    = {2021}
}

@inproceedings{jiang2020fantastic,
  title     = {Fantastic Generalization Measures and Where to Find Them},
  author    = {Jiang, Yiding and Neyshabur, Behnam and Mobahi, Hossein
               and Krishnan, Dilip and Bengio, Samy},
  booktitle = {International Conference on Learning Representations},
  year      = {2020},
  url       = {https://openreview.net/forum?id=SJgIPJBFvH}
}

@inproceedings{dziugaite2020search,
  title     = {In Search of Robust Measures of Generalization},
  author    = {Dziugaite, Gintare Karolina and Drouin, Alexandre
               and Neal, Brady and Rajkumar, Nitarshan
               and Caballero, Ethan and Wang, Linbo
               and Mitliagkas, Ioannis and Roy, Daniel M.},
  booktitle = {Advances in Neural Information Processing Systems},
  volume    = {33},
  pages     = {11723--11733},
  year      = {2020},
  publisher = {Curran Associates, Inc.}
}

@inproceedings{mcallester1998some,
  title     = {Some {PAC}-Bayesian Theorems},
  author    = {McAllester, David A.},
  booktitle = {Proceedings of the Eleventh Annual Conference on
               Computational Learning Theory},
  pages     = {230--234},
  year      = {1998},
  publisher = {Association for Computing Machinery},
  doi       = {10.1145/279943.279989}
}

@article{seeger2002pacbayes,
  title   = {{PAC}-Bayesian Generalisation Error Bounds for
             Gaussian Process Classification},
  author  = {Seeger, Matthias},
  journal = {Journal of Machine Learning Research},
  volume  = {3},
  pages   = {233--269},
  year    = {2002},
  month   = oct,
  url     = {https://www.jmlr.org/papers/v3/seeger02a.html}
}

@book{catoni2007pacbayes,
  title     = {{PAC}-Bayesian Supervised Classification:
               The Thermodynamics of Statistical Learning},
  author    = {Catoni, Olivier},
  series    = {IMS Lecture Notes--Monograph Series},
  volume    = {56},
  publisher = {Institute of Mathematical Statistics},
  address   = {Beachwood, Ohio},
  year      = {2007},
  doi       = {10.1214/074921707000000391}
}

@inproceedings{tsuzuku2020normalized,
  title     = {Normalized Flat Minima: Exploring Scale Invariant
               Definition of Flat Minima for Neural Networks Using
               {PAC}-Bayesian Analysis},
  author    = {Tsuzuku, Yusuke and Sato, Issei and Sugiyama, Masashi},
  booktitle = {Proceedings of the 37th International Conference on
               Machine Learning},
  pages     = {9636--9647},
  year      = {2020},
  editor    = {Daum{\'e} III, Hal and Singh, Aarti},
  volume    = {119},
  series    = {Proceedings of Machine Learning Research},
  publisher = {PMLR},
  url       = {https://proceedings.mlr.press/v119/tsuzuku20a.html}
}

@misc{lemirepaquin2026symmetrization,
  title         = {Symmetrization of Loss Functions for Robust Training of Neural Networks in the Presence of Noisy Labels},
  author        = {Lemire Paquin, Alexandre and Chaib-Draa, Brahim and Gigu{\`e}re, Philippe},
  year          = {2026},
  eprint        = {2605.20347},
  archivePrefix = {arXiv},
  primaryClass  = {cs.LG},
  doi           = {10.48550/arXiv.2605.20347},
  url           = {https://arxiv.org/abs/2605.20347}
}

@article{sokolic2017robust,
  title   = {Robust Large Margin Deep Neural Networks},
  author  = {Sokoli{\'c}, Jure and Giryes, Raja and Sapiro, Guillermo and Rodrigues, Miguel R. D.},
  journal = {IEEE Transactions on Signal Processing},
  volume  = {65},
  number  = {16},
  pages   = {4265--4280},
  year    = {2017},
  doi     = {10.1109/TSP.2017.2708039}
}

@article{varga2017gradient,
  title   = {Gradient Regularization Improves Accuracy of Discriminative Models},
  author  = {Varga, D{\'a}niel and Csisz{\'a}rik, Adri{\'a}n and Zombori, Zsolt},
  journal = {arXiv preprint arXiv:1712.09936},
  year    = {2017}
}

@inproceedings{lyu2022understanding,
  title     = {Understanding the Generalization Benefit of Normalization Layers: Sharpness Reduction},
  author    = {Lyu, Kaifeng and Li, Zhiyuan and Arora, Sanjeev},
  booktitle = {Advances in Neural Information Processing Systems},
  volume    = {35},
  year      = {2022}
}

@inproceedings{jacob2018quantization,
  title     = {Quantization and Training of Neural Networks for Efficient Integer-Arithmetic-Only Inference},
  author    = {Jacob, Benoit and Kligys, Skirmantas and Chen, Bo and Zhu, Menglong and Tang, Matthew and Howard, Andrew and Adam, Hartwig and Kalenichenko, Dmitry},
  booktitle = {Proceedings of the IEEE Conference on Computer Vision and Pattern Recognition},
  pages     = {2704--2713},
  year      = {2018}
}

@article{nagel2021white,
  title   = {A White Paper on Neural Network Quantization},
  author  = {Nagel, Markus and Fournarakis, Marios and Amjad, Rana Ali and Bondarenko, Yelysei and van Baalen, Mart and Blankevoort, Tijmen},
  journal = {arXiv preprint arXiv:2106.08295},
  year    = {2021}
}

@misc{Maurer2004,
  author       = {Andreas Maurer},
  title        = {A Note on the {PAC-Bayesian} Theorem},
  year         = {2004},
  eprint       = {cs/0411099},
  archivePrefix= {arXiv},
  primaryClass = {cs.LG},
  url          = {https://arxiv.org/abs/cs/0411099}
}

@inproceedings{letarte2019dichotomize,
  author    = {Ga{\"e}l Letarte and Pascal Germain and Benjamin Guedj and Fran{\c{c}}ois Laviolette},
  title     = {Dichotomize and Generalize: {PAC-Bayesian} Binary Activated Deep Neural Networks},
  booktitle = {Advances in Neural Information Processing Systems 32},
  pages     = {6869--6879},
  year      = {2019},
  publisher = {Curran Associates, Inc.}
}

@inproceedings{germain2009pacbayesian,
  author    = {Pascal Germain and Alexandre Lacasse and Fran{\c{c}}ois Laviolette and Mario Marchand},
  title     = {PAC-Bayesian Learning of Linear Classifiers},
  booktitle = {Proceedings of the 26th Annual International Conference on Machine Learning},
  series    = {ICML '09},
  pages     = {353--360},
  year      = {2009},
  publisher = {ACM},
  address   = {Montreal, Quebec, Canada},
  doi       = {10.1145/1553374.1553419}
}

@book{talagrand2005generic,
  title={The Generic Chaining: Upper and Lower Bounds of Stochastic Processes},
  author={Talagrand, Michel},
  series={Springer Monographs in Mathematics},
  year={2005},
  publisher={Springer},
  address={Berlin, Heidelberg}
}

@article{banerjee2020derandomized,
  author    = {Arindam Banerjee and Tiancong Chen and Yingxue Zhou},
  title     = {De-randomized PAC-Bayes Margin Bounds: Applications to Non-Convex and Non-Smooth Predictors},
  journal   = {CoRR},
  volume    = {abs/2002.09956},
  year      = {2020},
  url       = {https://arxiv.org/abs/2002.09956},
  eprint    = {2002.09956},
  archivePrefix = {arXiv}
}

@inproceedings{neyshabur2018spectralpacbayes,
  author    = {Behnam Neyshabur and Srinadh Bhojanapalli and David McAllester and Nathan Srebro},
  title     = {A PAC-Bayesian Approach to Spectrally-Normalized Margin Bounds for Neural Networks},
  booktitle = {International Conference on Learning Representations (ICLR)},
  year      = {2018}
}

@inproceedings{biggs2022margins,
  author    = {Fergus Immanuel Biggs and Benjamin Guedj},
  title     = {On Margins and Derandomization in PAC-Bayes},
  booktitle = {Proceedings of The 25th International Conference on Artificial Intelligence and Statistics (AISTATS)},
  year      = {2022}
}

@article{viallard2021general,
  author    = {Paul Viallard and Pascal Germain and Amaury Habrard and Emilie Morvant},
  title     = {A General Framework for the Practical Disintegration of PAC-Bayesian Bounds},
  journal   = {CoRR},
  volume    = {abs/2102.08649},
  year      = {2021},
  url       = {https://arxiv.org/abs/2102.08649},
  eprint    = {2102.08649},
  archivePrefix = {arXiv}
}

@inproceedings{mcallester1999pacbayes,
  author    = {David A. McAllester},
  title     = {PAC-Bayesian Model Averaging},
  booktitle = {Proceedings of the 12th Annual Conference on Computational Learning Theory (COLT)},
  year      = {1999},
  pages     = {164--170},
  publisher = {ACM}
}

@inproceedings{langford2002pacbayes,
  author    = {John Langford and John Shawe-Taylor},
  title     = {PAC-Bayes \& Margins},
  booktitle = {Advances in Neural Information Processing Systems (NeurIPS)},
  year      = {2002},
  volume    = {15}
}

@inproceedings{lacasse2007majorityvote,
  author    = {Alexandre Lacasse and Fran{\c{c}}ois Laviolette and Mario Marchand and Pascal Germain and Jean-Francis Roy},
  title     = {PAC-Bayes Bounds for the Risk of the Majority Vote and the Variance of the Gibbs Classifier},
  booktitle = {Advances in Neural Information Processing Systems (NeurIPS)},
  year      = {2007},
  volume    = {19}
}

@article{germain2015risk,
  author  = {Pascal Germain and Alexandre Lacasse and Fran{\c{c}}ois Laviolette and Mario Marchand and Jean-Francis Roy},
  title   = {Risk Bounds for the Majority Vote: From a PAC-Bayesian Analysis to a Learning Algorithm},
  journal = {Journal of Machine Learning Research},
  year    = {2015},
  volume  = {16},
  pages   = {787--860}
}

@inproceedings{masegosa2020secondorder,
  author    = {Andr{\'e}s R. Masegosa and Seraf{\'i}n Moral and Manuel Cebri{\'a}n},
  title     = {Second-Order PAC-Bayesian Bounds for Weighted Majority Votes},
  booktitle = {Advances in Neural Information Processing Systems (NeurIPS)},
  year      = {2020},
  volume    = {33}
}

@article{haddouche2020pacbayes,
  title   = {PAC-Bayes Unleashed: Generalisation Bounds with Unbounded Losses},
  author  = {Haddouche, Maxime and Guedj, Benjamin and Rivasplata, Omar and Shawe-Taylor, John},
  journal = {arXiv preprint arXiv:2006.07279},
  year    = {2020}
}

@BOOK{ShalevShwartz2014,
  author    = {Shai Shalev-Shwartz and Shai Ben-David},
  title     = {Understanding Machine Learning: From Theory to Algorithms},
  publisher = {Cambridge University Press},
  year      = {2014},
  isbn      = {9781107057135},
  address   = {Cambridge, UK},
}

@article{maurer2016vector,
  author       = {Andreas Maurer},
  title        = {A vector-contraction inequality for Rademacher complexities},
  journal      = {CoRR},
  volume       = {abs/1605.00251},
  year         = {2016},
  url          = {http://arxiv.org/abs/1605.00251},
  eprinttype    = {arXiv},
  eprint       = {1605.00251},
  bibsource    = {dblp computer science bibliography, https://dblp.org}
}

@book{Vershynin2018HDP,
  title        = {High-Dimensional Probability: An Introduction with Applications in Data Science},
  author       = {Vershynin, Roman},
  publisher    = {Cambridge University Press},
  year         = {2018},
  address      = {Cambridge},
  isbn         = {9781108415194},
  url          = {https://www.math.uci.edu/~rvershyn/papers/HDP-book/HDP-2.pdf},
  note         = {See Corollary~8.5.6 (Talagrand comparison inequality)}
}

@inproceedings{germain2016pac,
  title     = {PAC-Bayesian Theory Meets Bayesian Inference},
  author    = {Germain, Pascal and Bach, Francis and Lacoste, Alexandre and Lacoste-Julien, Simon},
  booktitle = {Advances in Neural Information Processing Systems 29 (NeurIPS 2016)},
  pages     = {1--9},
  year      = {2016},
  url       = {https://papers.nips.cc/paper/6569-pac-bayesian-theory-meets-bayesian-inference}
}

@article{alquier2016properties,
  author    = {Pierre Alquier and James Ridgway and Nicolas Chopin},
  title     = {On the Properties of Variational Approximations of Gibbs Posteriors},
  journal   = {Journal of Machine Learning Research},
  volume    = {17},
  number    = {239},
  pages     = {1--41},
  year      = {2016},
  publisher = {JMLR}
}

@article{Zhou2023unhinged,
  title={On the Dynamics Under the Unhinged Loss and Beyond},
  author={Xiong Zhou and Xianming Liu and Hanzhang Wang and Deming Zhai and Jiangjunjun and Xiangyang Ji},
  journal={Journal of Machine Learning Research},
  volume={24},
  pages={1--62},
  year={2023},
  url={https://jmlr.org/papers/v24/23-0771.html}
}

@article{Ghoshneuro,
author = {Ghosh, Aritra and Manwani, Naresh and Sastry, P.S.},
title = {Making Risk Minimization Tolerant to Label Noise},
year = {2015},
issue_date = {July 2015},
publisher = {Elsevier Science Publishers B. V.},
address = {NLD},
volume = {160},
number = {C},
issn = {0925-2312},
url = {https://doi.org/10.1016/j.neucom.2014.09.081},
doi = {10.1016/j.neucom.2014.09.081},
journal = {Neurocomput.},
month = {jul},
pages = {93–107},
numpages = {15}
}

\end{document}